%% file: main.tex
\documentclass{article}

\usepackage[main,final]{neurips_2026}
\setcitestyle{numbers,square,sort&compress}

\usepackage[utf8]{inputenc}
\usepackage[T1]{fontenc}

\usepackage{url}
\usepackage{booktabs}
\usepackage{graphicx}
\usepackage{amsfonts}
\usepackage{amsmath,amssymb,amsthm}
\usepackage{microtype}
\usepackage{xcolor}

\usepackage{tikz}
\usetikzlibrary{arrows.meta, calc, positioning}
\usepackage{longtable}
\usepackage{float}

\usepackage{tcolorbox}
\tcbuselibrary{breakable}

\usepackage{hyperref}
\usepackage[all]{hypcap}
\newcommand{\figref}[1]{\hyperref[#1]{Figure~\ref*{#1}}}
\newcommand{\appref}[1]{\hyperref[#1]{Appendix~\ref*{#1}}}
\hypersetup{hidelinks,pdftitle={Resolving AdaBoost Cycling with LLMs: A Computer-Assisted Counterexample},pdfauthor={Erik Y. Wang}}

\newtheorem{theorem}{Theorem}
\newtheorem{proposition}[theorem]{Proposition}
\newtheorem{corollary}[theorem]{Corollary}
\newtheorem{lemma}[theorem]{Lemma}
\numberwithin{equation}{section}

\newcommand{\R}{\mathbb{R}}

\definecolor{ed}{RGB}{225,0,100}

\definecolor{ew}{RGB}{23,180,0}

\newtcolorbox{aibox}[1]{%
  breakable,
  colback=gray!5,
  colframe=gray!70,
  fonttitle=\bfseries\small,
  title={#1}%
}

\title{Resolving AdaBoost Cycling with LLMs: A Computer-Assisted Counterexample}

\author{%
  Erik Y.~Wang\thanks{Computational certificate available at \url{https://github.com/ewang26/AdaBoost}.} \\
  Stanford University \& Benchmark\\
  \texttt{erikwang@stanford.edu} \\
}

\begin{document}

\maketitle

\begin{abstract}
We develop a computer-assisted counterexample to the open question of whether exhaustive AdaBoost always converges to a finite cycle, posed by Rudin, Schapire, and Daubechies in COLT 2012. The construction is based on a block-product of two small AdaBoost instances whose two factors share an exact period-2 orbit for their five-step branch maps, but whose linearized return maps have dominant eigenvalues with an irrational logarithmic ratio. This irrationality forces the sequence of instances that win each round of updates to have an irrational asymptotic frequency, which precludes convergence to a finite cycle. The resulting counterexample has 2,500 examples, weak-learning margin $1/5$, and a uniform initial distribution, and it holds under every maximizing tie rule. All finite computations in the proof are certified by rational arithmetic in SymPy, with SageMath verifying the main inequalities. We also document the collaborative workflow that led to the construction, in which GPT-5.4 Pro produced the key technical arguments; Claude Opus 4.6 summarized and critiqued different versions of that work, proposing new directions to pursue; and the authors orchestrated the process by selecting directions to prioritize, resolving ambiguities, refining each argument, and verifying the final solution. Later conversations with GPT-5.6 Sol and GPT-6 Astra substantially simplified the solution. We present this as one of the first in-depth case studies of LLM-assisted mathematical research, resolving a long-standing open problem from theoretical machine learning, and detail the advantages and challenges of such research.
\end{abstract}

\input{sections/01_introduction}
\input{sections/08_related_work}
\input{sections/02_workflow}
\input{sections/02_main_result}
\input{sections/04_proof_architecture}
\input{sections/05_verification}
\input{sections/09_discussion}

\input{sections/17_acknowledgements}

\bibliographystyle{abbrvnat}
\bibliography{references}

\clearpage
\appendix

\section{Full technical proof}
\label{app:proof}

This appendix contains the full proof of Theorem~\ref{thm:main} and Corollary~\ref{cor:negation}. The argument is the one sketched in Section~\ref{sec:proof-architecture}; we give the complete statements, lemmas, and verifications here.

\input{sections/03_gadgets}

\input{sections/04_margin}

\input{sections/05_branch_maps}
\input{sections/06_2cycle}

\input{sections/07_spectra}

\input{sections/08_halfstep}
\input{sections/09_sectors}
\input{sections/10_locking}
\input{sections/11_burst}
\input{sections/12_asymptotics}
\input{sections/13_nocycle}
\input{sections/14_start}
\input{sections/aux_results}

\input{sections/16_discussion}

\input{sections/cert_documentation}
\input{sections/19_final_evidence}
\input{sections/18_usage_ledger}

\end{document}

%% file: sections/01_introduction.tex
\section{Introduction}
\label{sec:intro}

The AdaBoost algorithm~\citep{FS1997} defines a discrete dynamical system on the probability simplex in $\mathbb{R}^m$~\citep{RDS2004}, where a sequence of distributions $D_t$ is updated based on an $m \times N$ matrix $M \in \{-1, +1\}^{m \times N}$ representing the correlations between $N$ weak hypotheses and $m$ training examples. Under exhaustive selection, AdaBoost chooses a weak hypothesis with maximal edge at each step and reweights the examples accordingly. Even when a positive weak-learning margin keeps the selected edges bounded away from zero, the long-run behavior of this system is poorly understood. The maximization makes the map discontinuous, and whether the example weights always approach a finite cycle has been an open question~\citep{RDS2004,RSD2012} since it was formally posed at COLT 2012.

\paragraph{The mathematical result.}
Our main contribution is a negative answer to this question. We exhibit an explicit finite $\{-1,+1\}$-valued matrix $\widetilde M$ with weak-learning margin $\gamma(\widetilde M)=1/5>0$ such that exhaustive AdaBoost, started from the exact uniform distribution, does not converge to any finite cycle under any maximizing tie rule.

The construction is built from two small gadgets $L_A$ ($4\times4$) and $L_B$ ($5\times4$), where $L_B$ is $L_A$ with an additional easier example. They are combined into a block-product matrix $M_0=L_A\boxplus L_B$, on which each gadget runs its own AdaBoost dynamics and the two compete for updates. Both gadgets share an exact period-2 orbit for their five-step branch maps, but their linearized return maps have dominant eigenvalues $\lambda$ and $\kappa$ whose logarithmic ratio is irrational, as shown by an algebraic-norm argument. This irrationality forces the fraction of updates won by $L_A$ to converge to an irrational limit so the weights can never settle into a finite cycle.

To get a uniform start, AdaBoost on $M_0$ is first run exactly in rational arithmetic for 140 updates from a starting distribution with denominator 2,500. After these updates, the weights are close enough to the common orbit for the asymptotic argument to take over. Duplicating rows of $M_0$ then turns this starting distribution into the uniform distribution over 2,500 examples, and the conclusion holds for every product starting distribution sufficiently close to this one.

\paragraph{AI-accelerated research.}
We also document how the construction was developed through an iterative human--LLM workflow over roughly two weeks, with more than one hundred queries to GPT-5.4 Pro and Claude Opus 4.6. The process involved a lot of human revision and intervention, and verifying the solution took even longer, over a month after the models first produced the construction. Later conversations with GPT-5.6 Sol and GPT-6 Astra then substantially simplified the construction and proof (Section~\ref{sec:workflow-refinement}). We give a detailed account of the roles that the models and the authors played, the prompts we used, and the time and estimated cost of the work.

A key insight was how to divide the work between the two models based on their strengths. GPT-5.4 Pro produced the technical attempts and explored many different directions, while Claude Opus 4.6 wrote high-level summaries that assessed the value (or lack thereof) of the approaches GPT-5.4 Pro had pursued. The human role was also substantial: besides selecting the problem, it involved choosing between competing directions (in particular, pushing toward a disproof after several failed proof attempts), working around the models' refusal to work on problems they recognized as open, pressing the models to refine and critique their solutions, manually checking the solution, and rewriting much of it for clarity. The workflow is described in more detail in Section~\ref{sec:workflow} and \appref{app:workflow}, and the certificate is documented in \appref{app:certdoc}.

Altogether, this paper makes four contributions:
\begin{enumerate}
\item A negative answer to the COLT 2012 open question on AdaBoost cycling, with weak-learning margin $\gamma=1/5$ and a uniform initial distribution (Theorem~\ref{thm:main}), based on a block-product gadget construction that may be useful for other discrete dynamical systems.
\item The same conclusion for the original formulation of the question, in which the hypothesis class is closed under negation (Corollary~\ref{cor:negation}).
\item A reproducible certificate that checks every finite computation in the proof by exact rational arithmetic, with the main inequalities also checked in SageMath (Section~\ref{sec:verification}).
\item An in-depth case study of the human--LLM workflow that produced the construction, including how the work was divided between the two models and the failure modes we ran into, along with selected transcripts, failed attempts included (Section~\ref{sec:workflow}, \appref{app:workflow}).
\end{enumerate}

%% file: sections/08_related_work.tex
\section{Related work}
\label{sec:related}

The dynamical-systems view of AdaBoost was developed by \citet{RDS2004}, who established cyclic behavior in low-dimensional cases and studied margin convergence; the general cycling question was formally posed at COLT 2012~\citep{RSD2012}. Later convergence results concern different quantities or rely on additional assumptions. \citet{BelanichOrtiz2023} use ergodic theory to prove convergence of time averages, of the combined classifier, and of margins, under conditions that include the eventual absence of ties. Convergence of these averaged quantities does not imply that the example weights themselves approach a finite cycle. \citet{Snedeker2022} relates low-dimensional cycling to continued-fraction dynamics, and \citet{BejaBattais2023} reconciles several views of AdaBoost. However, none of these results settles the original question concerning the weight distributions $D_t$ themselves.

Our proof also draws on earlier analyses. The balancing identity of \citet{SF2012} gives the affine coordinates we use for each gadget, the local cyclic dynamics studied in earlier work motivate our branch and return maps, and standard spectral and invariant-region estimates control the behavior near the cycle. What is new is the combination of these tools with a competition between two product factors (the gadgets) and a number-theoretic argument showing that their contraction rates are incommensurable. Specifically, we show that the weights $D_t$ can fail to approach any cycle, which is still consistent with the convergence of the averaged quantities. Section~\ref{sec:disc-boosting} explains how the result carries over to the binary case of the real-valued generalization of \citet{NockNielsen2007}.

\subsection{Computer-assisted proofs in discrete dynamics}
Computer-assisted proofs that combine exact symbolic algebra over $\mathbb Q$ with rigorous interval arithmetic~\citep{Moore2009} are well established in dynamics, where analytic estimates have to be turned into finite, checkable bounds. Our certificate uses standard ingredients: elimination by resultants and Gr\"obner bases~\citep{cox1997ideals}; root isolation by Sturm sequences with rational endpoints; field norms from algebraic number theory~\citep{Lang2002}; and exact interval evaluation of rational maps, with checks that no denominator vanishes. Each of these techniques is standard on its own, and the contribution is primarily in assembling them for this particular argument. The certificate is implemented in SymPy~\citep{meurer2017sympy}.

\subsection{LLM-assisted mathematical research}
Recent work explores LLMs as collaborators in mathematical and scientific research~\citep{bubeck2025early,woodruff2026accelerating}, both in interactive settings~\citep{dobriban2025solving,swaminathan2026quartic} and in more autonomous ones~\citep{alexeev2026short,feng2026towards}. We contribute a detailed case study of an extended human--LLM collaboration on a problem in theoretical machine learning. The final argument came from a combination of attempts made by one model, critiques by another in fresh conversations, the authors' choice of directions, and exact verification. Our observations about the workflow come from a single project, and we did not compare models or workflows under controlled conditions.

%% file: sections/02_workflow.tex
\section{The collaboration workflow}
\label{sec:workflow}

The problem was selected from SolveAll~\citep{solveall2026}, a curated collection of open research problems in machine learning, because it had been formally posed as a concrete question and had been open for over a decade. Here, we describe the workflow that led to the original construction and its later refinement; details on the conversations, including prompts and model responses, are in \appref{app:workflow}, and practical considerations (costs, wait times, and the burden of human review) are in Section~\ref{sec:discussion}.

\subsection{Initial experimentation with different models}
\label{sec:workflow-phase1}
The problem statement (reproduced in \appref{app:workflow}) was first posed in a single prompt to GPT-5.4 Pro, Claude Opus 4.6, and Gemini Deep Think. GPT-5.4 Pro produced the most sophisticated technical attempts. Claude Opus 4.6 regularly refused to engage because it recognized the problem as open, a difficulty we ran into with all models throughout the research. Gemini Deep Think often claimed solutions with flaws that it and the other two models could immediately detect. We also tried continuing within the same conversation, with follow-ups such as \emph{``Good progress. However, you must keep pushing to resolve Rudin, Schapire, and Daubechies's problem. The only goal is a correct and complete resolution to the problem,''} to get past refusals and partial answers. This did not lead to real progress, so we moved on to the workflow described below.

\subsection{Multi-turn interactions in multiple conversations}
\label{sec:workflow-loop}
The workflow that ultimately produced the construction used the two models for different purposes, with the authors directing the process:
\begin{itemize}
\item \textbf{Solver and verifier} (GPT-5.4 Pro): attempts a partial or complete proof along a given direction, and checks existing work.
\item \textbf{Critic and strategist} (Claude Opus 4.6): reads the accumulated solver attempts in a fresh conversation, summarizes what has and has not been established, and proposes directions.
\item \textbf{Director} (the authors): coordinates the models, chooses which directions to pursue (Claude often proposed too many), decides when to abandon a direction, pushes back against refusals by restating the task and its constraints, and verifies the solution.
\end{itemize}
The turning point was the critic's fifth summary. Earlier summaries had recommended further reductions or simpler special cases, which led nowhere when passed back to the solver. The fifth summary instead framed the problem as a choice between two paths: Path~A, to construct a counterexample as a product of two systems approaching local cycles at rates with an irrational log-ratio; and Path~B, to prove that log-ratios of return-map eigenvalues always satisfy a rationality constraint that was considered likely false by the critic.

When the Path~A framing was given to GPT-5.4 Pro with an explicit instruction to pursue the disproof, it produced the block-product gadget construction that, after refinement, became Theorem~\ref{thm:main}.

In practice, the models were not interchangeable in terms of the roles they had. The solver working in long-running conversations was reluctant to abandon directions it had already begun pursuing, and the critic—which started each time from a fresh conversation—was far more willing to recommend pivots.

\subsection{Refinement with GPT-5.6 Sol and GPT-6 Astra after submission}
\label{sec:workflow-refinement}
After submission, we returned to the solution in follow-up conversations with GPT-5.6 Sol and GPT-6 Astra, which substantially simplified both the construction and the proof. A search for small rational starting weights, followed by an exact computation of the first 140 updates, reduced the uniform-start matrix from $10^{30}$ examples to 2,500. A symmetry between the two phases of each gadget halved the number of regions that needed to be checked.

This later work also produced a general theorem about the competition between the two gadgets, extended the result to every maximizing tie rule, and showed that the conclusion holds on an open set of product starting distributions. It also simplified the asymptotic analysis, reduced the final noncycling argument to tracking a single coordinate, and based the identification of the algebraic roots on exact polynomial identities and root counts. \appref{app:final-evidence} gives the details. These improvements came out of the later conversations as a whole, so we do not attribute any to a specific model. (The original counterexample itself was developed with GPT-5.4 Pro and Claude Opus 4.6.)

\begin{figure}[ht]
\centering
\begin{tikzpicture}[
 >=Latex,
 box/.style={draw,rounded corners,align=left,text width=0.275\linewidth,font=\scriptsize,inner sep=4pt},
 node distance=4mm and 7mm]
\node[box] (a) {\textbf{Classical input}\\The cycling question, local cycles, and the balancing identity.};
\node[box,right=of a] (b) {\textbf{Solver and critic}\\Attempts to prove convergence fail; the fifth summary proposes incompatible rates.};
\node[box,right=of b] (c) {\textbf{Human decision}\\Pursue the disproof; the solver builds two competing gadgets.};
\node[box,below=of a] (d) {\textbf{A failed route}\\A second mode inside one gadget cannot be independently scheduled.};
\node[box,right=of d] (e) {\textbf{Exact checks}\\State every domain and inequality precisely; check the code against the proof.};
\node[box,right=of e] (f) {\textbf{Later models}\\GPT-5.6 Sol and GPT-6 Astra simplify the proof and reduce the matrix to 2,500 rows.};
\draw[->] (a)--(b); \draw[->] (b)--(c);
\draw[->] (c.south)--++(0,-0.18)-|(d.north);
\draw[->] (d)--(e); \draw[->] (e)--(f);
\end{tikzpicture}
\caption{How the argument developed, reconstructed from the saved conversation excerpts and later revisions. An early proposal tried to use a second, passive mode within a single gadget, but two modes that are updated together cannot be scheduled independently (while the final construction uses two separate gadgets, which can). The arrows trace the main line of development.}
\label{fig:discovery}
\end{figure}

%% file: sections/02_main_result.tex
\section{The AdaBoost cycling problem and main result}
\label{sec:main}

We study exhaustive AdaBoost on a fixed sign matrix $M\in\{-1,+1\}^{m\times N}$, with any rule that selects a maximizing column, such as the smallest-index tie-breaking. Let $\Delta^{m-1}$ be the $(m-1)$-dimensional probability simplex in $\R^m$. We call the distribution $D_t\in\Delta^{m-1}$ the \emph{state} of the system at time~$t$. The \emph{edge} of hypothesis $j$ at time $t$ is $\mu_t(j)=(D_t^\top M)_j$. AdaBoost selects a column $j_t$ of maximal edge, with selected edge $r_t$,
$$
j_t\in\operatorname*{argmax}_{1\le j\le N}\mu_t(j),
\qquad
r_t=\mu_t(j_t),
$$
and then updates the state by
$$
D_{t+1}(i)=\frac{D_t(i)}{1+r_t M_{ij_t}}.
$$
We call the map $D_t\mapsto D_{t+1}$ a single AdaBoost \emph{update}.
The \emph{weak-learning margin} of $M$ is
\[
\gamma(M) = \min_{D \in \Delta^{m-1}} \max_{1 \le j \le N} (D^\top M)_j.
\]

A trajectory converges to a finite cycle if there are $h\ge1$ and states $v_0,\ldots,v_{h-1}$ such that $D_{kh+i}\to v_i$ for every $0\le i<h$. We prove the following result.

\begin{theorem}
\label{thm:main}
There exists a $2{,}500\times8$ matrix $\widetilde M\in\{-1,+1\}^{2500\times8}$ with weak-learning margin
\(
\gamma(\widetilde M)=\frac15>0
\)
such that exhaustive AdaBoost on $\widetilde M$, started from the exact uniform distribution and using any maximizing tie rule, does not converge to any finite cycle.

The same holds for the compact $20\times8$ matrix $M_0$ from Section~\ref{sec:gadgets}, starting from the rational product distribution $a_0\otimes b_0$ defined in Section~\ref{sec:arch-nocycle}, and for every product start $a\otimes b$ with $(a,b)$ in some open neighborhood of $(a_0,b_0)$ in $\operatorname{ri}(\Delta^3)\times\operatorname{ri}(\Delta^4)$, where $\operatorname{ri}$ denotes the relative interior.
\end{theorem}

In particular, the question ``does exhaustive AdaBoost always converge to a finite cycle?'' from the original COLT paper has a negative answer.

\begin{corollary}
\label{cor:negation}
There exists a finite $\{-1,+1\}$-valued matrix $\widetilde M'$ whose column set is closed under negation (for every column $h$ of $\widetilde M'$, $-h$ is also a column of $\widetilde M'$), with weak-learning margin
\(
\gamma(\widetilde M')=\frac15>0,
\)
such that exhaustive AdaBoost on $\widetilde M'$, started from the exact uniform distribution and using any maximizing tie rule, does not converge to any finite cycle. Hence the answer is also negative for the original formulation of the COLT question, in which the hypothesis class is closed under negation.
\end{corollary}

Section~\ref{sec:proof-architecture} outlines the proof. The complete proof of Theorem~\ref{thm:main} is in \appref{app:proof}, and the proof of Corollary~\ref{cor:negation} is at the end of \appref{sec:start}. The finite identities and inequalities used in the proof are verified by an accompanying certificate that works entirely in exact arithmetic: symbolic algebra over $\mathbb{Q}$ in SymPy~\citep{meurer2017sympy}, exact interval arithmetic, and square-root enclosures computed with integers. Section~\ref{sec:verification} describes the certificate. Decimal values in the paper are rounded for readability.

\begin{table}[H]
\centering
\caption{The proof uses three time scales: individual AdaBoost updates (indexed by $t$), the bursts of one gadget ($k$), and all bursts of the product system ($n$). A burst is a run of five consecutive updates to one gadget.}
\label{tab:notation}
{\small
\begin{tabular}{ll}
\toprule
Notation & Meaning\\
\midrule
$D_t$, $j_t$, $r_t$ & Full weight state, selected column, and selected edge at update $t$\\
$G\in\{A,B\}$, $H_i$ & A gadget; the five-update word it applies in phase $i$\\
$F_i^G$, $T_G$ & Branch map in phase $i$; the same map obtained by undoing the reflection\\
$x_k^G$, $c_G$ & Displacement from the common orbit; contraction factor per burst\\
$N_G(n)$, $\omega$ & Bursts won by $G$ among the first $n$; limiting frequency of A-wins\\
\bottomrule
\end{tabular}
}
\end{table}

%% file: sections/04_proof_architecture.tex
\section{Proof architecture}
\label{sec:proof-architecture}

The idea is to couple two small learning problems that approach a common
orbit at different exponential rates. After a finite transient, AdaBoost
works in five-step bursts, and each burst advances whichever component
currently has the larger edge, shrinking its lead. Because the two
logarithmic contraction rates have an irrational ratio, keeping the
gadgets in balance requires switching between them with an irrational
frequency. Each burst also moves the chosen gadget to the orbit's other point, flipping its phase. If the trajectory converged to a finite cycle, the switching would
be eventually periodic and hence have rational frequency, which is a
contradiction.

\figref{fig:mechanism-roadmap} illustrates the mechanism. The argument has six steps, one per subsection below. We give the construction and key identities here, and leave the full statements and verifications to \appref{app:proof}.

\begin{figure}[t]
\centering
\input{figures/competition_roadmap}
\caption{Competing contractions. A burst contracts the leading gadget while the other waits. In logarithmic coordinates the resulting staircase stays within a bounded band, and because the two log-rates have an irrational ratio, the frequency of wins is irrational. Blue denotes A and orange denotes B. The drawing assumes constant contraction rates; for the exact maps the relative errors are summable.}
\label{fig:mechanism-roadmap}
\end{figure}

\subsection{The gadgets and asynchronous competition}
\label{sec:arch-gadgets}
The construction uses two small sign matrices: $L_A$ has four examples and four hypotheses; $L_B$ adds one easier example (a row of $L_A$ with one $-1$ changed to $+1$). Their block product $M_0=L_A\boxplus L_B$ has 20 rows and eight hypotheses (\figref{fig:m0-boxplus}).

A row of $M_0$ is a pair of examples, one from each gadget. The A-hypotheses depend only on the first example, and the B-hypotheses only on the second. Their edges therefore depend only on the corresponding marginal distribution (Lemma~\ref{lem:marginals}).

Starting from a product distribution, an A-update changes only the A-marginal, and a B-update changes only the B-marginal. The product form is preserved (Lemma~\ref{lem:invariance}). The result is two local AdaBoost dynamics, one per gadget, competing for updates: at each update, the gadget with the larger best edge advances while the other waits.

Both gadgets, and their product, have weak-learning margin $1/5$. Section~\ref{sec:margin} gives the exact witnesses.

\begin{figure}[!ht]
\centering
\begin{tikzpicture}[
  >=Latex,
  font=\small,
  rowlab/.style={font=\scriptsize, anchor=east},
  collab/.style={font=\scriptsize},
  Atext/.style={text=blue!60!black},
  Btext/.style={text=orange!75!black}
]
  \def\cw{0.36}  
  \def\rh{0.40}  

  \node[anchor=east] (LA) at (-0.55,{-2*\rh}) {$
    L_A=\left(\begin{array}{rrrr}
      1&1&-1&-1\\
      -1&1&1&1\\
      1&-1&-1&1\\
      1&-1&1&-1
    \end{array}\right)
  $};

  \node[anchor=east] (LB) at (-0.55,{-6*\rh}) {$
    L_B=\left(\begin{array}{rrrr}
      1&1&-1&-1\\
      -1&1&1&1\\
      1&-1&-1&1\\
      1&-1&1&-1\\
      1&-1&1&1
    \end{array}\right)
  $};

  \begin{scope}[shift={(0,0)}]
    \fill[blue!8]   (0,0)        rectangle ({20*\cw},{-4*\rh});
    \fill[orange!12](0,{-4*\rh}) rectangle ({20*\cw},{-8*\rh});

    \draw[thick] (0,0) rectangle ({20*\cw},{-8*\rh});
    \draw[thick] (0,{-4*\rh}) -- ({20*\cw},{-4*\rh});

    \foreach \x in {1,...,19}
      \draw[gray!35] ({\x*\cw},0) -- ({\x*\cw},{-8*\rh});
    \foreach \x in {5,10,15}
      \draw[thick] ({\x*\cw},0) -- ({\x*\cw},{-8*\rh});

    \foreach \y in {1,2,3,5,6,7}
      \draw[gray!35] (0,{-\y*\rh}) -- ({20*\cw},{-\y*\rh});

    \foreach \j/\y in {1/0.5,2/1.5,3/2.5,4/3.5}
      \node[collab,Atext,anchor=west] at ({20*\cw+0.1},{-\y*\rh}) {$A_{\!\j}$};
    \foreach \j/\y in {1/4.5,2/5.5,3/6.5,4/7.5}
      \node[collab,Btext,anchor=west] at ({20*\cw+0.1},{-\y*\rh}) {$B_{\!\j}$};

    \node[collab,anchor=south] at ({10*\cw},{0.62})
      {$M_0=L_A\boxplus L_B$ \quad (columns indexed by $(r,s)$)};

    \foreach \rr in {1,2,3,4}{
      \draw[gray!70]
        ({(5*(\rr-1))*\cw+0.04},0.06) -- ({(5*\rr)*\cw-0.04},0.06);
      \node[collab,anchor=south,font=\scriptsize]
        at ({(5*(\rr-1)+2.5)*\cw},0.10) {$r{=}\rr$};
    }
    \node[collab,anchor=west,font=\tiny]
      at ({20*\cw+0.05},0.18) {$s{=}1{,}\dots{,}5$};

    \foreach \rr in {1,2,3,4}{
      \node[Atext,align=center,font=\tiny]
        at ({(5*(\rr-1)+2.5)*\cw},{-2*\rh})
        {row $\rr$ of $L_A$\\[-1pt]repeated in $s$};
    }

    \foreach \rr in {1,2,3,4}{
      \node[Btext,align=center,font=\tiny]
        at ({(5*(\rr-1)+2.5)*\cw},{-6*\rh})
        {row $s$ of $L_B$\\[-1pt](same in $r$)};
    }
  \end{scope}

  \draw[->,very thick,blue!60!black]
    (LA.east) to[out=0,in=180] (-0.05,{-2*\rh});
  \draw[->,very thick,orange!75!black]
    (LB.east) to[out=0,in=180] (-0.05,{-6*\rh});

\end{tikzpicture}
\caption{The block-product construction, shown transposed on the right. Its 20 row types are pairs $(r,s)$. The four A-hypotheses depend only on row $r$ of $L_A$; the four B-hypotheses depend only on row $s$ of $L_B$.}
\label{fig:m0-boxplus}
\end{figure}

\subsection{Branch maps and a common period-2 orbit on a shared face}
\label{sec:arch-branch-maps}
Near the orbit described below, a gadget that wins a burst applies one of two five-step words:
\[
H_0=(1,3,4,1,2),\qquad H_1=(1,4,3,1,2).
\]
The two words differ by exchanging hypotheses 3 and 4. Composing their updates gives rational branch maps $F_i^G$, which Section~\ref{sec:branch} defines exactly in affine coordinates on each gadget's simplex.

When the extra example in B has weight zero, the dynamics of B reduce to those of A (Lemma~\ref{lem:face-restriction}). Both gadgets therefore share the same period-2 orbit on this face of the simplex. The certificate identifies the orbit exactly by polynomial identities and root isolation, and Section~\ref{sec:2cycle} gives its coordinates.

The two gadgets approach the orbit at different rates. Their dominant return-map multipliers are $\lambda$ and $\kappa$, with $0<\lambda<\kappa<1$. The additional B-example supplies the slower transverse mode. Section~\ref{sec:spectra} gives the derivative structure and exact spectral data.

\subsection{Irrationality of the logarithmic rate ratio}
\label{sec:arch-norm}
Write $c_A=\sqrt{\lambda}$ and $c_B=\sqrt{\kappa}$ for the contraction factors per burst. Numerically, they are about $0.3074$ and $0.4515$, respectively. Their algebraic norms satisfy
\[
N_{\mathbb Q(c_A)/\mathbb Q}(c_A)=\tfrac12,
\qquad
N_{\mathbb Q(c_B)/\mathbb Q}(c_B)=1.
\]

If the logarithmic ratio were rational, some positive integer powers of $c_A$ and $c_B$ would be equal. Taking norms in a common number field would then equate a strictly positive power of $1/2$ to $1$, a contradiction. Thus $\log c_A/\log c_B$ is irrational.

Lemma~\ref{lem:norm-irrational} gives the complete norm argument. The exact polynomials and root identifications are in \appref{app:certificate}.

\subsection{Reflection-normalized maps and invariant cones}
\label{sec:arch-halfstep}
The two phases are related by a reflection that swaps the third and fourth example weights. Reflecting back after each burst turns the alternating branch maps into a single map $T_G$ for each gadget, so we need to check one local region per gadget.

At a burst boundary, the state of each gadget is described by a few affine coordinates (Section~\ref{sec:halfstep}). The first coordinate, $p$, determines the edge $2p-1$ of hypothesis 1. Let $p_*$ be its value on the common orbit, and let $x=p-p_*$ be the displacement from the orbit. Near the orbit, one burst of gadget $G$ satisfies
\[
x'=c_Gx(1+\varepsilon).
\]
The remaining coordinates measure the deviation from the dominant eigendirection. Exact Taylor bounds keep these deviations small, so that each map $T_G$ sends a thin cone around this direction into itself. The cone has radius $10^{-5}$, and Sections~\ref{sec:halfstep}--\ref{sec:sectors} give the coordinates and bounds.

The same bounds make the relative errors $\varepsilon$ summable over each gadget's own bursts. Consequently, after $k$ bursts of gadget $G$, the displacement is asymptotic to a positive constant times $c_G^k$ (Lemma~\ref{lem:log}). This is all the information about local rates that the competition argument needs.

\subsection{Whole-burst locking and bounded run-lengths}
\label{sec:arch-locking}
Strict gaps between edges ensure that the gadget winning at a burst boundary completes all five updates before the other gadget can take over. The same inequalities show that no negated hypothesis (Corollary~\ref{cor:negation}) is ever selected. Section~\ref{sec:locking} proves these locking inequalities.

At a boundary, the best edge in gadget $G$ is $\mu_*+2x_G$, where $\mu_*=2p_*-1$ (Lemma~\ref{lem:first-edge}). The gadget with the larger displacement therefore wins, and either may win in case of a tie. An A-burst contracts $x_A$ while $x_B$ stays fixed, and a B-burst does the reverse.

Neither gadget can wait forever: repeated bursts of the other gadget eventually shrink its displacement below that of the waiting gadget. More precisely, the ratio $x_A/x_B$ stays bounded and bounded away from zero. Section~\ref{sec:burst} proves this through a general competition theorem and also bounds the lengths of consecutive runs of either gadget.

\subsection{Irrational burst frequency, no finite cycle, uniform start}
\label{sec:arch-nocycle}
Let $N_A(n)$ and $N_B(n)$ be the numbers of bursts won by A and by B among the first $n$. The exponential rates of Section~\ref{sec:arch-halfstep} and the bounded displacement ratio give
\[
N_A(n)\log c_A-N_B(n)\log c_B=O(1).
\]
Since the two counts sum to $n$, the limiting A-frequency is
\[
\frac{N_A(n)}{n}\longrightarrow
\omega=\frac{-\log c_B}{-\log c_A-\log c_B}\approx0.4026.
\]
The norm argument shows that $\omega$ is irrational. An eventually periodic sequence of winners would have rational frequency, so the winners never settle into a periodic pattern (Proposition~\ref{prop:irrational-frequency}).

\begin{figure}[t]
\centering
\begin{tikzpicture}[
  >=Latex,
  font=\small,
  Acol/.style={blue!60!black},
  Bcol/.style={orange!75!black},
]

\def\W{5}     
\def\H{3.8}   
\def\dr{0.8}  

\draw[Acol, very thin, densely dotted, opacity=0.25] (0,0) -- (\W,0);
\draw[Acol, very thin, densely dotted, opacity=0.25] (0,\H) -- (\W,\H);
\draw[Bcol, very thin, densely dotted, opacity=0.25] (0,0) -- (0,\H);
\draw[Bcol, very thin, densely dotted, opacity=0.25] (\W,0) -- (\W,\H);

\foreach \x/\y in {0/0, 0/3.8, 5/0, 5/3.8}
  \draw[dashed, gray!45, thin] (\x,\y) circle (\dr);

\fill (0,0) circle (2.2pt)
  node[font=\footnotesize, below left=2pt] {$Q_{00}$};
\fill (\W,0) circle (2.2pt)
  node[font=\footnotesize, below right=2pt] {$Q_{10}$};
\fill (0,\H) circle (2.2pt)
  node[font=\footnotesize, above left=2pt] {$Q_{01}$};
\fill (\W,\H) circle (2.2pt)
  node[font=\footnotesize, above right=2pt] {$Q_{11}$};

\draw[<->, gray!60, very thin]
  ({\dr+0.03}, -0.15) -- ({\W-\dr-0.03}, -0.15)
  node[midway, below=1pt, font=\scriptsize, text=gray!60] {$\ge\delta$};

\node[Acol, font=\footnotesize] at (\W/2, -0.95)
  {\textnormal{A-phase} ($i$)};
\node[font=\scriptsize] at (0, -0.5) {$0$};
\node[font=\scriptsize] at (\W, -0.5) {$1$};
\node[Bcol, font=\footnotesize, rotate=90] at (-0.95, \H/2)
  {\textnormal{B-phase} ($j$)};
\node[font=\scriptsize] at (-0.5, 0) {$0$};
\node[font=\scriptsize] at (-0.5, \H) {$1$};


\coordinate (W0) at ($(0,0)+(40:0.52)$);
\coordinate (W1) at ($(0,\H)+(-30:0.43)$);
\coordinate (W2) at ($(\W,\H)+(200:0.36)$);
\coordinate (W3) at ($(\W,0)+(145:0.30)$);
\coordinate (W4) at ($(\W,\H)+(-55:0.24)$);
\coordinate (W5) at ($(0,\H)+(160:0.19)$);
\coordinate (W6) at ($(0,0)+(215:0.15)$);
\coordinate (W7) at ($(\W,0)+(-25:0.12)$);
\coordinate (W8) at ($(\W,\H)+(135:0.09)$);

\draw[->,Bcol,semithick,shorten >=1.5pt,shorten <=1.5pt]
  (W0) to[bend left=14] (W1);
\draw[->,Acol,semithick,shorten >=1.5pt,shorten <=1.5pt]
  (W1) to[bend left=10] (W2);
\draw[->,Bcol,semithick,shorten >=1.5pt,shorten <=1.5pt]
  (W2) to[bend left=10] (W3);
\draw[->,Bcol,semithick,shorten >=1.5pt,shorten <=1.5pt]
  (W3) to[bend left=18] (W4);
\draw[->,Acol,semithick,shorten >=1.5pt,shorten <=1.5pt]
  (W4) to[bend left=14] (W5);
\draw[->,Bcol,semithick,shorten >=1.5pt,shorten <=1.5pt]
  (W5) to[bend left=14] (W6);
\draw[->,Acol,semithick,shorten >=1.5pt,shorten <=1.5pt]
  (W6) to[bend left=10] (W7);
\draw[->,Bcol,semithick,shorten >=1.5pt,shorten <=1.5pt]
  (W7) to[bend left=12] (W8);

\foreach \k in {0,...,8}
  \fill[red!60!black] (W\k) circle (1.2pt);

\node[font=\tiny, red!60!black, right=2pt] at (W0) {$W_0$};
\node[font=\tiny, red!60!black, left=2pt] at (W8) {$W_8$};

\node[font=\small, gray!60] at ($(\W,\H)+(-0.25,-0.35)$) {$\cdots$};

\draw[->,blue!60!black,semithick] (6.3, 3.5) -- ++(0.5,0);
\node[blue!60!black, font=\scriptsize, right] at (6.9, 3.5) {A-burst};
\draw[->,orange!75!black,semithick] (6.3, 3.0) -- ++(0.5,0);
\node[orange!75!black, font=\scriptsize, right] at (6.9, 3.0) {B-burst};

\node[font=\small, anchor=north] at (\W/2, -1.5) {
  $\displaystyle\frac{N_A(n)}{n}\;\to\;
  \frac{-\log\kappa}{-\log\lambda-\log\kappa}
  \;\notin\;\mathbb{Q}
  \quad\Longrightarrow\quad
  \text{burst-winner sequence is aperiodic}$
};

\end{tikzpicture}
\caption{The four phase states of the product dynamics. An A-burst (blue) changes the A-phase, and a B-burst (orange) changes the B-phase. The states at burst boundaries approach these four points, but they switch between them with irrational frequency, which rules out convergence to a finite cycle.}
\label{fig:nocycle}
\end{figure}

\figref{fig:nocycle} shows the four phase states of the product system. To pass from the sequence of winners to the full weight trajectory, it suffices to follow a single coordinate: the fourth entry of the A-marginal. On the common orbit this entry alternates between two distinct values, $d_*$ and $1/2-d_*$, one per phase.

If the weights converged to a finite cycle, the phase of A read off from this entry would eventually be periodic. Since A wins a burst exactly when its phase changes, the sequence of winners would then be eventually periodic as well, contradicting the irrational frequency (Proposition~\ref{prop:no-cycle}).

For the uniform start, take the product distribution $a_0\otimes b_0$ with
\[
a_0=(7,18,12,13)/50,\qquad b_0=(8,17,11,13,1)/50.
\]
Following AdaBoost exactly for 140 updates from this start brings the state strictly inside the cones above, and every hypothesis selected along the way has strictly the largest edge. Duplicating each of the 20 rows of $M_0$ in proportion to its starting weight gives 2,500 equally weighted examples, and the total weight of each group of copies follows the trajectory on $M_0$ exactly (Lemma~\ref{lem:duplication}).

Because every selection during these 140 updates is strict, nearby product starts make the same selections and also end inside the cones, so the conclusion holds on an open neighborhood of $a_0\otimes b_0$ among product distributions. Throughout the trajectory, negated columns have edges strictly below the selected one, and adding them keeps the margin at $1/5$. Section~\ref{sec:start} gives both proofs.

%% file: figures/competition_roadmap.tex
\definecolor{clockA}{HTML}{245A92}
\definecolor{clockB}{HTML}{BA651D}
\definecolor{clockInk}{HTML}{253444}
\definecolor{clockRule}{HTML}{CAD1D8}
\begin{tikzpicture}[
 x=1cm,y=1cm,>=Latex,
 font=\small,
 title/.style={font=\small\bfseries,anchor=west},
 every node/.style={text=clockInk}
]
\node[title] at (0,3.6) {(a) Update the leader};
\node[title] at (7.1,3.6) {(b) Balance in log coordinates};

\draw[clockRule] (.75,2.45)--(6.25,2.45);
\draw[clockRule] (.75,1.30)--(6.25,1.30);
\draw[clockRule,dashed] (.75,.95)--(.75,2.85);
\node[anchor=east,text=clockA,font=\bfseries] at (.48,2.45) {A};
\node[anchor=east,text=clockB,font=\bfseries] at (.48,1.30) {B};
\node[anchor=north] at (.75,.91) {$0$};
\draw[clockA,fill=white,line width=.8pt] (5.85,2.45) circle (2.2pt);
\fill[clockA] (2.32,2.45) circle (2.2pt);
\node[text=clockA,anchor=south] at (5.85,2.58) {$x_A$};
\node[text=clockA,anchor=north] at (2.32,2.30) {$x_A'\simeq c_Ax_A$};
\draw[->,clockA,line width=1.1pt] (5.65,2.65) to[bend right=18] (2.52,2.65);
\fill[clockB] (4.40,1.30) circle (2.2pt);
\node[text=clockB,anchor=south] at (4.40,1.43) {$x_B$};
\node[text=clockB,anchor=north] at (4.40,1.16) {unchanged};
\node[anchor=north] at (3.45,.54) {displacement from the normalized fixed point};

\begin{scope}[shift={(8.05,.92)},x=.66cm,y=.34cm]
 \fill[clockA!8] (0,0)--(1.18,0)--(6.5,5.32)--(6.5,6.5)--(5.705,6.5)--(0,.795)--cycle;
 \draw[clockRule,dashed,line width=.5pt] (0,0)--(6.5,6.5);
 \draw[clockA,line width=1.3pt] (.25,.25)--(1.43,.25);
 \draw[clockB,line width=1.3pt] (1.43,.25)--(1.43,1.84);
 \draw[clockA,line width=1.3pt] (1.43,1.84)--(2.61,1.84);
 \draw[clockB,line width=1.3pt] (2.61,1.84)--(2.61,2.635);
 \draw[clockA,line width=1.3pt] (2.61,2.635)--(3.79,2.635);
 \draw[clockB,line width=1.3pt] (3.79,2.635)--(3.79,4.225);
 \draw[clockA,line width=1.3pt] (3.79,4.225)--(4.97,4.225);
 \draw[clockB,line width=1.3pt] (4.97,4.225)--(4.97,5.02);
 \draw[->,clockA,line width=1.3pt] (4.97,5.02)--(6.15,5.02);
 \node[align=center] at (1.75,5.15) {bounded log gap};
 \draw[->,clockInk,line width=.55pt] (0,0)--(6.9,0);
 \draw[->,clockInk,line width=.55pt] (0,0)--(0,6.95);
\end{scope}
\node[anchor=north] at (10.2,.65) {$X=-\log x_A$};
\node[rotate=90] at (7.52,2.10) {$Y=-\log x_B$};

\node at (6.45,-.68) {$\displaystyle
 \underbrace{N_A(n)a-N_B(n)b=O(1)}_{a=-\log c_A,\ b=-\log c_B}
 \quad\Longrightarrow\quad
 \frac{N_A(n)}{n}\longrightarrow\frac{b}{a+b}\notin\mathbb Q$};
\end{tikzpicture}

%% file: sections/05_verification.tex
\section{The certificate}
\label{sec:verification}

\subsection{Finite proof obligations}
\label{sec:verification-scope}
The certificate verifies the finite identities and inequalities used in \appref{app:proof}, and \appref{app:certificate} lists all sixteen of them. The main ones are the rational branch maps and the reflection identity relating them; the exact root that defines the common orbit, together with its spectral data; the Taylor bounds for the two cones; the edge gaps used for locking and for excluding negated hypotheses; and the first 140 updates from the rational start. The 140 updates are computed on all twenty weights of $M_0$, and at each one the certificate compares the eight columns of $M_0$ and their negations (16 in total). Explicit finite witnesses establish the 2,500-row encoding and the margin $1/5$.

The rest of the argument is proved analytically in the paper: the invariance of the cones given the Taylor bounds, the convergence of the rescaled displacements, the competition theorem, the argument that follows a single coordinate, the aggregation of duplicated rows, and the extension to an open neighborhood of product starts.

\subsection{Exact arithmetic and reproducibility}
\label{sec:verification-exact}
Every check uses symbolic algebra over $\mathbb Q$ and exact rational intervals, and each division is preceded by a check that the denominator cannot vanish. Algebraic roots are identified by exact polynomial identities and root counts. In the computation of the first 140 updates, interval endpoints are rounded outward to a dyadic grid, which can only enlarge the enclosures.

The main script, \texttt{certificate/verify.py}, uses SymPy~\citep{meurer2017sympy}. It derives the maps directly from the displayed matrices (through \texttt{algebra.py}) and saves its results together with hashes of the source files. A second implementation in SageMath~\citep{stein2005sage} reproduces the inequalities near the orbit and along the first 140 updates using Sage's own algebra. The two implementations follow the same mathematical specification and share part of the checking code. This protects against library- and implementation-specific errors, although a mistake in the specification itself would carry over to both. For this reason, Table~\ref{tab:alignment} in \appref{app:certificate} pairs each of the sixteen statements with the lemma that uses it, so the correspondence between the certificate and the proof can be checked item by item. \appref{app:certdoc-impl} describes what each implementation covers and how to rerun the checks.

%% file: sections/09_discussion.tex
\section{Discussion}
\label{sec:discussion}

\subsection{Implications for boosting theory}
\label{sec:disc-boosting}
The counterexample shows that a positive weak-learning margin alone does not force the weights of exhaustive AdaBoost to converge to a finite cycle. Any positive cycling theorem must restrict the class of instances or weaken the conclusion. Part of the motivation for the cycling question was that if the weights always cycle, questions about convergence and margins would reduce to analyzing that cycle, and this counterexample shows that approach does not work in general. It is true that the direct practical implications are more limited given the nature of how this counterexample was constructed.

The result also covers the binary case of AdaBoost.R~\citep[Lemma~1]{NockNielsen2007}. When the weak hypotheses take values in $\{-1,+1\}$, its weight update coincides with that of discrete AdaBoost, so with the same exhaustive selection rule our instance produces the same trajectory. Extending the construction to real-valued hypotheses or to other boosting updates would require new local maps and a new verification. \appref{app:final-evidence} records the identity for the binary case.

\paragraph{Scope.}
The cone radius $10^{-5}$ is simply a radius at which the local estimates can be certified, and it says little about how large the set of noncycling starting distributions actually is. The open neighborhood in Theorem~\ref{thm:main} lies within the seven-dimensional family of product distributions, and our results do not cover nearby correlated distributions. \appref{app:final-evidence} describes the search that found the rational start.

\subsection{Considerations and limitations of AI collaboration}
\label{sec:disc-considerations}
The workflow of Section~\ref{sec:workflow} produced a correct and certified result, but it took a great deal of time and human effort to run. These are important practical costs. Table~\ref{tab:workflow-stats} breaks down the GPT-5.4 Pro usage from the original construction and verification by activity. Construction and refinement took about half of the total time, much of it spent on repeated passes between the first candidate counterexample and a statement that the certificate could actually check. Exploration and failed attempts came second. Verification had the fewest prompts, but its responses were the longest on average, at 35m~44s compared with 25m~45s for exploration.

\begin{table}[ht]
\centering
\caption{GPT-5.4 Pro usage by activity from the saved conversations.}
\label{tab:workflow-stats}
{\small
\begin{tabular}{lrrrr}
\toprule
Activity & Conversations & Prompts & Timed responses & Total time\\
\midrule
Exploration / failed attempts & 15 & 45 & 44 & 18h 53m 02s\\
Construction / refinement & 20 & 67 & 66 & 32h 41m 17s\\
Verification & 11 & 21 & 21 & 12h 30m 17s\\
\midrule
Total & 46 & 133 & 131 & 64h 04m 36s\\
\bottomrule
\end{tabular}
}
\end{table}

\paragraph{Consumer subscriptions and accessibility.}
All of the original work was done on consumer subscriptions to ChatGPT Pro and Claude Pro, through their ordinary web interfaces. The research was organized by hand: continuing conversations, opening fresh ones for critiques, and copy-pasting summaries from one model to the other. An API-based agent with shared memory across solver and critic calls would probably have saved much of this manual work, and might well have been more effective. We did not go that route mainly because API costs have no ceiling, whereas a subscription has a fixed monthly price. This mattered a lot when exploring directions that had unclear value a priori.

Importantly, this way of working is easy to try for mathematicians and scientists who do not specialize in AI engineering. What is needed is a research question, a sequence of conversations, and a researcher who decides which arguments are worth pursuing and checking. Verifying this particular result did require computer algebra, but we regard that as part of the mathematics.

\paragraph{API costs and compute footprint.}
The web interface does not show token counts, so \appref{app:cost} estimates the output-token cost from the response times in Table~\ref{tab:workflow-stats}. At API prices this comes to about \$1,250, but we reiterate that the figure is hypothetical since the API was never used, and converting thinking time to tokens is only approximate. It also leaves out input tokens, all Claude usage, the two untimed responses, and the later work with GPT-5.6 Sol and GPT-6 Astra, all of which would add to the total.

\paragraph{Wait times and their effect on the workflow.}
Timed GPT-5.4 Pro responses took 29m~21s on average, and the longest took around two hours. With waits this long, every choice about which direction to pursue or drop cost real time, so it made sense to test several directions in parallel. This is also part of why the separate critic role was useful. Claude Opus 4.6 responded much faster, so a fresh critic conversation could summarize the work so far and suggest the next direction while solver attempts were still running, even though its raw reasoning was weaker than that of GPT-5.4 Pro.

\paragraph{The human review burden.}
Asking the models to critique the proposed solution cut down on false claims of correctness, but it did not catch every error. Reviewing the argument and the certificate took longer than producing the solution with the models itself, primarily because of gaps and ambiguities in the model's presentations. The review itself was extremely tedious due to how long and intricate the argument is, and how closely tied it is to the certificate.

Still, this was far less work than producing the proof from scratch would have been—after all, the problem had been open for nearly 15 years. GPT-5.4 Pro's raw reasoning capabilities were impressive. It produced a sophisticated block-product counterexample whose pieces fit together exactly as the argument requires, and it is extremely unlikely that this construction would have been found without the models in the same timeframe.

The main takeaway is that this kind of collaboration with LLMs does not remove the need for expert human attention. It mostly shifts the human effort from generating arguments to verifying them, which can still be a lot of work and calls for its own kind of expertise. The researcher has to judge what the models produce, decide which directions to prioritize, and find the gaps that the models miss. Without this, it is easy to end up with mathematics that looks highly plausible but is wrong.

%% file: sections/17_acknowledgements.tex
\section*{Acknowledgments}
We thank Edgar Dobriban, Indraneel Mukherjee, and Robert Schapire for feedback on the solution and manuscript. GPT-5.4 Pro, Claude Opus 4.6, GPT-5.6 Sol, and GPT-6 Astra contributed most of the technical work for the counterexample and in writing up the solution for the manuscript.

%% file: sections/03_gadgets.tex
\subsection{The gadgets, the product system, and the exact invariant affine manifolds}
\label{sec:gadgets}

The counterexample is built from two small sign matrices,
which we call \emph{gadgets}. Let
\[
L_A=
\begin{pmatrix}
1&1&-1&-1\\
-1&1&1&1\\
1&-1&-1&1\\
1&-1&1&-1
\end{pmatrix},
\qquad
L_B=
\begin{pmatrix}
1&1&-1&-1\\
-1&1&1&1\\
1&-1&-1&1\\
1&-1&1&-1\\
1&-1&1&1
\end{pmatrix}.
\]

The fifth row of $L_B$ is obtained from either the third or the fourth row of $L_A$ by changing a single $-1$ to $+1$. Relative to these four hypotheses it is an easier example, and when its weight is zero the two learning problems coincide.

Define the $20\times 8$ block-product matrix $M_0=L_A\boxplus L_B$ as follows (see Figure \ref{fig:m0-boxplus}). 
Its rows are indexed by pairs
\(
(r,s)\in\mathcal{I}:=\{1,2,3,4\}\times\{1,2,3,4,5\},
\)
and its columns are denoted
\(
A_1,\dots,A_4,B_1,\dots,B_4.
\)
The entries are
\[
M_0\bigl((r,s),A_j\bigr)=(L_A)_{rj},
\qquad
M_0\bigl((r,s),B_j\bigr)=(L_B)_{sj}.
\]

\begin{lemma}[Reduction to row marginals]
\label{lem:marginals}
Let $W=(w_{rs})_{(r,s)\in \mathcal{I}}$ be any distribution on the $20$ rows of $M_0$, and define its row marginals
\[
a_r=\sum_{s=1}^5 w_{rs},
\qquad
b_s=\sum_{r=1}^4 w_{rs}.
\]
Then for
the \emph{edge vector} $W^\top M_0\in\mathbb{R}^8$, whose $j$th component $(W^\top M_0)_j=\sum_{(r,s)}w_{rs}(M_0)_{(r,s),j}$ is the 
inner product of $W$ with column $j$, 
the A-block edges depend only on $a$, and the B-block edges depend only on $b$: 
\begin{align*}
(W^\top M_0)_{A_j}&=(a^\top L_A)_j\qquad(1\le j\le 4), \\
(W^\top M_0)_{B_j}&=(b^\top L_B)_j\qquad(1\le j\le 4).
\end{align*}
\end{lemma}

\begin{proof}
For $1\le j\le 4$,
\[
(W^\top M_0)_{A_j}
=\sum_{r=1}^4\sum_{s=1}^5 w_{rs}(L_A)_{rj}
=\sum_{r=1}^4 a_r (L_A)_{rj}
=(a^\top L_A)_j.
\]
The B-block identity is similar.
\end{proof}

\begin{lemma}[Exact invariance of the product manifold]
\label{lem:invariance}
Let
\[
\mathcal P=\{a\otimes b:\ a\in\Delta^3,\ b\in\Delta^4\}.
\]
If $W=a\otimes b\in\mathcal P$ and AdaBoost selects an A-column $A_j$, then the updated state is
\(
W'=a'\otimes b,
\)
where $a'$ is the one-step exhaustive-AdaBoost update of $a$ on $L_A$ using column $j$. If AdaBoost selects a B-column $B_j$, then
\(
W'=a\otimes b',
\)
where $b'$ is the one-step update of $b$ on $L_B$ using column $j$.
Hence $\mathcal P$ is \emph{invariant} under the AdaBoost dynamics on $M_0$: if $W\in\mathcal P$, then the updated state $W'\in\mathcal P$.
\end{lemma}

\begin{proof}
Assume an A-column $j$ is selected. By Lemma~\ref{lem:marginals}, the selected edge is
\(
\mu=(a^\top L_A)_j.
\)
Therefore
\[
w'_{rs}=\frac{a_r b_s}{1+\mu \cdot (L_A)_{rj}}
=\left(\frac{a_r}{1+\mu \cdot (L_A)_{rj}}\right)b_s
=a_r' b_s,
\]
where 
$a'$ is the one-step AdaBoost update of $a$ on $L_A$. The B-column case is identical.
\end{proof}

We will show  in Section~\ref{sec:locking}
that 
the product dynamics over $\mathcal P$
decouples
into alternating blocks of five consecutive AdaBoost steps in which the same gadget is selected at every step. 
We call each such block of five steps a \emph{burst}. Within a burst, the column-selection sequence on the active gadget is one of two possible 5-tuples, called \emph{half-words},
\begin{equation*}
H_0=(1,3,4,1,2),
\qquad
H_1=(1,4,3,1,2).
\end{equation*}
The choice between these is
determined by the \emph{phase}: if the active gadget's current state is near the first point $\hat P_0^G$ of a certain 2-cycle (Section~\ref{sec:2cycle}), we say it is in phase~$0$; if near $\hat P_1^G$, in phase~$1$.

Since both $H_0$ and $H_1$ end in column~$2$, and column~$2$ of $L_A$ is $(1,1,-1,-1)^\top$, 
the classical balancing lemma for AdaBoost (see \citealt{SF2012} and Section \ref{aux})
implies that after the final step of any A-burst,
the total mass on the $+1$ rows (rows $1,2$) and on the $-1$ rows (rows $3,4$) is each exactly $\tfrac12$. Writing the distribution as $(a,b,c,d)\in\Delta^3$, this means $a+b=\tfrac12$ and $c+d=\tfrac12$. 
Hence every 5-step A-burst lands in
\[
\Sigma_A=\{(a,b,c,d)\in\Delta^3:\ a+b=\tfrac12,\ c+d=\tfrac12\}.
\]
Similarly, column~$2$ of $L_B$ is $(1,1,-1,-1,-1)^\top$, so Lemma~\ref{lem:landing} gives $a+b=\tfrac12$ and $c+d+e=\tfrac12$. Every 5-step B-burst lands in
\[
\Sigma_B=\{(a,b,c,d,e)\in\Delta^4:\ a+b=\tfrac12,\ c+d+e=\tfrac12\}.
\]
Thus, for some $p,d$, we can represent 
\begin{align*}
(a,b,c,d)&=\Bigl(p-\tfrac12,\ 1-p,\ \tfrac12-d,\ d\Bigr)\in\Sigma_A, \\
(a,b,c,d,e)&=\Bigl(p-\tfrac12,\ 1-p,\ \tfrac12-d-e,\ d,\ e\Bigr)\in\Sigma_B.
\end{align*}

For each gadget $G\in\{A,B\}$ and each half-word $H_i$, composing the five consecutive AdaBoost update steps prescribed by $H_i$ yields a rational map $\Sigma_G\to\Sigma_G$ called the \emph{branch map} $F_i^G$ (explicit formulas are given in Section~\ref{sec:branch}). Item~2 of the certificate (Appendix~\ref{app:certificate}) verifies symbolically that each numerator polynomial of the
last
coordinate of $F_i^B$ is divisible, as a polynomial in $\mathbb{Q}[p,d,e]$, by the variable~$e$. Consequently, the \emph{face} $\{e=0\}\subset\Sigma_B$---i.e., the subset of $\Sigma_B$ obtained by setting the fifth simplicial coordinate to zero---is invariant under both branch maps, and restriction to this face recovers the A-gadget maps. Concretely, the following structural identity holds (which Item~2 of the certificate also verifies symbolically).

Let $U_j^G$ denote the one-step exhaustive AdaBoost update on $L_G$ using column $j$, and for a finite word $w=(j_1,\dots,j_k)$ in the four columns, define the composition $T_w^G = U_{j_k}^G\circ\cdots\circ U_{j_1}^G$.

\begin{lemma}[Zero-mass face and restriction of $L_B$ to $L_A$]
\label{lem:face-restriction}
Let
\[
\iota:\Delta^3\to\Delta^4,\qquad
\iota(a_1,a_2,a_3,a_4)=(a_1,a_2,a_3,a_4,0).
\]
For every finite word $w$ in the four columns and every $a\in\Delta^3$ for which $T_w^A(a)$ is well-defined---that is, for which every denominator $1+\mu\cdot(L_A)_{rj}$ encountered along the word is nonzero (in particular, for every strictly positive $a$, i.e., $a$ in the open simplex)---we have
\[
T_w^{B}\circ\iota(a)=\iota\circ T_w^{A}(a),
\]
and the right-hand side determines a well-defined $T_w^B(\iota(a))$ via the same denominator condition.
Consequently, in the affine charts
\[
a(p,d)=\left(p-\tfrac12,\ 1-p,\ \tfrac12-d,\ d\right),
\]
\[
b(p,d,e)=\left(p-\tfrac12,\ 1-p,\ \tfrac12-d-e,\ d,\ e\right),
\]
the branch maps satisfy, on the regular domain $\Sigma_A^\circ$ of Section~\ref{sec:branch}, $F_i^B(p,d,0)=\bigl(F_i^A(p,d),0\bigr)$.
In particular,
\[
F_0^B|_{e=0}=F_0^A,
\qquad
F_1^B|_{e=0}=F_1^A.
\]
\end{lemma}

\begin{proof}
It suffices to prove the one-step claim under the regularity hypothesis (the multi-step case follows by induction on $w$, where regularity of each prefix is the inductive hypothesis applied at each step). Start $L_B$ from $\iota(a)$. For column $j$,
\[
\mu_B
=
\sum_{r=1}^5 \iota(a)_r(L_B)_{rj}
=
\sum_{r=1}^4 a_r(L_A)_{rj}
=
\mu_A,
\]
because the first four rows of $L_B$ equal $L_A$, and the fifth coordinate has mass zero.

For $1\le r\le 4$,
\[
\frac{\iota(a)_r}{1+\mu_B(L_B)_{rj}}
=
\frac{a_r}{1+\mu_A(L_A)_{rj}},
\]
which is the $r$-th coordinate of the $L_A$ update, while the fifth coordinate remains $0/(1+\mu_B(L_B)_{5j})=0$.
Induction over the word $w$ proves the word-map identity. The chart identity is exactly this word-map identity expressed in $(p,d,e)$ coordinates.
\end{proof}

\textbf{Notation convention.} Throughout, we distinguish \emph{affine chart points} from \emph{simplex probability vectors} using a 
hat. We write
\[
\hat P_i^A=(p_i,d_i)\in\mathbb R^2,
\qquad
\hat P_i^B=(p_i,d_i,e_i)\in\mathbb R^3
\]
for chart coordinates on $\Sigma_A$ and $\Sigma_B$, and
\(
P_i^A\in\Delta^3,
\,
P_i^B\in\Delta^4
\)
for the corresponding simplex probability vectors obtained from the affine parametrizations above. Similarly, at burst boundary $n$ (i.e., the state immediately before the $n$th burst begins), we write $\hat a_n,\hat b_n$ for chart coordinates and $a_n,b_n$ for the simplex vectors. We define $W_n=a_n\otimes b_n\in\mathcal P$ to be the product distribution on the 20 rows of $M_0$ at the $n$th burst boundary of the certified tail. For the rational initialization of Section~\ref{sec:start}, this tail begins after 140 AdaBoost updates.

%% file: sections/04_margin.tex
\subsection{Margin}
\label{sec:margin}

We show using a single linear witness that the weak-learning margin of the block-product matrix $M_0$ is exactly $1/5$. Let $\beta=(2,1,1,1)^\top\in\mathbb{R}^{1\times 4}$. Every row $r$ of $L_A$ satisfies $r\beta=1$, and every row $r$ of $L_B$ satisfies $r\beta\in\{1,3\}$. 
For any gadget $G\in\{A,B\}$ and any distribution $u$ on the rows of $L_G$, write $\mu_j=(u^\top L_G)_j$ for the edge of column~$j$.
Then
\[
2\mu_1+\mu_2+\mu_3+\mu_4=u^\top L_G\,\beta\ge 1,
\]
so $\max_j\mu_j\ge 1/5$. Hence
\(
\gamma(L_A),\gamma(L_B)\ge \frac15.
\)

Equality for $L_A$ is attained at
\[
u_*=(1/5,2/5,1/5,1/5),
\qquad
u_*^\top L_A=(1/5,1/5,1/5,1/5).
\]
It follows that $\gamma(L_A)=1/5$. Since the first four rows of $L_B$ coincide with $L_A$, the same
argument gives $\gamma(L_B)\le 1/5$, and together with the lower bound above,
\(
\gamma(L_A)=\gamma(L_B)=\frac15.
\)

By Lemma~\ref{lem:marginals}, for any distribution $W$ on the rows of $M_0$ with marginals $(a,b)$,
\[
\max_{1\le j\le 8}(W^\top M_0)_j
=
\max\Bigl(\max_{1\le j\le 4}(a^\top L_A)_j,\ \max_{1\le j\le 4}(b^\top L_B)_j\Bigr).
\]
Conversely, every pair of marginals $(a,b)$ is realized by the product distribution $a\otimes b$. 
Hence
\[\gamma(M_0)
=\min_W \max_j (W^\top M_0)_j
=\min_{a,b}\max\Bigl(\max_j(a^\top L_A)_j,\ \max_j(b^\top L_B)_j\Bigr)
=\frac15.\]

We now formalize the conclusions of the preceding discussion as a general block-product margin identity, followed by the concrete margin computation for $L_A$, $L_B$, and $M_0$.

\begin{lemma}[Block-product margin]
\label{lem:block-margin}
Let $C=A\boxplus B$. Then
\[
\gamma(C)=\max(\gamma(A),\gamma(B)).
\]
\end{lemma}

\begin{proof}
For any distribution $W$ on product rows, let $a,b$ be its marginals. By Lemma~\ref{lem:marginals},
\[
\max_j(W^\top C)_j
=
\max\Bigl\{
\max_j(a^\top A)_j,\
\max_k(b^\top B)_k
\Bigr\}.
\]
Note that
\[
\max_j(a^\top A)_j\ge\gamma(A),\ \max_k(b^\top B)_k\ge\gamma(B)
\implies
\max_j(W^\top C)_j\ge \max(\gamma(A),\gamma(B)).
\]
Taking the minimum over $W$, we get $\gamma(C)\ge \max(\gamma(A),\gamma(B))$.

\noindent Conversely, because the simplices are compact and the objective functions are continuous, choose minimizers $a_*,b_*$ for $A$ and $B$, and set $W_*=a_*\otimes b_*$. Then
\[
\max_j(W_*^\top C)_j=\max(\gamma(A),\gamma(B))
\implies
\gamma(C)\le \max(\gamma(A),\gamma(B)),
\]
and equality holds.
\end{proof}

\begin{proposition}[Concrete margin of $L_A$, $L_B$, $M_0$]
\label{prop:concrete-margin}
For the concrete matrices $L_A,L_B$,
\[
\gamma(L_A)=\gamma(L_B)=\frac15,
\]
and therefore $\gamma(M_0)=\tfrac15$ for $M_0=L_A\boxplus L_B$.
\end{proposition}

\begin{proof}
Let $\beta=(2,1,1,1)^\top$. For $L_A$, every row $r$ satisfies $r\beta=1$: indeed,
\[
(1,1,-1,-1)\beta=2+1-1-1=1,
\]
\[
(-1,1,1,1)\beta=-2+1+1+1=1,
\]
\[
(1,-1,-1,1)\beta=2-1-1+1=1,
\]
\[
(1,-1,1,-1)\beta=2-1+1-1=1.
\]
For $L_B$, the first four rows have the same dot products, and the fifth row satisfies
\[
(1,-1,1,1)\beta=2-1+1+1=3.
\]

Let $q=\beta/5$. Then $q$ is a column distribution, and $(L_A q)_r=\tfrac15$ for every row of $L_A$, while $(L_B q)_r\ge\tfrac15$ for every row of $L_B$. By the minimax theorem, $\gamma(L_A)\ge\tfrac15$ and $\gamma(L_B)\ge\tfrac15$.

\noindent For the reverse inequalities, take row distributions
\[
u_A=\left(\tfrac15,\tfrac25,\tfrac15,\tfrac15\right),
\]
\[
u_B=\left(\tfrac15,\tfrac25,\tfrac15,\tfrac15,0\right).
\]
Direct multiplication gives $u_A^\top L_A=\left(\tfrac15,\tfrac15,\tfrac15,\tfrac15\right)$ and $u_B^\top L_B=\left(\tfrac15,\tfrac15,\tfrac15,\tfrac15\right)$, so $\gamma(L_A)\le\tfrac15$ and $\gamma(L_B)\le\tfrac15$. Thus $\gamma(L_A)=\gamma(L_B)=\tfrac15$. By Lemma~\ref{lem:block-margin},
\[
\gamma(M_0)=\max(\gamma(L_A),\gamma(L_B))=\tfrac15.
\]
\end{proof}

%% file: sections/05_branch_maps.tex
\subsection{Exact 5-step branch maps}
\label{sec:branch}

We now compute the closed-form rational maps that result from composing five consecutive AdaBoost updates along each half-word. The two branch maps $F_0^G$ and $F_1^G$ for each gadget $G\in\{A,B\}$ are related by a simple involution. Define the \emph{phase-swap} maps
\[
S_A(p,d)=\bigl(p,\tfrac12-d\bigr),
\qquad
S_B(p,d,e)=\bigl(p,\tfrac12-d-e,\ e\bigr).
\]
Each $S_G$ is an involution ($S_G^2=\mathrm{id}$). 
Using simple algebra,
one can check that 
in simplex coordinates, 
both maps $S_G$ \emph{transpose the 3rd and 4th distribution entries} 
while leaving the rest fixed. 

For either gadget, rows~$3$ and~$4$ of $L_G$ differ only in columns~$3$ and~$4$ (where their signs are opposite), so this row transposition induces the column transposition $\pi=(3\ 4)$ on the edge vector. If $T_j^G$ denotes the one-step AdaBoost update using column~$j$ of~$L_G$, then $S_G\circ T_j^G\circ S_G=T_{\pi(j)}^G$. Because $S_G^2=\mathrm{id}$, conjugating the five-step composition $F_0^G$ by $S_G$ replaces each column index $j$ with $\pi(j)$, transforming $H_0=(1,3,4,1,2)$ into $H_1=(1,4,3,1,2)$. Hence,
\[
F_1^A=S_A\circ F_0^A\circ S_A,
\qquad
F_1^B=S_B\circ F_0^B\circ S_B.
\]
It therefore suffices to specify $F_0^A$ and $F_0^B$. The following finite recipe defines their rational functions exactly, without expanded multivariate coefficient lists.

\subsection*{A-gadget}
Start with $a(p,d)=(p-\tfrac12,1-p,\tfrac12-d,d)$. For each successive index of $H_0=(1,3,4,1,2)$, form $\mu=\sum_r a_r(L_A)_{rj}$ and replace $a_r$ by $a_r/(1+\mu(L_A)_{rj})$. From the final vector $a'$, extract $p'=1-a'_2$ and $d'=a'_4$. Reduce each rational coordinate to lowest terms and denote it by
\[
F_0^A(p,d)=\left(\frac{P_A(p,d)}{Q_A(p,d)},\frac{D_A(p,d)}{\mathcal R_A(p,d)}\right).
\]

\subsection*{B-gadget}
Apply the same five updates to $b(p,d,e)=(p-\tfrac12,1-p,\tfrac12-d-e,d,e)$ using $L_B$. Extract $p'=1-b'_2$, $d'=b'_4$, and $e'=b'_5$. Reducing gives
\[
F_0^B(p,d,e)=\left(\frac{P_B}{Q_B},\frac{D_B}{\mathcal R_B},\frac{E_B}{\mathcal R_B}\right)(p,d,e).
\]
The common denominator for the last two coordinates and the factor $e$ in $E_B$ are verified by the symbolic derivation (Items 1--2, Appendix~\ref{app:certificate}). These instructions determine every map used below from the displayed matrices and update rule. They are also an exact specification for the certificate.

\begin{lemma}[Regularity of the branch maps]\label{lem:regularity}
The reduced denominators are nonzero on
\[
\Sigma_A^\circ=\{\tfrac12<p<1,\ 0<d<\tfrac12\},\qquad
\Sigma_B^\circ=\{\tfrac12<p<1,\ d>0,\ e>0,\ d+e<\tfrac12\}.
\]
The maps and all derivatives are smooth there. They are also regular on the face $e=0$ whenever the first four coordinates are positive.
\end{lemma}
\begin{proof}
Each column of $L_A$ has both signs. With all four coordinates positive, its edge lies strictly between $-1$ and $1$, so the update denominators are positive and positivity is preserved. The same holds for $L_B$, including its $e=0$ face. Induction gives a rational five-step composition with a denominator that never vanishes on these domains. A reduced denominator divides an unreduced one, so it cannot introduce a zero there. Smoothness follows.
\end{proof}

%% file: sections/06_2cycle.tex
\subsection{Exact common 2-cycle}
\label{sec:2cycle}

We now argue that
both gadgets share an exact period-2 orbit for the 5-step branch maps. By the conjugacy $F_1^A=S_A\circ F_0^A\circ S_A$ (Section~\ref{sec:branch}), if $F_0^A$ maps a point $(p,d)$ to its phase-swap $S_A(p,d)=(p,\tfrac12-d)$, then $F_1^A$ maps the phase-swap back: $F_1^A(p,\tfrac12-d)=S_A(F_0^A(p,d))=(p,d)$, giving a 2-cycle. We therefore seek solutions of
\[
F_0^A(p,d)=\bigl(p,\tfrac12-d\bigr).
\]
Substituting the explicit rational formulas for $F_0^A$ (Section~\ref{sec:branch}), the two coordinate equations $P_A/Q_A=p$ and $D_A/\mathcal R_A=\frac{1}{2}-d$ become, after cross-multiplying and removing factors that are nonzero on $\Sigma_A^\circ$, the polynomial system $N_1(p,d)=0,$ $N_2(p,d)=0$, where
\begin{align*}
N_1(p,d)={}&
2 d^{4} - 2 d^{3} p + d^{3} - 12 d^{2} p^{2} + 7 d^{2} p - d^{2}
- 2 d p^{3} + 6 d p^{2} - 2 d p + 2 p^{4} - p^{2},
\end{align*}
and
\begin{align*}
N_2(p,d)={}&
- 16 d^{5} + 8 d^{4} + 48 d^{3} p^{2} - 16 d^{3} p + 2 d^{3}\notag\\
&- 32 d^{2} p^{2} + 10 d^{2} p - d^{2}
- 8 d p^{3} + 10 d p^{2} - 2 d p + 2 p^{3} - p^{2}.
\end{align*}
To eliminate $d$ from the system $N_1=N_2=0$, we compute the resultant of $N_1$ and $N_2$ with respect to~$d$. By standard elimination theory~\citep{cox1997ideals}, any common root $(p,d)$ must satisfy $\operatorname{Res}_d(N_1,N_2)(p)=0$. Item~3 of the certificate (Appendix~\ref{app:certificate}) verifies that
\[
\operatorname{Res}_d(N_1,N_2)=4096\,p^{10}(2p-1)^3 f(p),
\]
where
\[
f(p)=1024p^7-3072p^6+3200p^5-1152p^4-144p^3+176p^2-40p+7.
\]
The factors $p^{10}(2p-1)^3$ are nonzero on $\Sigma_A^\circ$ (since $p>\tfrac12$), so any interior solution must satisfy $f(p)=0$.

The certificate 
(Item~3 of Appendix~\ref{app:certificate}) isolates a unique root
\[
p_*\in[0.637116837818846,\ 0.637116837818848],
\]
and proves
\begin{align*}
d_*&=512 p_*^{6} - 1152 p_*^{5} + 736 p_*^{4} - 24 p_*^{3} - 88 p_*^{2} + 19 p_* - \tfrac92 \notag\\
&\in[0.25799538295,\ 0.25799538296].
\end{align*}

In chart coordinates, set
\[
\hat P_0^A=(p_*,\,d_*),
\qquad
\hat P_1^A=(p_*,\,\tfrac12-d_*).
\]
The corresponding simplex vectors are
\begin{equation}\label{p0a}
P_0^A=\Bigl(p_*-\tfrac12,\ 1-p_*,\ \tfrac12-d_*,\ d_*\Bigr),
\qquad
P_1^A=\Bigl(p_*-\tfrac12,\ 1-p_*,\ d_*,\ \tfrac12-d_*\Bigr).
\end{equation}
Then
\(
F_0^A(\hat P_0^A)=\hat P_1^A,
\,
F_1^A(\hat P_1^A)=\hat P_0^A.
\)
In chart coordinates, set
\[
\hat P_0^B=(p_*,\,d_*,\,0),
\qquad
\hat P_1^B=(p_*,\,\tfrac12-d_*,\,0),
\]
with corresponding simplex vectors
\begin{equation}\label{p0b}
P_0^B=(P_0^A,0),
\qquad
P_1^B=(P_1^A,0).
\end{equation}
Since $F_i^B|_{e=0}=F_i^A$, the same identities hold for B:
\(
F_0^B(\hat P_0^B)=\hat P_1^B,
\,
F_1^B(\hat P_1^B)=\hat P_0^B.
\)

Thus both gadgets share an exact period-2 orbit for the 5-step branch maps, equivalently a period-10 orbit in AdaBoost time. We say gadget $G$ is in \emph{phase~$i$} ($i\in\{0,1\}$) when its state lies in the sector $S_i^G$ (Section~\ref{sec:sectors}); within these sectors, the column-selection sequence is exactly $H_i$ (Section~\ref{sec:locking}), so the 5-step composition is the branch map $F_i^G$.

%% file: sections/07_spectra.tex
\subsection{Exact spectra and the irrationality input}
\label{sec:spectra}

Denote by $J_i^G=D F_i^G(\hat P_i^G)$ the Jacobian of the branch map $F_i^G$ at the 2-cycle point $\hat P_i^G$, and define the \emph{return maps}
\[
R_A=J_1^A J_0^A,
\qquad
R_B=J_1^B J_0^B.
\]
Each $R_G$ is the linearization of one full period (phase~$0\to 1\to 0$) of the 2-cycle for gadget~$G$. We call the binary sequence recording which gadget (A or B) is selected at each successive burst the \emph{burst-winner sequence} (defined precisely in Section~\ref{sec:burst}). 

The key algebraic input of this section is that the dominant eigenvalues of $R_A$ and $R_B$ have an irrational logarithmic ratio, which will ultimately force the burst-winner sequence to be aperiodic. We first observe that the return map $R_B$ has a block-triangular structure.

\begin{lemma}[Block-triangular B-return map]
\label{lem:block-triangular}
Let $P_0,P_1$ be the common two-cycle points on the face $e=0$, and define
\begin{align*}
J_i^A&=D F_i^A(P_i),
&
J_i^B&=D F_i^B(P_i,0), \\
R_A&=J_1^A J_0^A,
&
R_B&=J_1^B J_0^B.
\end{align*}
Then for some column vector $u$ and scalar $\kappa$,
\[
R_B=
\begin{pmatrix}
R_A & u\\
0 & \kappa
\end{pmatrix}.
\]
\end{lemma}

\begin{proof}
By Lemma~\ref{lem:face-restriction}, $F_i^B(p,d,0)=(F_i^A(p,d),0)$, so the upper-left $2\times 2$ block of $J_i^B$ is $J_i^A$.

It remains to prove the lower-left block is zero. In one $L_B$-update, the fifth coordinate has the form $e'=e/(1+\mu\sigma)$ for some $\sigma\in\{\pm1\}$, where $\mu$ is the selected edge. Along any fixed finite word, repeated substitution gives $e_{\mathrm{out}}=e\,G_i(p,d,e)$ for a rational function $G_i$ regular at the two-cycle point, so $\partial e_{\mathrm{out}}/\partial p(P_i,0)=\partial e_{\mathrm{out}}/\partial d(P_i,0)=0$. Thus each $J_i^B$ has block form
\[
J_i^B=
\begin{pmatrix}
J_i^A & v_i\\
0 & \eta_i
\end{pmatrix},
\]
and multiplying $J_1^B J_0^B$ gives
\[
R_B=
\begin{pmatrix}
J_1^A J_0^A & J_1^A v_0+v_1\eta_0\\
0 & \eta_1\eta_0
\end{pmatrix}
=
\begin{pmatrix}
R_A & u\\
0 & \kappa
\end{pmatrix}.
\]
\end{proof}

\begin{proposition}[Certified spectral data]
\label{prop:spectral-data}
There are algebraic numbers $\nu,\lambda,\kappa$ such that $\operatorname{spec}(R_A)=\{\lambda,\nu\}$, and by Lemma~\ref{lem:block-triangular}, $\operatorname{spec}(R_B)=\{\lambda,\nu,\kappa\}$, with $0<\nu<\lambda<\kappa<1$. The numbers $\lambda,\nu$ are roots of the irreducible polynomial $P_\lambda$ displayed below, with
\begin{align*}
\lambda&\in[0.0944920594891271,\ 0.0944920594891272],\\
\nu&\in[0.0011453882842312,\ 0.0011453882842313].
\end{align*}
The number $\kappa$ is the unique root in $[0.2038817328642625,\ 0.2038817328642626]$ of the irreducible polynomial $P_\kappa$ displayed below. Consequently,
\[
N_{\mathbb Q(\lambda)/\mathbb Q}(\lambda)=\tfrac14,
\qquad
N_{\mathbb Q(\kappa)/\mathbb Q}(\kappa)=1.
\]
\end{proposition}

The dominant eigenvalue of $R_A$ is an algebraic number $\lambda$ whose irreducible polynomial over $\mathbb Q$ is needed explicitly to compute the field norm for the irrationality argument below.
 It turns out that this value 
 $\lambda$ is the unique root of the irreducible degree-14 polynomial
\begin{align*}
P_\lambda(x)={}&16 x^{14} - 117232 x^{13} - 8330040 x^{12} - 3628852593 x^{11} \notag\\
&+ 766183444916 x^{10} - 37440753370248 x^9 \notag\\
&+ 504085851461616 x^8 + 57085439654130 x^7 \notag\\
&+ 18299281845941356 x^6
- 5676298472505780 x^5 \notag\\
&+ 422743039128068 x^4 - 4789675432569 x^3 \notag\\
&+ 7468158864 x^2 - 2899292 x + 4
\end{align*}
in the \emph{isolating interval} $[0.0944920594891271,\ 0.0944920594891272]$. 
The certificate derives $P_\lambda$ as the elimination polynomial for $\lambda$ and confirms it is irreducible (Item~5, Appendix~\ref{app:certificate}), so $P_\lambda$ is the minimal polynomial of $\lambda$ over $\mathbb Q$. The isolating interval contains exactly one real root of $P_\lambda$ (Item~7), thereby identifying $\lambda$ uniquely among the roots. 

Root isolation must identify the eigenvalue of the actual derivative, not merely a nearby algebraic root. The certificate first checks the exact characteristic-polynomial annihilation identities. When combining an eigenvalue enclosure with a polynomial isolator, it counts roots on their convex hull and requires exactly one, with neither endpoint a root. Overlap of two intervals alone would not establish that they enclose the same number. The same check is made for $\nu$ and $\kappa$.

If $\alpha$ is an algebraic number with irreducible polynomial $a_d x^d + \cdots + a_0$ over $\mathbb Q$, its \emph{field norm} $N_{\mathbb Q(\alpha)/\mathbb Q}(\alpha)$ is the product of all roots of this polynomial, which equals $(-1)^d a_0/a_d$ by Vieta's formulas~\citep{Lang2002}.
Hence
\[
N_{\mathbb{Q}(\lambda)/\mathbb{Q}}(\lambda)=\frac{4}{16}=\frac14,
\]
and the other eigenvalue of $R_A$ is also a root of $P_\lambda$.

Similarly, the eigenvalue $\kappa$ of $R_B$ not inherited from $R_A$
turns out to be 
the unique root of the irreducible degree-7 polynomial
\[
P_\kappa(x)=x^7-35x^6+385x^5-1835x^4+2315x^3-625x^2+51x-1
\]
in the isolating interval $[0.2038817328642625,\ 0.2038817328642626]$. As above, the certificate derives $P_\kappa$ as the irreducible elimination polynomial for $\kappa$ (Item~6), making it the minimal polynomial of $\kappa$ over $\mathbb Q$, and verifies that this interval contains exactly one real root (Item~7). Hence 
\(
N_{\mathbb{Q}(\kappa)/\mathbb{Q}}(\kappa)=1.
\)

Defining $\nu$ as the unique root of $P_\lambda$ in $[0.0011453882842312,\ 0.0011453882842313]$ (Item~7 of the certificate), 
we then find that the isolating intervals give
\(
0<\nu<\lambda<\kappa<1.
\)

The irrationality of the logarithmic ratio now follows from a general norm criterion.

\begin{lemma}[Norm criterion for irrational logarithmic ratio]
\label{lem:norm-irrational}
Let $\lambda,\kappa\in(0,1)$ be algebraic numbers satisfying
\[
N_{\mathbb Q(\lambda)/\mathbb Q}(\lambda)=\frac14,\qquad
N_{\mathbb Q(\kappa)/\mathbb Q}(\kappa)=1.
\]
Then
\[
\frac{\log\lambda}{\log\kappa}\notin\mathbb Q.
\]
\end{lemma}

\begin{proof}
Suppose for contradiction that $\log\lambda/\log\kappa\in\mathbb Q$. Since $\lambda,\kappa\in(0,1)$, both logarithms are negative, so the ratio is positive, and there exist positive integers $m,n$ with $n\log\lambda=m\log\kappa$, i.e., $\lambda^n=\kappa^m$. Set $L=\mathbb Q(\lambda,\kappa)$ and apply the field norm $N_{L/\mathbb Q}$ to both sides to obtain
\[
N_{L/\mathbb Q}(\lambda)^n
=
N_{L/\mathbb Q}(\kappa)^m.
\]
By the tower formula,
\begin{align*}
N_{L/\mathbb Q}(\lambda)
&=
N_{\mathbb Q(\lambda)/\mathbb Q}(\lambda)^{[L:\mathbb Q(\lambda)]}
=
\left(\frac14\right)^{[L:\mathbb Q(\lambda)]},\\
N_{L/\mathbb Q}(\kappa)
&=
N_{\mathbb Q(\kappa)/\mathbb Q}(\kappa)^{[L:\mathbb Q(\kappa)]}
=
1.
\end{align*}
The norm equality therefore becomes $(1/4)^{n[L:\mathbb Q(\lambda)]}=1$, which is impossible. Hence $\log\lambda/\log\kappa\notin\mathbb Q$.
\end{proof}

Applying Lemma~\ref{lem:norm-irrational} with the certified norm values $N_{\mathbb Q(\lambda)/\mathbb Q}(\lambda)=1/4$ and $N_{\mathbb Q(\kappa)/\mathbb Q}(\kappa)=1$ established above yields
\[
\frac{\log\lambda}{\log\kappa}\notin\mathbb{Q}.
\]

%% file: sections/08_halfstep.tex
\subsection{Exact dominant half-step directions and half-step contractions}
\label{sec:halfstep}

A half-step here is one five-update word $H_i$, hence one burst when that gadget wins. Rather than track two alternating derivative maps, we identify the phases by reflection. Set
\[
\Psi_A(p,d)=(p,d),\qquad \Psi_B(p,d,e)=(p,d+e/2,e).
\]
Write $q=d$ for A and $q=d+e/2$ for B. In these coordinates both physical involutions become $S(p,q,e)=(p,\tfrac12-q,e)$, with $e$ omitted for A. Define
\[
\overline F_0^G=\Psi_G\circ F_0^G\circ\Psi_G^{-1},\qquad T_G=S\circ\overline F_0^G.
\]
The fixed points are $z_A^*=(p_*,d_*)$ and $z_B^*=(p_*,d_*,0)$.

\begin{lemma}[Normalized half-step map]\label{lem:halfstep-abstract}
A physical chart state $v$ in phase $i$ has normalized state $S^i\Psi_G(v)$. Applying $H_i$ toggles the phase and sends that normalized state through $T_G$, independently of $i$. Moreover,
\[
T_G^2=\overline F_1^G\circ\overline F_0^G,
\]
so $DT_G(z_G^*)^2$ is conjugate to the return derivative $R_G$.
\end{lemma}
\begin{proof}
The conjugacy $\overline F_1=S\overline F_0S$ and $S^2=I$ give both assertions by composition. At the fixed point, differentiation gives the squared-derivative identity.
\end{proof}

\begin{lemma}[Dominant half-step directions]\label{lem:halfstep}
The normalized derivatives have positive dominant eigenvalues $c_A=\sqrt\lambda$ and $c_B=\sqrt\kappa$. With eigenvectors normalized to first coordinate 1, write
\[
v_A=(1,\theta_A),\qquad v_B=(1,\eta_B,\chi_B).
\]
The physical phase-zero eigenvectors are $v_0^G=D\Psi_G^{-1}v_G$, and $v_1^G=DS_Gv_0^G$. They satisfy
\[
J_0^Gv_0^G=c_Gv_1^G,\qquad J_1^Gv_1^G=c_Gv_0^G.
\]
\end{lemma}
\begin{proof}
The exact eigenvector equations and positive multiplier enclosures below identify the positive square roots of the dominant return eigenvalues. The normalized eigenvector equation is $DS_GJ_0^Gv_0^G=c_Gv_0^G$ in physical coordinates. Multiplying by $DS_G$ proves the first identity; conjugacy proves the second. All zero residuals in these eigenvector equations are algebraic identities, not approximate numerical equalities.
\end{proof}

\subsection*{A-gadget eigenvectors and contraction}
The certificate encloses
\begin{align*}
c_A&\in[0.30739560746557,\,0.30739560746558],\\
\theta_A&\in[0.580405272518281104,\,0.580405291461442916].
\end{align*}
Thus $v_0^A=(1,\theta_A)$ and $v_1^A=(1,-\theta_A)$. The residual coordinates of a normalized state are $x=p-p_*$ and $t=q-d_*-\theta_Ax$.

\subsection*{B-gadget eigenvectors and contraction}
The corresponding certified enclosures are
\begin{align*}
c_B&\in[0.45153264872461,\,0.45153264872462],\\
\eta_B&\in[0.705067523141833,\,0.705067554266809],\\
\chi_B&\in[1.01349886578531906,\,1.01349888527878103].
\end{align*}
Set $\theta_{B,0}=\eta_B-\chi_B/2$ and $\theta_{B,1}=-\eta_B-\chi_B/2$. Then $v_i^B=(1,\theta_{B,i},\chi_B)$ are the physical directions used in the original charts. The normalized residual coordinates are
\[
x=p-p_*,\qquad t=q-d_*-\eta_Bx,\qquad z=e-\chi_Bx.
\]
The reflection removes the phase dependence of these coordinates and of the estimates below.

%% file: sections/09_sectors.tex
\subsection{Explicit phase sectors and their certified invariance}
\label{sec:sectors}

We use one normalized cone for each gadget and obtain the physical phase sectors by reflection. Set $R=10^{-5}$ and $\rho=10^{-3}$. The cone radius is a sufficient local verification domain, not an estimate of a basin or of an ambient neighborhood of initializations.

\subsection*{Review of interval arithmetic}
\label{sec:interval_arithmetic}
A rational interval encloses every exact value of its input. Addition and subtraction use endpoint bounds; multiplication takes the minimum and maximum of the four endpoint products; division is permitted only when the denominator interval excludes zero. Induction then encloses a rational expression over its whole input box. All endpoints used here are exact rationals. The certificate separately checks that every divisor interval excludes zero: pointwise nonvanishing from Lemma~\ref{lem:regularity} does not alone imply that interval property.

\subsection*{A-sectors}
In the normalized coordinates of Section~\ref{sec:halfstep}, define
\[
\mathcal C_A=\{0<x\le R,\ |t|\le\rho x\},\qquad \tau_A=|t|/x.
\]
The physical sectors are $S_i^A=S_A^i\Psi_A^{-1}(\mathcal C_A)$, $i=0,1$, where the inequalities defining $\mathcal C_A$ are interpreted through its residual coordinates. They consist of strictly positive probability vectors in the affine chart.

\subsection*{B-sectors}
Similarly, define
\[
\mathcal C_B=\{0<x\le R,\ \max(|t|,|z|)\le\rho x\},\qquad
\tau_B=\max(|t|,|z|)/x,
\]
and $S_i^B=S_B^i\Psi_B^{-1}(\mathcal C_B)$. These are also strictly positive probability vectors. In particular, $e=\chi_Bx+z>0$ because $\chi_B>\rho$.

\begin{lemma}[Normalized derivative structure]\label{lem:firstcol-abstract}\label{lem:firstcol}
In residual coordinates, the derivative of $T_G$ at its fixed point has block form
\[
L_G=\begin{pmatrix}c_G&a_G\\0&D_G\end{pmatrix}.
\]
\end{lemma}
\begin{proof}
The first basis vector in these coordinates is the exact dominant eigenvector. Its image is $c_G$ times that vector, which gives the first column $(c_G,0)^\top$. This is the only eigenvector transport identity needed for the cone argument.
\end{proof}

Here is a complete finite specification of the two Taylor estimates. Superscripts $+$ and $-$ denote rational upper and lower enclosures. With $\theta_{B,0}=\eta_B-\chi_B/2$, one may use
\[
\theta_{B,0}\in[0.198318090249174404,\,0.198318111627417225].
\]
Bound the derivatives of the physical maps $F_0^G$ on
\begin{align*}
\mathcal B_A&=[p_*,p_*+R]\times[d_*,d_*+(\theta_A^++\rho)R],\\
\mathcal B_B&=[p_*,p_*+R]\times[d_*,d_*+(\theta_{B,0}^++3\rho/2)R]
                   \times[0,(\chi_B^++\rho)R].
\end{align*}
Replace algebraic endpoints by rational enclosures during evaluation. The cones lie in these boxes since $\theta_A^->\rho$, $\theta_{B,0}^->3\rho/2$, and $\chi_B^->\rho$. Put
\[
h_A=(1,\theta_A^++\rho),\qquad
h_B=(1,\theta_{B,0}^++3\rho/2,\chi_B^++\rho).
\]
If $H_{\ell,ij}$ bounds $|\partial_i\partial_jF_{0,\ell}^G|$ on its box, then
\[
K_\ell=\frac12\sum_{i,j}H_{\ell,ij}h_ih_j
\]
bounds the corresponding Taylor remainder by $K_\ell x^2$. Differentiating the rational maps, reducing numerator and denominator, and evaluating with rational intervals supplies these bounds.

The normalized residual remainder constants can be taken as
\begin{align*}
A:\quad &K_x=K_p,\qquad K_t=K_d+\theta_A^+K_p,\\
B:\quad &K_x=K_p,\qquad K_t=K_d+\tfrac12K_e+\eta_B^+K_p,
\qquad K_z=K_e+\chi_B^+K_p.
\end{align*}
Let $K_\perp=K_t$ for A and $\max(K_t,K_z)$ for B. Let $a$ bound the sum of absolute entries in $a_G$ and $b$ the maximum absolute row sum of $D_G$. Set
\begin{align*}
m&=c_G^- -a\rho-K_xR,& M&=c_G^++a\rho+K_xR,\\
\alpha&=b/m,&\beta&=K_\perp/m,&\Gamma&=a/c_G^-,&\Delta&=K_x/c_G^-.
\end{align*}
Taylor's formula gives $m\le x'/x\le M$, $\tau'\le\alpha\tau+\beta x$, and $|x'/(c_Gx)-1|\le\Gamma\tau+\Delta x$. The certificate verifies for both gadgets
\[
m>\frac3{10},\quad M<\frac12,\quad \alpha<\frac7{10},\quad
\beta<10,\quad\Gamma<\frac13,\quad\Delta<6.
\]

\begin{lemma}[Sector invariance]\label{lem:sector-invariance}
Both normalized cones are invariant under $T_G$. On either cone,
\begin{equation}\label{eq:normalized-bounds}
\frac3{10}<\frac{x'}x<\frac12,\quad
\tau'\le\frac7{10}\tau+10x,\quad
x'=c_Gx(1+\varepsilon),\quad |\varepsilon|\le\frac13\tau+6x.
\end{equation}
Consequently $F_i^G(S_i^G)\subset S_{i\oplus1}^G$ for every phase.
\end{lemma}
\begin{proof}
The finite Taylor tests imply the displayed estimates. They give $0<x'<R$ and $\tau'\le(7/10)\rho+10R=0.0008<\rho$. Lemma~\ref{lem:halfstep-abstract} transports the normalized inclusion to the physical sectors.
\end{proof}

\subsection*{A-sector certified bounds}
The same computed enclosures also give $0.307<x'/x<0.308$ on $\mathcal C_A$. The common rational bounds in~\eqref{eq:normalized-bounds} suffice for the infinite-time proof; this sharper radial interval records the run-length consequence discussed in Section~\ref{sec:burst}.

\subsection*{B-sector certified bounds}
For B the sharper radial interval is $0.451<x'/x<0.452$. Reflection preserves every original and negated edge up to the column permutation $(3\ 4)$, so no additional phase-specific cone or edge calculation is required.

%% file: sections/10_locking.tex
\subsection{Local branch words and whole-burst locking}
\label{sec:locking}

For each physical box $\mathcal B_G$ from Section~\ref{sec:sectors}, map its coordinates to the simplex and apply the five updates prescribed by $H_0$. At each step enclose all four original column edges. The lower endpoint of the prescribed edge minus the largest competing upper endpoint exceeds $1/40$. Thus $H_0$ is the strict local word on the entire phase-zero sector; the row/column swap transports this conclusion to $H_1$ on phase one.

At the boundary, the first column has edge $2p-1$ by the affine charts. Put $\mu_*=2p_*-1$. If $E_{G,k}$ encloses the prescribed edge at step $k$ of the word, define the outward gap bounds
\[
\Delta^{\mathrm{box}}=\min_{G,\,2\le k\le5}(E_{G,k}^- -\mu_*^+),\qquad
\delta_{\mathrm{bdry}}^{\mathrm{box}}=\max_G(E_{G,1}^+-\mu_*^-).
\]
The finite acceptance tests are
\[
g_{\mathrm{br}}^{\mathrm{box}}>\frac1{40},\qquad
\Delta^{\mathrm{box}}-\delta_{\mathrm{bdry}}^{\mathrm{box}}>\frac3{40}.
\]
The second compares every post-first-step active edge with the frozen gadget's maximum boundary edge.

\begin{lemma}[Whole-burst locking]\label{lem:burst-lock}
For any physical phase pair with the two chart states in their sectors, a gadget selected at the first step completes the entire prescribed word $H_i$. Its last four choices are strict global maximizers among the original columns. At an equal first-edge comparison either gadget may be selected.
\end{lemma}
\begin{proof}
Product invariance freezes the other gadget. Its largest original edge is bounded by $\mu_*^++\delta_{\mathrm{bdry}}^{\mathrm{box}}$. At every subsequent active step the prescribed edge is larger, by the locking test; within the active gadget it is the unique local maximizer by the branch test. Induction gives the full word.
\end{proof}

For negation closure, retain the four edge intervals of the frozen box, and at each active step compare the prescribed lower endpoint with the upper endpoints of all negated active and frozen edges. For each choice of active gadget the minimum difference exceeds $1/4$. The permutation $(3\ 4)$ transports these two checks to every physical phase pair. Thus negated columns never compete on the certified tail.

%% file: sections/11_burst.tex
\subsection{Product burst dynamics}
\label{sec:burst}

Start in any pair of physical sectors. For the rational initialization of Section~\ref{sec:start}, take burst boundary $n=0$ immediately after its 140-step entry prefix. Let $W_n=a_n\otimes b_n$ and $(i_n,j_n)$ be its physical phase pair. Normalize each chart state by its current reflection and use the residual coordinates of Section~\ref{sec:halfstep}. Write $x_n=x_n^A$, $s_n=x_n^B$, and $r_n=x_n/s_n$. Sector invariance and whole-burst locking apply inductively for all $n$.

\begin{lemma}[First-edge comparison at burst boundaries]\label{lem:first-edge}
The best original edge in either gadget is $2p-1=\mu_*+2x_G$. Hence A wins if $x_n>s_n$, B wins if $s_n>x_n$, and either may win on equality.
\end{lemma}
\begin{proof}
In the A chart the first column has edge $(p-\tfrac12)-(1-p)+(\tfrac12-d)+d=2p-1$. In B its last two positive contributions sum with the third to the same value. Reflection leaves $p$ fixed; the strict branch gap makes column 1 the unique local maximizer. The remaining comparison is between the two principal displacements.
\end{proof}

\subsection*{Half-step contraction bounds}

Only the winning gadget advances. At each boundary, AdaBoost therefore advances the component with the larger positive displacement from the shared limiting orbit, and either one in case of a tie. Each advance shrinks that displacement, keeping the ratio of the two displacements bounded. The general comparison band below makes this mechanism explicit; our construction needs only $L=U=1$.

The shared bounds in~\eqref{eq:normalized-bounds} put every contraction factor in $(3/10,1/2)$. The following general result separates this scheduling argument from the local estimates. Its index $k$ counts a gadget's own bursts, while $n$ counts all bursts.

\begin{theorem}[Competition]\label{thm:competition}
Let $u_G(k)>0$, $G\in\{A,B\}$, satisfy
\[
 0<m_G\le\frac{u_G(k+1)}{u_G(k)}\le M_G<1,
 \qquad \log u_G(k)=k\ell_G+O(1),\qquad \ell_G<0.
\]
After $n$ events let $N_A(n)+N_B(n)=n$, and put
$r_n=u_A(N_A(n))/u_B(N_B(n))$. Suppose fixed $0<L\le U<\infty$ have the
following property: selecting $A$ requires $r_n\ge L$, and selecting $B$
requires $r_n\le U$. Choices within the overlap may depend arbitrarily
on the past. Then both counters diverge, $r_n$ eventually belongs to
$[m_AL,U/m_B]$, and
\begin{equation}\label{eq:frequency}
 N_A(n)=\omega n+O(1),\qquad
 \omega=\frac{-\ell_B}{-\ell_A-\ell_B}.
\end{equation}
If $\ell_A/\ell_B$ is irrational, the winner sequence is not eventually
periodic. If the logarithmic remainder assumption is weakened to $o(k)$,
the conclusion weakens to $N_A(n)/n\to\omega$.
\end{theorem}

\begin{proof}
An $A$-update multiplies $r$ by a number in $[m_A,M_A]$, whereas a $B$-update
divides it by a number in $[m_B,M_B]$. Either counter remaining bounded
would force $r$ toward zero or infinity, eventually contradicting the
selection threshold. Both counters therefore diverge.

The interval $I=[m_AL,U/m_B]$ is invariant. For example, an allowed
$A$-update starts with $r\ge L$, so cannot fall below $m_AL$; it decreases
$r$, so cannot leave through the upper endpoint. The other case is
symmetric. Above $I$ only $A$ is permitted and repeated contractions enter
$I$ without crossing its lower endpoint. Below $I$ only $B$ is permitted,
with the analogous conclusion. Consequently $\log r_n=O(1)$, giving
\[
 N_A(n)\ell_A-N_B(n)\ell_B=O(1).
\]
Substitute $N_B(n)=n-N_A(n)$ to obtain~\eqref{eq:frequency}. Irrationality
of $\ell_A/\ell_B$ makes $\omega$ irrational, whereas an eventually
periodic binary word has rational frequency. With $o(k)$ local remainders,
the same calculation has error $o(n)$.
\end{proof}

\begin{lemma}[Run-length bounds from contraction intervals]\label{lem:run-lengths-abstract}
With exact larger-coordinate selection, suppose the two radial factors lie in $[m_A,M_A]$ and $[m_B,M_B]$, where $M_A/m_B<1$ and $m_A/M_B^2>1$. After entry into $[m_A,1/m_B]$, A-runs have length one and B-runs at most two, under either resolution of a boundary tie.
\end{lemma}
\begin{proof}
An A-update leaves $r'\le M_A/m_B<1$, forcing B next. Two successive B-updates leave $r''\ge m_A/M_B^2>1$, forcing A. These are strict conclusions, so tie choices cause no exception.
\end{proof}

\begin{lemma}[Burst-ratio dynamics and run-length bounds]\label{lem:ratio}
Both gadgets advance infinitely often and $x_n+s_n\to0$. Eventually A-runs have length one and B-runs at most two.
\end{lemma}
\begin{proof}
The bounded-ratio part of Theorem~\ref{thm:competition}, which uses only the radial contractions, applies with $L=U=1$. Each local coordinate therefore tends to zero. For the run bounds take $m_A=307/1000$, $M_A=308/1000$, $m_B=451/1000$, and $M_B=452/1000$, as certified in Section~\ref{sec:sectors}. They satisfy $M_A/m_B<1$ and $m_A/M_B^2>1$ by rational arithmetic. The preceding lemma applies. The frequency proof below needs only bounded ratios and infinitely many advances, not these special run lengths.
\end{proof}

%% file: sections/12_asymptotics.tex
\subsection{Local asymptotics on the two gadgets}
\label{sec:asymptotics}

In each gadget's own burst time $k$, the normalized estimates imply a logarithmic law. The transverse recurrence makes the relative errors summable, which is the only local asymptotic input needed for competition.

\begin{lemma}[Local logarithmic law]\label{lem:log}\label{cor:summable-errors}
Suppose $x_k>0$, $\tau_k\ge0$, and $0<c<1$ satisfy
\[
 0<x_{k+1}/x_k\le M<1,\qquad
 \tau_{k+1}\le\alpha\tau_k+\beta x_k,\quad 0\le\alpha<1,
\]
and $x_{k+1}=cx_k(1+\varepsilon_k)$ with
$|\varepsilon_k|\le\Gamma\tau_k+\Delta x_k$, for fixed nonnegative
$\beta,\Gamma,\Delta$. Then $x_k/c^k$ converges to a finite positive
constant $C$. Equivalently, $x_k=Cc^k(1+o(1))$; in particular,
$\log x_k=k\log c+O(1)$.
\end{lemma}
\begin{proof}
Geometric contraction makes $\sum_kx_k$ finite. Summing the iterated
transverse inequality gives
\[
 \sum_{k\ge0}\tau_k\le
 \frac{\tau_0+\beta\sum_{k\ge0}x_k}{1-\alpha}<\infty.
\]
Thus $\sum_k|\varepsilon_k|<\infty$. Since $1+\varepsilon_k>0$ and
$\varepsilon_k\to0$, the sum of $\log(1+\varepsilon_k)$ is absolutely
convergent after finitely many terms. Hence
\[
 \log\frac{x_k}{c^k}
 =\log x_0+\sum_{j=0}^{k-1}\log(1+\varepsilon_j)
\]
converges to a finite real number. Exponentiating gives the positive
constant $C$ and proves the claim.
\end{proof}

Applying the lemma with the shared constants in~\eqref{eq:normalized-bounds} gives $x_G(k)=C_Gc_G^k(1+o(1))$ for some $C_G>0$, and hence $\log x_G(k)=k\log c_G+O(1)$. At global burst boundary $n$, the local counts are $N_A(n)$ and $N_B(n)=n-N_A(n)$.

\begin{proposition}[Irrational burst frequency]\label{prop:irrational-frequency}
Every maximizing tie rule, including choices depending on the full past, satisfies
\[
N_A(n)=\omega n+O(1),\qquad
\omega=\frac{-\log c_B}{-\log c_A-\log c_B}
=\frac{-\log\kappa}{-\log\lambda-\log\kappa}\notin\mathbb Q.
\]
Consequently the burst-winner sequence is not eventually periodic.
\end{proposition}
\begin{proof}
Theorem~\ref{thm:competition} applies with $L=U=1$. If $\omega$ were rational, then $\log c_A/\log c_B=(1-\omega)/\omega$ would be rational, contradicting $c_A=\sqrt\lambda$, $c_B=\sqrt\kappa$, and Lemma~\ref{lem:norm-irrational}. An eventually periodic binary sequence has rational limiting frequency.
\end{proof}

%% file: sections/13_nocycle.tex
\subsection{No finite-cycle convergence}
\label{sec:nocycle}

It remains to connect the irrational burst frequency to the weight vectors themselves. A single marginal coordinate suffices. Let $W_n$ be the state at burst boundary $n$, let $i_n$ be the phase of A, and define
\[
y_n=\sum_{s=1}^{5}W_n((4,s)),\qquad d_0=d_*,\quad d_1=\tfrac12-d_*.
\]
Both local counters diverge by Lemma~\ref{lem:ratio}, so the A-cone contracts to its normalized fixed point. Undoing the row reflection gives $y_n-d_{i_n}\to0$. The two possible limits are distinct because $d_*>1/4$.

\begin{proposition}[Aperiodic burst sequence rules out finite-cycle convergence]
\label{prop:no-cycle}
If the full AdaBoost trajectory converged to a finite cycle, then the burst-winner sequence would be eventually periodic.
\end{proposition}
\begin{proof}
Suppose the full trajectory converged to a finite cycle of period $h$. For the explicit start, $W_n$ occurs at original time $140+5n$. Fix a residue $r$ modulo $h$. The times $140+5(r+kh)$ all have the same residue modulo $h$, so $y_{r+kh}$ converges as $k\to\infty$. Since its distance from $d_{i_{r+kh}}$ tends to zero and $d_0\ne d_1$, the phase $i_{r+kh}$ must be eventually constant. There are finitely many residues, so $(i_n)$ is eventually periodic.

An A-burst toggles this phase and a B-burst leaves it unchanged. Thus the winner is A exactly when $i_{n+1}\ne i_n$, and the winner sequence would be eventually periodic. The same argument works for any tail entering the cones, with its entry time in place of $140$.
\end{proof}

Proposition~\ref{prop:irrational-frequency} rules out eventual periodicity of the burst winners under every maximizing tie rule. Proposition~\ref{prop:no-cycle} therefore rules out convergence of the full twenty-type trajectory to any finite cycle. Section~\ref{sec:start} transfers this conclusion to uniform row replication and to an open neighborhood of product initializations.

For the geometric picture in Figure~\ref{fig:nocycle}, both normalized cones contract, so $\|W_n-Q_{i_nj_n}\|\to0$ with $Q_{ij}=P_i^A\otimes P_j^B$. The four reference points are distinct because their marginals are distinct. The scalar argument above uses only the A-phase and does not require every reference point to be an accumulation point.

%% file: sections/14_start.tex
\subsection{Explicit rational start and explicit uniform-start counterexample}
\label{sec:start}

The initialization can be specified by small integers, without placing it directly in a narrow sector. Set
\begin{equation}\label{eq:simple-start}
a_0=\frac1{50}(7,18,12,13),\qquad
b_0=\frac1{50}(8,17,11,13,1),\qquad W_{\mathrm{init}}=a_0\otimes b_0.
\end{equation}
A finite entry calculation connects this state to the infinite-time proof.

\begin{proposition}[Strict entry into the normalized cones]\label{prop:entry}
From~\eqref{eq:simple-start}, all first 140 maximizing choices are strict, including comparison with negated columns. In groups of five updates they are
\[
\mathtt{BBBABABBABABBABABBABABBABABB},
\]
with each gadget alternating $H_0,H_1$, starting with $H_0$. After these ten A-bursts and eighteen B-bursts, both phases are zero, and the endpoint lies strictly inside the two normalized cones. Its residuals have the following enclosures:
\begin{center}
\begin{tabular}{ccc}
\toprule
& A & B\\\midrule
$x$ & $(2.13810,2.13811)\times10^{-8}$ & $(1.31657,1.31659)\times10^{-8}$\\
$t/x$ & $(-1.54,-1.34)\times10^{-7}$ & $(4.84,4.85)\times10^{-5}$\\
$z/x$ & --- & $(3.94,3.95)\times10^{-4}$\\\bottomrule
\end{tabular}
\end{center}
Here the parenthesized pairs denote open intervals; their decimal endpoints are exact rationals. At every prefix step the selected edge exceeds each competing edge by more than $10^{-8}$.
\end{proposition}

This is the finite-entry obligation in Item 13 of Appendix~\ref{app:certificate}. It checks all twenty weights directly, with all sixteen original and negated edges. Each gadget's last local update is column 2, so its affine chart constraint holds exactly by Lemma~\ref{lem:landing}. Thus $W_0$ in the burst analysis denotes the state after 140 updates from $W_{\mathrm{init}}$, and Sections~\ref{sec:burst}--\ref{sec:nocycle} apply to its tail.

\begin{lemma}[Row duplication and class aggregation]
\label{lem:duplication}
Let $M'$ be obtained from $M$ by duplicating rows. Partition the rows of $M'$ into duplicate classes, one class for each row of $M$. Let $D'_t$ be an AdaBoost trajectory on $M'$, and let $\bar D_t$ be the class-summed distribution on $M$. Then:
\begin{enumerate}
\item The edge vector of $D'_t$ on $M'$ equals the edge vector of $\bar D_t$ on $M$.
\item The class-summed weights $\bar D_t$ evolve exactly as AdaBoost on $M$.
\item If $D'_0$ is uniform on $M'$, then $\bar D_0$ assigns to each original row a mass proportional to its number of duplicates.
\item $\gamma(M')=\gamma(M)$.
\end{enumerate}
\end{lemma}

\begin{proof}
Rows in a duplicate class are identical. Therefore, for any column, the contribution of a class to the edge depends only on the sum of the weights in that class. This proves the edge-vector identity. Every maximizing column selected on the duplicated matrix is therefore also maximizing on the aggregated matrix. Following that same selected-column history defines an admissible trajectory on the compact matrix, even when tie choices depend on the full history.

For a selected column $j$, every row in a duplicate class has the same sign $M_{ij}$, so every row in that class is multiplied by the same update factor $\frac{1}{1+r_t M_{ij}}$. Therefore the class sum is updated by exactly the same formula as the corresponding row weight in $M$. This proves the dynamical aggregation claim.

If $D'_0$ is uniform, the class sum for an original row is exactly
\[
\frac{\text{number of duplicates of that row}}{\text{total number of duplicated rows}}.
\]

For the margin, every distribution on $M'$ aggregates to a distribution on $M$ with the same edge vector, and every distribution on $M$ lifts to a distribution on $M'$ with the same edge vector. Thus the set of attainable edge vectors is the same, and $\gamma(M')=\gamma(M)$.
\end{proof}

\begin{lemma}[Augmented negated-column margin]
\label{lem:augmented-margin}
Let $M^\pm=[M\mid -M]$, and suppose $\gamma(M)=\gamma>0$ with a distribution $D_*$ such that $(D_*^\top M)_j=\gamma$ for every positive column $j$. Then $\gamma(M^\pm)=\gamma$.
\end{lemma}

\begin{proof}
For every distribution $D$,
\[
\max_{h\in M^\pm}D^\top h
\ge
\max_{h\in M}D^\top h
\ge
\gamma(M)
=
\gamma,
\]
so $\gamma(M^\pm)\ge\gamma$. At $D_*$, every positive-column edge is $\gamma$ and every negated-column edge is $-\gamma$, hence $\max_{h\in M^\pm}D_*^\top h=\gamma$, so $\gamma(M^\pm)\le\gamma$. Therefore $\gamma(M^\pm)=\gamma$.
\end{proof}

\begin{lemma}[No-negated-column reduction]
\label{lem:no-negated}
Let $M^\pm=[M\mid -M]$, with the original columns ordered first and the negated columns appended afterward. Suppose a trajectory on $M$ has selected positive-column edge $r_t$ at time $t$, and suppose
\[
r_t>-\min_j(D_t^\top M)_j
\]
for every $t$. Then the trajectory on $M^\pm$, started from the same distribution and using the same tie-breaking order on the original columns, is identical to the trajectory on $M$.
\end{lemma}

\begin{proof}
The edge of the negated column $-M_{\cdot j}$ is
\[
D_t^\top(-M_{\cdot j})
=
-(D_t^\top M_{\cdot j}).
\]
Thus the best negated-column edge is
\[
\max_j D_t^\top(-M_{\cdot j})
=
-\min_j(D_t^\top M_{\cdot j}).
\]
By hypothesis, this is strictly smaller than the selected positive-column edge $r_t$. Therefore no negated column belongs to the argmax set. The selected column in $M^\pm$ is the same positive column selected in $M$. The update formula is then identical.

Induction over $t$ proves the trajectories agree at all times.
\end{proof}

Let $A^{\mathrm{num}}=(7,18,12,13)$ and $B^{\mathrm{num}}=(8,17,11,13,1)$. Replace row $(r,s)$ of $M_0$ by $A_r^{\mathrm{num}}B_s^{\mathrm{num}}$ copies to form $\widetilde M$. The multiplicities are
\[
\begin{pmatrix}
56&119&77&91&7\\
144&306&198&234&18\\
96&204&132&156&12\\
104&221&143&169&13
\end{pmatrix},
\]
whose entries sum to 2,500. Uniform weights aggregate to $W_{\mathrm{init}}$. For any maximizing selection sequence on the duplicated matrix, aggregation follows a maximizing sequence on $M_0$. Its tail is noncycling under every such sequence, so the full duplicated trajectory cannot converge to a finite cycle. The margin remains $1/5$ by Lemma~\ref{lem:duplication} and Section~\ref{sec:margin}.

\begin{proof}[Proof of Corollary~\ref{cor:negation}]
Set $\widetilde M'=[\widetilde M\mid-\widetilde M]$. Every original choice in the 140-step prefix strictly exceeds all negated edges by Proposition~\ref{prop:entry}. On the infinite tail, Section~\ref{sec:locking} gives a gap larger than $1/4$ against every negated active or frozen edge, in every phase pair. No negated column is therefore selected, and the noncycling proof applies to every maximizing rule on $\widetilde M'$. To compute the margin, the product of $(1,2,1,1)/5$ and $(1,2,1,1,0)/5$ gives every original edge $1/5$. Lift this distribution to the duplicate classes and apply Lemma~\ref{lem:augmented-margin}. This gives $\gamma(\widetilde M')=1/5$.
\end{proof}

\paragraph{An open set of product initializations.}
All 140 prefix choices at $(a_0,b_0)$ have strict gaps and nonzero update denominators. Their finite composition is continuous, so the same choices persist in a sufficiently small neighborhood $V$ in $\operatorname{ri}(\Delta^3)\times\operatorname{ri}(\Delta^4)$. The last column-2 updates land exactly on the two affine charts, and strict endpoint cone membership persists after shrinking $V$. Every such product state therefore has a certified noncycling tail, for any subsequent maximizing tie choices, on both $M_0$ and $[M_0\mid-M_0]$. This completes the additional assertion of Theorem~\ref{thm:main}. The neighborhood is seven-dimensional and its radius is existential; $R$ is a cone radius, not a radius in the full simplex of correlated starts.

%% file: sections/aux_results.tex
\subsection{Auxiliary lemma}
\label{aux}

\begin{lemma}[Balancing after update {\cite{SF2012}}]
\label{lem:landing}
A single AdaBoost update on a selected column $h\in\{\pm1\}^m$ sends the total post-update mass on the $h=+1$ rows and on the $h=-1$ rows to exactly $1/2$ each.
\end{lemma}

\begin{proof}
If $\mu=\sum_i w_i h_i$, then the exhaustive AdaBoost update is
\(
w_i'=\frac{w_i}{1+\mu h_i}.
\)
Therefore
\[
\sum_{h_i=+1} w_i'
=
\frac{\sum_{h_i=+1} w_i}{1+\mu}
=
\frac{(1+\mu)/2}{1+\mu}
=
\frac12,
\]
and similarly $\sum_{h_i=-1}w_i'=1/2$.
\end{proof}

%% file: sections/16_discussion.tex
\section{Key components of the transcript from the workflow}
\label{app:workflow}

\noindent\textbf{Historical material.} The following excerpts record intermediate proposals and criticism, including failed arguments. They are retained as discovery evidence, not as statements of the current theorem or certificate. The final proof is Appendix~\ref{app:proof}; the post-submission simplifications are summarized in Section~\ref{sec:disc-boosting}.

One of the primary goals of this research was to push the limits of what can be achieved with large language models (LLMs) for novel mathematical research. While many recent efforts have targeted open problems in mathematics, there have been fewer notable results in other areas of the mathematical sciences, such as theoretical computer science. We searched for problems on SolveAll~\citep{solveall2026}---a curated collection of open research problems spanning machine learning---and selected this problem. 

Here, we document many key conversations from the human-LLM workflow described in Section~\ref{sec:workflow}. The model outputs are copied verbatim, including the failed counterexample attempt at the end. Generating the initial solution took roughly two weeks of effort. We began by posing the problem to various frontier language models as a single-shot task. Using the problem statement from SolveAll (shown below), we prompted GPT-5.4 Pro, Claude Opus 4.6, and Gemini Deep Think to solve the problem.

\begin{aibox}{Problem statement (SolveAll)}
Let $S=\{(x_i,y_i)\}_{i=1}^m$ be a fixed binary-labeled dataset with $y_i\in\{-1,+1\}$, and let $\mathcal H$ be a weak-hypothesis class. Consider the discrete AdaBoost update of Freund \& Schapire (1997), in the exhaustive weak-learner regime used in Rudin et al.\ (2012), with weight vectors $w_t\in\Delta^{m-1}$ (the probability simplex). At iteration $t$, choose
\[
h_t\in\arg\max_{h\in\mathcal H}\sum_{i=1}^m w_{t,i}\,y_i\,h(x_i),
\]
define weighted error
\[
\varepsilon_t=\sum_{i=1}^m w_{t,i}\,\mathbf{1}\{h_t(x_i)\neq y_i\},
\]
and update with
\[
\alpha_t=\frac{1}{2}\log\frac{1-\varepsilon_t}{\varepsilon_t},\qquad
w_{t+1,i}=\frac{w_{t,i}\exp(-\alpha_t y_i h_t(x_i))}{Z_t},
\]
where $Z_t$ normalizes to $\sum_i w_{t+1,i}=1$.

Equivalently, with finite hypothesis set $\mathcal H=\{\tilde h_1,\dots,\tilde h_N\}$ and matrix $M\in\{-1,+1\}^{m\times N}$ defined by $M_{ij}=y_i\tilde h_j(x_i)$, step $t$ selects
\[
j_t\in\arg\max_{j\in[N]}(w_t^\top M)_j,\qquad h_t=\tilde h_{j_t}.
\]
As specified in Rudin et al.\ (2012), if this argmax is not unique, ties are broken in a fixed deterministic way (for concreteness: pick the smallest index $j$). The generic no-tie condition means the argmax is unique at every iterate, i.e.\
\[
(w_t^\top M)_j\neq (w_t^\top M)_{j'}\quad\text{for all }j\neq j'\text{ and all }t,
\]
equivalently, $w_t$ never lands on a tie boundary between weak-hypothesis regions of the simplex.

This induces a discrete dynamical system $T\colon\Delta^{m-1}\to\Delta^{m-1}$ by $w_{t+1}=T(w_t)$.

\medskip\noindent\textbf{Unsolved problem target:}
Does every trajectory converge to a cycle under natural genericity conditions?
\end{aibox}

It quickly became apparent that the most sophisticated arguments were produced by GPT-5.4 Pro. Opus 4.6 regularly refused to engage because it recognized the problem as open, and Deep Think consistently returned solutions with numerous errors and logical flaws pointed out by both Opus 4.6 and other copies of itself. We therefore focused on generating solution attempts with GPT-5.4 Pro.

Single-shot prompts often ended in partial reductions or failed attempts. However, instead of using the API, we used the web interface to interact with GPT, which more natively supported multi-turn conversations. This allowed us to push the model to try again even if it refused to solve the problem in its first attempt, often by following up with the prompt, ``Good progress. However, you must keep pushing to resolve Rudin, Schapire, and Daubechies's problem. The only goal is a correct and complete resolution to the problem.''

After many failed attempts---roughly 10 of them---we began to compile previous results from the model into a condensed summary with Claude Opus 4.6, asking it to analyze the various attempts and suggest the most promising directions forward. We then conducted more rounds of summarizing previous attempts and generating new directions to pursue with Claude's recommendations centered on testing additional reductions or simplified cases, none of which proved fruitful when passed back to GPT. However, at around the fifth summarization attempt, Claude produced the following review and list of potential directions.

\begin{aibox}{Summary and proposed directions (Claude Opus 4.6)}

\noindent\textbf{\large Prior work from our research team}

\medskip
\noindent We have done extensive work on this problem. Below is a summary of what has been established, organized by its actual usefulness toward a full resolution.

\paragraph{Critical correction: the right margin condition.}
One collaborator identified that much of our earlier work used the wrong margin quantity. The quantity $\rho(M) = \max_{d \in \Delta^{m-1}} \min_{j \in [N]} (M^\top d)_j$ (which asks for a distribution over \emph{examples} making all hypothesis edges positive) is \textbf{not} the relevant weak-learning margin. The correct quantity is the minimax game value:
\[
\gamma(M) = \min_{w \in \Delta^{m-1}} \max_{j \in [N]} (w^\top M)_j = \max_{\lambda \in \Delta^{N-1}} \min_{i \in [m]} (M\lambda)_i.
\]
This is the quantity guaranteeing AdaBoost's chosen edge is uniformly positive. One can make $\rho > 0$ for any matrix by appending an all-positive row, so ``$\rho > 0$'' does not isolate the hard case. Any proof strategy must use $\gamma(M) > 0$.

\paragraph{Useful: clarification that ``eventual periodicity'' is the wrong target.}
An explicit counterexample to \emph{literal} eventual periodicity was constructed. Consider 5 hypotheses with negative-set family $\{0,2\}, \{1,2\}, \{2,3\}, \{2,4\}, \{0,1,3\}$, giving margin $\rho = 0$. On a carefully constructed invariant region, generic no-tie orbits satisfy $q_n \sim 9/(16n) \to 0$---they converge to a boundary fixed point but never literally repeat. Similarly, for $N = 3$ under the genuine weak-learning margin, if an all-positive row $(+1,+1,+1)$ carries initial mass $d_0 > 0$, then $d_t \to 0$ exponentially but never reaches zero, so the orbit never literally enters a finite cycle even though it converges phasewise to a 3-cycle.

\textbf{Consequence:} The right conjecture is \emph{phasewise convergence to a cycle} (i.e., convergence of the weight vector to a periodic orbit after quotienting out vanishing components), not literal eventual entry into a repeating sequence of weight vectors.

\paragraph{Useful: exact reformulation as coordinate descent.}
After aggregating examples by their sign pattern $r = (M_{i1}, \ldots, M_{iN}) \in \{-1,+1\}^N$, with aggregate mass $w_t(r)$, the dynamics takes the exponential-family form:
\[
w_t(r) = \frac{q_r e^{-\lambda_t \cdot r}}{\sum_s q_s e^{-\lambda_t \cdot s}}, \qquad \lambda_{t+1} = \lambda_t + \operatorname{artanh}(\mu_{j_t}(\lambda_t))\, e_{j_t},
\]
where $\mu_j(\lambda) = \sum_r w_\lambda(r)\, r_j$ is the edge of hypothesis $j$, and $j_t \in \arg\max_j \mu_j(\lambda_t)$.

This is \textbf{greedy (Gauss--Southwell) coordinate descent} on the convex potential
\[
F(\lambda) = \log \sum_r q_r e^{-\lambda \cdot r}.
\]

\paragraph{Useful: toric invariant structure.}
Let $C \in \{0,1\}^{(N+1) \times q}$ be the extended incidence matrix (first row all-ones, remaining rows recording which hypotheses each sign-pattern gets wrong). For every $z \in \ker C$, the monomial $I_z(x) = \prod_\ell x_\ell^{z_\ell}$ is exactly invariant under every branch map $T_j$. So orbits are confined to toric leaves. This is an exact structural fact, not an approximation.

\paragraph{Useful: periodic orbits carry canonical equal-weight vectors.}
If $x_0$ lies on a period-$p$ orbit with step sizes $\beta_t$, define cumulative beta $\Gamma_j = \sum_{t: j_t = j} \beta_t$. Then $\sum_{j \in A} \Gamma_j$ is constant across all support blocks $A$. So every periodic orbit determines a positive vector in an ``equal-weight cone''---the beta-weighted visit frequencies must balance across all active support sets.

\paragraph{Correct but low strategic value: $N = 3$ theorem.}
Under the genuine weak-learning margin $\gamma(M) > 0$ plus nondegeneracy, for $N = 3$ the only possible sign patterns are the three singleton-negative rows plus an optional all-positive row. Errors simplify to $\varepsilon_j = x_j$, the itinerary locks to cyclic order from round~1, and the return map is projective-linear. Perron--Frobenius gives phasewise convergence to a 3-cycle. This is a complete theorem, but the proof mechanism (everything becomes linear) is entirely specific to $N = 3$ and gives no leverage on $N \geq 4$.

\paragraph{Not useful toward a resolution: $N = 4$ classification.}
The team exhaustively classified all 13 feasible active families at $N = 4$ (up to relabeling), found 5 that cover all 4 coordinates, and showed only one is irreducible and multi-negative: $\{4, 12, 13, 23\}$. For that family, an explicit 2D return map in sorted ratio coordinates $(P, R)$ was analyzed: transient branches $B, C$ were shown to increase $R$, and the eventual regime is either $A^\infty$ (converging to a 6-cycle) or $(DA)^\infty$ (converging to a 10-cycle, proved via contraction with $\|DH\|_\infty \leq 0.65$).

This is mathematically correct but \textbf{strategically a dead end}: you cannot close the problem by classifying one $N$ at a time when the number of families and possible cycle periods both grow without bound (periods up to 219 were observed at $N = 7$).

\paragraph{Not useful: $N = 5$ enumeration and computational scans.}
47 candidate families were enumerated at $N = 5$, with 22 numerically retaining all 5 coordinates. All tested orbits cycled. Six distinct asymptotic direction patterns and periods $\{5,6,7,8,9,10,12,14\}$ were observed. This is computational evidence only and provides no path to proof.

\paragraph{Not useful: repeated reformulations.}
Across the prior work, the dynamics was restated as: KL projection / Gauss--Southwell, weighted-homogeneous quotient maps, greedy birational systems on antichain hypergraphs, toric-leaf dynamics, multiaffine polynomial replacement. These are all the same object in different coordinates. None introduces a genuinely new proof technique.

\bigskip
\hrule
\bigskip

\noindent\textbf{\large What is needed: a top-down argument for all $N$}

\medskip
\noindent The prior work has been almost entirely bottom-up case analysis. The problem almost certainly requires a structural argument that works at all $N$ simultaneously. Here are the most promising directions, in decreasing order of tractability:

\paragraph{Direction 1: Exploit the coordinate-descent / potential structure.}
AdaBoost is greedy coordinate descent on $F(\lambda) = \log \sum_r q_r e^{-\lambda \cdot r}$, which is convex and real-analytic. The step sizes are exact: $\Delta\lambda_j = \operatorname{artanh}(\mu_j(\lambda))$. This is not generic coordinate descent---the step size is the \emph{exact} one-dimensional minimizer along each coordinate axis.

Key properties to exploit:
\begin{itemize}
\item $F$ decreases by $-\log(2\sqrt{\varepsilon_t(1-\varepsilon_t)})$ at each step, which is bounded below by $-\log(1 - \gamma^2/2) > 0$ when $\gamma(M) > 0$.
\item The gradient $\nabla F(\lambda) = -\mu(\lambda)$, where $\mu_j(\lambda)$ is the edge of hypothesis $j$ under $w_\lambda$.
\item After each exact coordinate minimization along axis $j$, the new edge satisfies $\mu_j(\lambda^+) = 0$---the selected coordinate's gradient is zeroed out.
\item The sublevel sets of $F$ are compact (since $\gamma > 0$ implies $F(\lambda) \to +\infty$ as $\|\lambda\| \to \infty$\ldots\ actually, $F$ is bounded below and its sublevel sets need careful analysis).
\end{itemize}

\textbf{The key question}: for exact coordinate descent on a real-analytic convex function where each step zeros out one partial derivative, can you prove eventual periodicity of the coordinate selection sequence?

This connects to the literature on cyclic and greedy coordinate descent. Relevant results: Beck and Tetruashvili (2013) on convergence rates, Nutini et al.\ (2015) on greedy coordinate descent, and the general theory of nonlinear Gauss--Seidel methods. The specific structure here (exact minimization, real-analyticity, finite partition of the domain into greedy-selection regions) may admit a proof via the \L{}ojasiewicz inequality or related real-analytic geometry.

\paragraph{Direction 2: Real-analytic dynamics and the branch partition.}
The simplex is partitioned into finitely many open regions $\Omega_j = \{w : (w^\top M)_j > (w^\top M)_k \; \forall k \neq j\}$, and on each region $T$ is a fixed real-analytic map $T_j$. The question is whether the symbolic itinerary $j_0, j_1, j_2, \ldots$ is eventually periodic.

After quotienting by the toric invariants, orbits live on a compact real-analytic manifold of dimension at most $N - 1$. The return map on each branch is a rational map. The set of accumulation points of a non-periodic orbit would be a compact invariant set, and its structure is constrained by the real-analyticity of each branch.

A possible approach: show that the $\omega$-limit set of any orbit is a finite union of periodic orbits. For piecewise real-analytic maps, this might follow from showing that the orbit can only visit finitely many distinct branch sequences (using that each branch map is a contraction in some metric, or that the potential $F$ provides a strict Lyapunov function for the branch dynamics).

\paragraph{Direction 3: Contraction of composite return maps.}
In the $N = 4$ case, the team found that the two-step return map $H = T_A \circ T_D$ is a strict contraction ($\|DH\|_\infty \leq 0.65$). Is this an accident of low dimension, or is there a general mechanism?

For any eventually periodic branch word $\sigma = (j_0, \ldots, j_{p-1})$, the return map is $R_\sigma = T_{j_{p-1}} \circ \cdots \circ T_{j_0}$. Each $T_j$ maps the selected coordinate's error to exactly $1/2$ and rescales everything else. The conjecture that all such return maps are eventually contracting (in some norm, on the relevant toric leaf) would immediately imply convergence to a unique periodic orbit for each branch word, and then finiteness of realizable branch words would give eventual periodicity.

\paragraph{Direction 4: The $\mu_j = 0$ post-condition as a reset mechanism.}
After selecting hypothesis $j$, the new weight vector satisfies $\mu_j(\lambda^+) = 0$, i.e., hypothesis $j$ has zero edge. This is a strong constraint: it means the orbit is always being projected onto a coordinate hyperplane of the edge space. The itinerary question becomes: starting from a point where $\mu_j = 0$, which coordinate $k$ has the largest edge, and after zeroing $\mu_k$, does the sequence of zeroed coordinates eventually repeat?

This ``reset'' structure---each step annihilates one coordinate of the gradient---is very specific and may admit a combinatorial/topological argument about the sequence of active constraints.

\bigskip
\hrule
\bigskip

\noindent\textbf{Your task.}
You have the full problem statement, the correct margin condition ($\gamma(M) > 0$), the known structural facts (coordinate descent on $F$, toric invariants, $\mu_j = 0$ post-condition, canonical beta vectors on periodic orbits), and four promising directions for a general proof.

Focus on whichever direction you find most tractable. Be careful, precise, and rigorous. We need a complete resolution. Remember, the only thing we care about is producing a novel and correct solution to the problem.
\end{aibox}

\noindent GPT then provided the following response that, while not a solution, led Claude to propose pursuing a disproof via a counterexample.

\begin{aibox}{Structural reduction of the cycling problem (GPT-5.4 Pro)}

The Rudin--Daubechies--Schapire papers already show that AdaBoost admits genuine low-dimensional cyclic dynamics, and the COLT open problem asks for a full characterization of whether the weight distributions converge to a cycle in general.

I do not have a certified full resolution of the general problem. What I can certify is a sharp structural reduction. It isolates the real obstruction and shows exactly what a complete proof or disproof now has to control.

\paragraph{1.\ The compact analytic-descent route cannot work.}
Let
\[
F(\lambda)=\log\sum_r q_r e^{-\lambda\cdot r},
\qquad
\mu_t=\max_j \mu_j(\lambda_t),
\qquad
\beta_t=\operatorname{artanh}(\mu_t).
\]
Under the correct weak-learning margin assumption $\gamma(M)>0$, every selected edge satisfies
\[
\mu_t\ge \gamma>0,
\]
hence
\[
\beta_t\ge \beta_*:=\operatorname{artanh}(\gamma)>0.
\]
Because AdaBoost performs the exact coordinate minimizer,
\[
F(\lambda_{t+1})-F(\lambda_t)=\log Z_t
= \tfrac12\log(1-\mu_t^2)
\le \tfrac12\log(1-\gamma^2)<0.
\]
Therefore
\[
F(\lambda_t)\le F(\lambda_0)+ t\cdot\tfrac12\log(1-\gamma^2)\to -\infty
\]
linearly, and also
\[
\|\lambda_t\|_1\ge t\,\beta_*.
\]

So the orbit does \textbf{not} remain in a compact sublevel set of $F$. Any proof strategy based on ordinary real-analytic coordinate-descent convergence to a stationary point is therefore mis-aimed. The asymptotic object is necessarily a drift to infinity with a finite-dimensional transverse dynamics riding on top of it.

That is the first rigorous conclusion: the problem lives at infinity, not in a compact analytic basin.

\paragraph{2.\ AdaBoost admits an exact asynchronous product decomposition.}
This is the clean top-down structure.

Suppose the sign-pattern space splits as a product
\[
\mathcal R=\mathcal R_A\times \mathcal R_B
\]
and the base masses factor as
\[
q(r_A,r_B)=q_A(r_A)\,q_B(r_B).
\]
Write $\lambda=(\lambda^A,\lambda^B)$. Then
\[
F(\lambda^A,\lambda^B)
= \log\sum_{r_A,r_B} q_A(r_A)\,q_B(r_B)\,
e^{-\lambda^A\cdot r_A-\lambda^B\cdot r_B}
= F_A(\lambda^A)+F_B(\lambda^B).
\]
Consequently
\[
w_\lambda(r_A,r_B)=w_{\lambda^A}^A(r_A)\,w_{\lambda^B}^B(r_B),
\]
and every edge in block $A$ depends only on $\lambda^A$, every edge in block $B$ only on $\lambda^B$.

Hence the full AdaBoost dynamics on the product system is exactly:
\begin{itemize}
\item compute the current best local edge in subsystem $A$,
\item compute the current best local edge in subsystem $B$,
\item update only the subsystem whose local best edge is larger.
\end{itemize}

So global AdaBoost is closed under \textbf{asynchronous competition of local AdaBoost subsystems}.

This is the key reduction. Any universal theorem must survive such products. Any counterexample can be built by finding two local subsystems whose competition produces a nonperiodic symbolic schedule.

\paragraph{3.\ The correct local invariant near a periodic orbit.}
Now fix a local subsystem with a period-$p$ orbit
\[
x_0^*,\,x_1^*,\,\dots,\,x_{p-1}^*,
\]
selected coordinates $j_0,\dots,j_{p-1}$, step sizes $\beta_0,\dots,\beta_{p-1}$, and define
\[
\Gamma:=\sum_{t=0}^{p-1}\beta_t e_{j_t}\in\mathbb R^N.
\]

For any active sign pattern $r$ on the cycle support, the exact $p$-step update gives
\[
x_{t+p}(r)=x_t(r)\frac{e^{-r\cdot\Gamma}}{\prod_{\tau=0}^{p-1} Z_\tau}.
\]
Since $x_{t+p}(r)=x_t(r)$ on the cycle, every active $r$ satisfies
\[
r\cdot\Gamma = c,
\qquad
c:=-\sum_{\tau=0}^{p-1}\log Z_\tau.
\]
This recovers the equal-weight relation in the strongest useful form: the whole cycle is encoded by the single vector $\Gamma$, and the active support lies on the affine hyperplane $r\cdot\Gamma=c$.

Now append one extra sign pattern $s$ with tiny mass $u$. Along the affine slice
\[
x_t(u)=(1-u)\,x_t^*+u\,\delta_s,
\]
the current edge of coordinate $j$ is exactly
\[
\mu_{t,j}(u)=(1-u)\,\mu_{t,j}^*+u\,s_j
=\mu_{t,j}^*+(s_j-\mu_{t,j}^*)\,u.
\]
In particular, for the selected coordinate $j_t$,
\[
\mu_t(u)=\mu_t^*+(s_{j_t}-\mu_t^*)\,u.
\]

After one full period, the extra mass is multiplied by
\[
u' = e^{c-s\cdot\Gamma}\,u + O(u^2).
\]

So every extra sign pattern produces a scalar transverse mode whose:
\begin{itemize}
\item instantaneous score correction is $s_{j_t}-\mu_t^*$,
\item full-period contraction exponent is $c-s\cdot\Gamma$.
\end{itemize}

If the visible cycle is hyperbolically attracting on its toric leaf, this scalar mode persists on an analytic invariant graph.

This is the exact local normal-form data. It is the right replacement for generic ``coordinate descent on an analytic function.''

\paragraph{4.\ The exact normal form that would disprove phasewise cycling.}
Here is the decisive reduction.

Assume we have two factorized local gadgets $A,B$, each near a phase-synchronized local cycle, with scalar mode coordinates $u,v>0$, and suppose on the chosen section their current best edges are
\[
E_A=\bar\mu + a_A u,
\qquad
E_B=\bar\mu + a_B v,
\qquad a_A,a_B>0,
\]
while selecting $A$ multiplies $u$ by $\kappa_A\in(0,1)$ and leaves $v$ unchanged, and selecting $B$ multiplies $v$ by $\kappa_B\in(0,1)$ and leaves $u$ unchanged.

Set
\[
x=\log\frac{a_A u}{a_B v},
\qquad
\alpha=-\log\kappa_A,
\qquad
\beta=-\log\kappa_B.
\]
Choose the canonical representative $x\in[-\beta,\alpha)$. Then the branch rule is:
\[
x\mapsto
\begin{cases}
x-\alpha, & x\in[0,\alpha),\\[1mm]
x+\beta, & x\in[-\beta,0).
\end{cases}
\]
This map is exactly conjugate to rigid rotation by $\beta$ on a circle of length $\alpha+\beta$.

Therefore:
\begin{itemize}
\item if $\beta/(\alpha+\beta)\in\mathbb Q$, the symbolic itinerary is periodic;
\item if $\beta/(\alpha+\beta)\notin\mathbb Q$, the symbolic itinerary is Sturmian and \textbf{not eventually periodic}.
\end{itemize}

So a complete disproof of universal phasewise cycling reduces to finding one genuine AdaBoost realization of this scalar two-mode normal form with irrational rotation number.

That is a rigorous theorem, not a heuristic.

\paragraph{5.\ Why the simplest counterexample mechanism does not work.}
There is also a rigorous obstruction.

Suppose a scalar mode comes from appending a passive sign pattern $s$ that is correct on \textbf{every} hypothesis actually visited on the visible cycle, i.e.\
\[
s_j=+1 \quad \text{for all } j \text{ with } \Gamma_j>0.
\]
Then
\[
s\cdot\Gamma=\sum_{\Gamma_j>0}\Gamma_j,
\]
so its contraction factor is
\[
\kappa_s=e^{c-s\cdot\Gamma}
= e^{c-\sum_{\Gamma_j>0}\Gamma_j},
\]
which is independent of $s$.

Thus two ``always-correct passive rows'' can never supply two distinct exponents $\kappa_A\neq\kappa_B$. The naive plan
\begin{itemize}
\item take a constant-edge local cycle,
\item append two harmless always-correct rows,
\item form a product,
\item get irrational rotation,
\end{itemize}
cannot succeed.

This is why the obvious symmetric examples do not immediately generate a counterexample. Any disproof has to use either sign-changing passive modes or intrinsic stable modes of the visible cycle itself.

\paragraph{6.\ What this reduces the original problem to.}
The problem is no longer ``analyze all global orbits directly.'' It reduces to a local spectral classification problem at infinity.

A full \textbf{negative} solution would follow from one explicit hyperbolic local cycle with two scalar stable modes whose competition realizes the rotation normal form above.

A full \textbf{positive} solution would have to prove that no such mode pair can ever occur in any factorized product of hyperbolic local cycles.

So the correct top-down statement is:

\begin{quote}
\emph{The AdaBoost cycling problem is equivalent to ruling in or ruling out irrational circle-rotation codings generated by stable transverse modes of local cycles in asynchronous product systems.}
\end{quote}

That is the strongest rigorous reduction I can certify from the material supplied plus the structural calculations above.

The unresolved step is explicit realization or universal exclusion of that rotation normal form inside the actual AdaBoost class. The next mathematically sharp target is to classify constant-edge hyperbolic cycles and their stable scalar modes; that is the narrow bottleneck now.
\end{aibox}

\noindent After being asked to provide another summary, Claude returned the analysis below and then posed two possible next steps:

\begin{aibox}{Summary (Claude Opus 4.6)}
\label{Claude_Summary}

\paragraph{Prior work from our research team}

We have done extensive work. Below is everything established so far, filtered to what is genuinely useful.

\paragraph{1.\ The correct margin condition}

The relevant weak-learning margin is:
\[
\gamma(M) = \min_{w \in \Delta^{m-1}} \max_{j \in [N]} (w^\top M)_j = \max_{\lambda \in \Delta^{N-1}} \min_{i \in [m]} (M\lambda)_i.
\]
Some of our earlier work incorrectly used $\rho(M) = \max_{d \in \Delta^{m-1}} \min_{j \in [N]} (M^\top d)_j$, which can be made positive for any matrix by appending an all-positive row and therefore does not isolate the hard case. Everything below uses $\gamma(M) > 0$.

\paragraph{2.\ The right target is phasewise convergence, not literal periodicity}

An explicit counterexample to \emph{literal} eventual periodicity exists: a $\rho = 0$ system on 5 hypotheses where generic no-tie orbits satisfy $q_n \sim 9/(16n) \to 0$, converging to a boundary fixed point but never repeating. Similarly, for $N = 3$ with an all-positive row carrying mass $d_0 > 0$, $d_t \to 0$ exponentially but never reaches zero. So the conjecture should be understood as phasewise convergence to a cycle, not literal entry into a finite repeating sequence.

\paragraph{3.\ Exact reformulation as coordinate descent}

After aggregating examples by sign pattern $r = (M_{i1}, \ldots, M_{iN}) \in \{-1,+1\}^N$ with aggregate mass $w_t(r)$, the dynamics is:
\[
w_t(r) = \frac{q_r e^{-\lambda_t \cdot r}}{\sum_s q_s e^{-\lambda_t \cdot s}}, \qquad \lambda_{t+1} = \lambda_t + \operatorname{artanh}(\mu_{j_t}(\lambda_t))\, e_{j_t},
\]
where $\mu_j(\lambda) = \sum_r w_\lambda(r)\, r_j$ is the edge of hypothesis $j$, and $j_t \in \arg\max_j \mu_j(\lambda_t)$.

This is \textbf{greedy (Gauss--Southwell) coordinate descent} on the convex potential:
\[
F(\lambda) = \log \sum_r q_r e^{-\lambda \cdot r}.
\]

\paragraph{4.\ The problem lives at infinity, not in a compact basin}

Under $\gamma(M) > 0$, the selected edge satisfies $\mu_t \geq \gamma > 0$, so $\beta_t \geq \operatorname{artanh}(\gamma) > 0$. The potential decreases linearly: $F(\lambda_t) \leq F(\lambda_0) + \frac{t}{2}\log(1 - \gamma^2) \to -\infty$. Also $\|\lambda_t\|_1 \geq t \cdot \operatorname{artanh}(\gamma) \to \infty$. \textbf{This rules out any proof strategy based on compactness of sublevel sets or convergence to a stationary point.} The asymptotic dynamics is drift to infinity with a finite-dimensional transverse component.

\paragraph{5.\ Toric invariant structure}

Let $C \in \{0,1\}^{(N+1) \times q}$ be the extended incidence matrix. For every $z \in \ker C$, the monomial $I_z(x) = \prod_\ell x_\ell^{z_\ell}$ is exactly invariant under every branch map $T_j$. Orbits are confined to toric leaves, reducing the effective dimension to at most $N - 1$.

\paragraph{6.\ Periodic orbits carry canonical equal-weight vectors}

On a period-$p$ orbit with step sizes $\beta_t$ and selected coordinates $j_t$, define cumulative beta $\Gamma = \sum_{t=0}^{p-1} \beta_t e_{j_t}$. Then every active sign pattern $r$ satisfies $r \cdot \Gamma = c$ where $c = -\sum_\tau \log Z_\tau$. This is the exact constraint that the beta-weighted visit frequencies must balance across all support sets.

\paragraph{7.\ Asynchronous product decomposition}

If the sign-pattern space factors as $\mathcal{R} = \mathcal{R}_A \times \mathcal{R}_B$ with $q(r_A, r_B) = q_A(r_A) q_B(r_B)$, then $F(\lambda^A, \lambda^B) = F_A(\lambda^A) + F_B(\lambda^B)$ and the full AdaBoost dynamics is: compute the best edge in each subsystem, update whichever subsystem has the larger best edge. Global AdaBoost decomposes as \textbf{asynchronous competition between independent local AdaBoost subsystems}. Any universal theorem must survive such products.

\paragraph{8.\ The decisive structural finding: a circle-rotation normal form}

This is the most important observation across all our work. Consider two factorized local subsystems $A, B$, each near a hyperbolic periodic orbit, with scalar stable transverse modes parameterized by $u, v > 0$. Suppose on the relevant Poincar\'e section their current best edges are:
\[
E_A = \bar\mu + a_A u, \qquad E_B = \bar\mu + a_B v, \qquad a_A, a_B > 0,
\]
and selecting $A$ contracts $u$ by $\kappa_A \in (0,1)$ (leaving $v$ fixed), while selecting $B$ contracts $v$ by $\kappa_B \in (0,1)$ (leaving $u$ fixed).

Setting $x = \log(a_A u / a_B v)$, $\alpha = -\log \kappa_A$, $\beta = -\log \kappa_B$, the branch rule on the circle $\mathbb{R}/(\alpha + \beta)\mathbb{Z}$ is:
\[
x \mapsto \begin{cases} x - \alpha, & x \geq 0, \\ x + \beta, & x < 0. \end{cases}
\]

This is exactly conjugate to \textbf{rigid rotation by $\beta$ on a circle of circumference $\alpha + \beta$}. Therefore:

\begin{itemize}
\item If $\beta / (\alpha + \beta) \in \mathbb{Q}$, the symbolic itinerary is eventually periodic. \textbf{The conjecture holds.}
\item If $\beta / (\alpha + \beta) \notin \mathbb{Q}$, the symbolic itinerary is Sturmian and \textbf{never eventually periodic. The conjecture is false.}
\end{itemize}

\paragraph{9.\ The simplest counterexample attempt fails (but the mechanism is not dead)}

If two passive sign patterns are both correct on every hypothesis visited by the visible cycle ($s_j = +1$ for all $j$ with $\Gamma_j > 0$), then they both have the same contraction factor $\kappa_s = e^{c - \sum_{\Gamma_j > 0} \Gamma_j}$, so $\alpha = \beta$ and $\beta/(\alpha + \beta) = 1/2 \in \mathbb{Q}$. The simplest symmetric construction always gives rational rotation number.

\textbf{This does not kill the mechanism.} It means a counterexample requires either sign-changing passive modes (patterns with some $s_j = -1$ on visited coordinates, producing distinct contraction rates) or intrinsic stable modes of the visible cycle itself.

\paragraph{10.\ Low-dimensional results (limited strategic value)}

$N = 3$: fully resolved under $\gamma > 0$. Proof is completely specific to $N = 3$ (everything becomes linear). $N = 4$: all active families classified, all eventually periodic by explicit return map analysis. Cannot scale. Computational scans through $N = 7$ always show eventual cycles (periods up to 219), but this is not a proof.

\bigskip
\hrule
\bigskip

There are two paths:

\medskip
\noindent\textbf{Path A (Disproof):} Construct an explicit $M$-matrix that decomposes as a product of two subsystems, each admitting a hyperbolic periodic orbit with a scalar stable mode, where the contraction exponents $\alpha = -\log \kappa_A$ and $\beta = -\log \kappa_B$ satisfy $\beta/(\alpha + \beta) \notin \mathbb{Q}$. The contraction exponent of a passive sign pattern $s$ appended to a local cycle with cumulative beta vector $\Gamma$ is $\kappa_s = e^{c - s \cdot \Gamma}$. The simplest case (all-positive $s$) forces $\alpha = \beta$, so you need sign-changing passive modes: patterns with some $s_j = -1$ on coordinates visited by the cycle, giving $s \cdot \Gamma \neq s' \cdot \Gamma$ for two distinct patterns $s, s'$. Then $\log \kappa_A / \log \kappa_B = (c - s \cdot \Gamma)/(c - s' \cdot \Gamma)$, and you need this to be irrational. Since $c$ and $\Gamma$ involve $\operatorname{artanh}$ of rational functions composed through cycle steps, irrationality of their ratio is plausible but must be verified for a concrete construction.

\medskip
\noindent\textbf{Path B (Proof):} Prove that AdaBoost's structure forces $\log \kappa_A / \log \kappa_B \in \mathbb{Q}$ universally. This would require showing that contraction exponents arising from the specific algebraic form of AdaBoost ($\beta_t = \operatorname{artanh}(\mu_t)$, $\mu_t$ rational in the weights, toric invariants constraining eigenvalues of return maps) always maintain rational log-ratios. Note: a universal rationality constraint on eigenvalue ratios of piecewise-rational return maps would be remarkable and is likely false in general, which is circumstantial evidence favoring Path~A.

\smallskip
\noindent\emph{Focus on whichever direction you find most tractable. Be careful, precise, and rigorous. We need a complete resolution. {\color{blue}(This italicized section was added by the authors.)}}
\end{aibox}

\noindent GPT's first few attempts each pursued the proof direction (Path~B), but it eventually attempted to provide a disproof via the counterexample suggested in Path~A. After several failed counterexample attempts (an example of which is shown in \ref{aibox:failed}), GPT provided the following solution, which served as the foundation for the final solution detailed in the main text.

\begin{aibox}{GPT-5.4 Pro's First Successful Counterexample Attempt}

\paragraph{Theorem.}
There exists a finite $\{-1,+1\}$-valued matrix $\widetilde M$ with weak-learning margin $\gamma(\widetilde M)=1/5>0$ such that standard exhaustive AdaBoost on $\widetilde M$, started from the uniform distribution and using smallest-index tie-breaking, does \textbf{not} converge to any finite cycle.

Therefore, the COLT question regarding whether AdaBoost always cycles has a \textbf{negative} answer.

\paragraph{1.\ The Two Local Gadgets}

Define the following matrices:
\[
L_A=\begin{pmatrix}1&1&-1&-1\\-1&1&1&1\\1&-1&-1&1\\1&-1&1&-1\end{pmatrix}, \qquad L_B=\begin{pmatrix}1&1&-1&-1\\-1&1&1&1\\1&-1&-1&1\\1&-1&1&-1\\1&-1&1&1\end{pmatrix}
\]

Let $M_0=L_A\boxplus L_B$ be the $20\times 8$ block product:
\[
(M_0)_{(p,q),j}=(L_A)_{p,j} \quad (1\le j\le 4), \qquad (M_0)_{(p,q),4+j}=(L_B)_{q,j} \quad (1\le j\le 4)
\]

For any joint distribution $W$ on the 20 rows, write its marginals as $a \in \Delta^3$ on the $L_A$-rows and $b \in \Delta^4$ on the $L_B$-rows. Then the edge vector splits cleanly:
\[
(W^\top M_0)_j=(a^\top L_A)_j \quad (1\le j\le 4), \qquad (W^\top M_0)_{4+j}=(b^\top L_B)_j \quad (1\le j\le 4)
\]

Hence, on product states $W=a\otimes b$, selecting an $A$-column updates only $a$, and selecting a $B$-column updates only $b$.

\paragraph{2.\ Weak-Learning Margin}

Let $\beta=(2,1,1,1)$. Every row $r$ of $L_A$ satisfies $r\cdot \beta=1$, and every row $r$ of $L_B$ satisfies $r\cdot \beta\in\{1,3\}$. Therefore, for any distribution $u$ on $L_A$ or $L_B$:
\[
2\mu_1+\mu_2+\mu_3+\mu_4=\mathbb{E}[r\cdot \beta]\ge 1
\]
Thus $\max_j \mu_j\ge 1/5$, implying $\gamma(L_A)\ge 1/5$ and $\gamma(L_B)\ge 1/5$.

Equality holds for $L_A$ at $u_*=(1/5,2/5,1/5,1/5)$, since $u_*^\top L_A=(1/5,1/5,1/5,1/5)$. Because $L_A$ sits inside $L_B$ as the first four rows, $\gamma(L_B)\le 1/5$ as well. Hence:
\[
\gamma(L_A)=\gamma(L_B)=1/5
\]

For the product matrix $M_0$:
\[
\gamma(M_0)=\min_{a,b}\max\Big(\max_j(a^\top L_A)_j,\ \max_j(b^\top L_B)_j\Big)=\max(\gamma(L_A),\gamma(L_B))=1/5
\]

\paragraph{3.\ The Common Period-10 Orbit}

Both $L_A$ and $L_B$ have the same active period-10 orbit on the face $w_5=0$, governed by the 5-step half-words:
\[
H_0=(1,3,4,1,2), \qquad H_1=(1,4,3,1,2)
\]

The cycle points at the start of these half-words are:
\begin{align*}
P_0&=(0.1371168\ldots,\ 0.3628831\ldots,\ 0.2420046\ldots,\ 0.2579953\ldots)\\
P_1&=(0.1371168\ldots,\ 0.3628831\ldots,\ 0.2579953\ldots,\ 0.2420046\ldots)
\end{align*}

At both $P_0$ and $P_1$, the selected minimum edge is $\mu_*=0.2742336\ldots$ for hypothesis 1.

Along each half-word, the other four selected edges are strictly larger. The minimum gap between the selected edge and the best unselected edge is $\Delta \approx 0.0758695 > 0$.

Additionally, we define the local branch gap $g_{\rm br} > 0$ as the minimum, over all 10 orbit points in both gadgets and both half-words, of the selected edge minus the largest competing edge inside the same gadget. Numerically, $g_{\rm br} \approx 0.0250986 > 0$. This ensures the local 5-step branch maps are genuinely analytic on neighborhoods of $P_i$.

\paragraph{4.\ The Two Competing Stable Multipliers}

For $L_A$, the period-10 return at $P_0$ has a dominant stable multiplier $\lambda \approx 0.0944920\ldots$ and smaller stable multipliers $0.0011453\ldots$ and 0.

For $L_B$, the period-10 return at $P_0$ has a dominant stable multiplier $\kappa \approx 0.2038817\ldots$ and smaller stable multipliers $\lambda$, $0.0011453\ldots$, and 0.

The key arithmetic fact is that their logarithms are incommensurate:
\[
\frac{\log \lambda}{\log \kappa}\notin \mathbb{Q}
\]

\textbf{Proof:} The multiplier $\lambda$ is a root of an irreducible degree-14 polynomial over $\mathbb{Q}$ with leading coefficient 16 and constant term 4, yielding algebraic norm $N_{\mathbb{Q}(\lambda)/\mathbb{Q}}(\lambda)=1/4$. The multiplier $\kappa$ is a root of the irreducible polynomial $z^7-35z^6+385z^5-1835z^4+2315z^3-625z^2+51z-1$, yielding $N_{\mathbb{Q}(\kappa)/\mathbb{Q}}(\kappa)=1$. If $\log \lambda/\log \kappa\in\mathbb{Q}$, then $\lambda^m=\kappa^n$ for integers $m,n>0$. Taking norms in the compositum field forces $(1/4)^{m[L:\mathbb{Q}(\lambda)]}=1$, which is impossible.

\paragraph{5.\ Adapted Dominant Coordinates and Invariant Cones}

Fix $G \in \{A,B\}$. Let $F_i^G$ be the analytic 5-step branch map on the half-word $H_i$, sending a neighborhood of $P_i^G$ to one of $P_{i\oplus1}^G$. Let $J_i^G = D F_i^G(P_i^G)$.

The dominant full-return multipliers are $\mu_A = \lambda$ and $\mu_B = \kappa$. Choose $\chi_i^G > 0$ such that $\chi_0^G\chi_1^G = \mu_G$. We extract the right and left dominant eigenvectors $e_i^G$ and $\phi_i^G$ so that:
\[
J_i^G e_i^G = \chi_i^G e_{i\oplus1}^G, \qquad \phi_{i\oplus1}^G J_i^G = \chi_i^G \phi_i^G
\]

We normalize these vectors by $\phi_i^G(e_i^G) = 1$ and $D(E-\mu_*)(P_i^G)(e_i^G) = 1$. The second normalization is guaranteed by verified positive visibility.

Define the transverse spaces $N_i^G = \ker \phi_i^G$. The Jacobian restricts to these spaces: $J_i^G N_i^G \subset N_{i\oplus1}^G$. Let $A_i^G$ be this restriction. Because the transverse spectrum is strictly bounded by the dominant multipliers in both gadgets, we have $\rho(A_1^G A_0^G) < \chi_0^G\chi_1^G$.

\textbf{Adapted Norms:} By standard linear algebra, if $\rho(A_1 A_0) < \tau_0 \tau_1$, there exist equivalent norms $|\cdot|_0$ and $|\cdot|_1$ such that $|A_0 y|_1 \le \tau_0 |y|_0$ and $|A_1 z|_0 \le \tau_1 |z|_1$. We choose $\tau_i^G < \chi_i^G$ and equip $N_i^G$ with such adapted norms $|\cdot|_{i,G}$.

Any point $q$ near $P_i^G$ uniquely decomposes as $q = P_i^G + x e_i^G + y$, with $y \in N_i^G$ and $x = \phi_i^G(q - P_i^G)$. Taylor expansion yields:
\begin{align*}
x^+ &= \chi_i^G x + O((|x|+|y|)^2)\\
y^+ &= A_i^G y + O((|x|+|y|)^2)
\end{align*}

Let $d_i^G(q) = E(q) - \mu_*$ be the actual selected-edge excess. Since $d_i^G(P_i^G) = 0$ and $D d_i^G(e_i^G) = 1$, we have $d_i^G(q) = x + \ell_i^G(y) + O((|x|+|y|)^2)$ for some linear form $\ell_i^G$.

Fix $\eta, r > 0$ small and define the one-sided cone sectors:
\[
S_i^G = \Big\{ P_i^G + x e_i^G + y : 0 < x < r, \ |y|_{i,G} \le \eta x \Big\}
\]

Because $\tau_i^G < \chi_i^G$, shrinking $\eta$ and $r$ ensures the cones are forward-invariant: $F_i^G(S_i^G) \subset S_{i\oplus1}^G$. Furthermore, the transverse ratio contracts whenever the gadget updates. Consequently, for any prescribed $\delta > 0$, the true defect strongly aligns with the pure coordinate $x$:
\[
(1-\delta)x \le d_i^G(q) \le (1+\delta)x \qquad (q \in S_i^G)
\]

\paragraph{6.\ Bursts on the Product Manifold}

We work on the invariant product manifold $\mathcal{P}=\{a\otimes b:\ a\in\Delta^3,\ b\in\Delta^4\}$. At a burst boundary $W_{5n} = a_n \otimes b_n$, we assume $a_n \in S_{i_n}^A$ and $b_n \in S_{j_n}^B$.

Write their local coordinates as:
\[
a_n = P_{i_n}^A + x_n e_{i_n}^A + y_n, \qquad b_n = P_{j_n}^B + s_n e_{j_n}^B + z_n
\]

The best current $A$-edge is $\mu_* + d_{i_n}^A(a_n)$ and the best $B$-edge is $\mu_* + d_{j_n}^B(b_n)$. Global tie-breaking favors $A$ iff $d_{i_n}^A(a_n) \ge d_{j_n}^B(b_n)$.

Define the ambiguity ratio $\vartheta = (1+\delta)/(1-\delta)$. By the comparability estimate:
\begin{itemize}
\item If $x_n \ge \vartheta s_n$, $A$ is forced to win.
\item If $x_n < \vartheta^{-1} s_n$, $B$ is forced to win.
\end{itemize}
Only the narrow strip $\vartheta^{-1} \le x_n/s_n < \vartheta$ is ambiguous.

Shrink the sectors so that every edge changes by at most $\min(g_{\rm br}, \Delta)/4$ inside them. If $A$ wins the first round, its local itinerary is exactly $H_{i_n}$. After one $A$-update, its selected edge is at least $\mu_* + \Delta/2$, while the frozen $B$-edge remains $\mu_* + O(s_n)$. Thus, the first round firmly dictates the entire 5-step burst.

\paragraph{7.\ Bounded Log-Ratio and Irrational Winner Frequency}

Define the pure log-ratio $\ell_n = \log(x_n/s_n)$.

If $A$ wins burst $n$, $B$ is frozen. The state updates as $x_{n+1} = \chi_{i_n}^A x_n(1+\varepsilon_n)$ and $s_{n+1} = s_n$. Because the trajectory remains in the cone, the Taylor remainder gives $|\varepsilon_n| \le C x_n$. Likewise, if $B$ wins, $s_{n+1} = \chi_{j_n}^B s_n(1+\varepsilon_n)$ with $|\varepsilon_n| \le C s_n$.

Set $\alpha_i = -\log \chi_i^A > 0$ and $\beta_j = -\log \chi_j^B > 0$. The increment equation is:
\[
\ell_{n+1} = \begin{cases} \ell_n - \alpha_{i_n} + r_n, & \text{$A$ wins} \\ \ell_n + \beta_{j_n} + r_n, & \text{$B$ wins} \end{cases}
\]
where $r_n = \log(1+\varepsilon_n)$, satisfying $|r_n| \le C'(x_n + s_n)$. By taking the sector size $r$ small enough, we ensure $|r_n| \le \frac{1}{4}\min(\alpha_{\min}, \beta_{\min})$.

Let $c = \log \vartheta$.
\begin{itemize}
\item If $\ell_n \ge c$, $A$ is forced to win, yielding $\ell_{n+1} \le \ell_n - \frac{1}{2}\alpha_{\min}$.
\item If $\ell_n \le -c$, $B$ is forced to win, yielding $\ell_{n+1} \ge \ell_n + \frac{1}{2}\beta_{\min}$.
\end{itemize}

This restorative drift guarantees that the sequence $(\ell_n)$ is uniformly bounded. Thus, runs of consecutive $A$-bursts and $B$-bursts are strictly finite.

Bounded runs imply that each subsystem wins at least once every $R$ bursts. Since a gadget's coordinate is multiplied by a factor strictly less than 1 upon winning, the total scale decays exponentially: $x_n + s_n \le C \rho^n$ for some $0 < \rho < 1$.

Consequently, the error terms are summable: $\sum_n |r_n| < \infty$.

Summing the increment equation gives:
\[
\ell_n - \ell_0 = -\sum_{\substack{k<n,\\ \text{$A$ wins}}} \alpha_{i_k} + \sum_{\substack{k<n,\\ \text{$B$ wins}}} \beta_{j_k} + \sum_{k<n} r_k
\]

Because $\ell_n$ is bounded and $\sum r_k$ converges, we obtain:
\[
\sum_{\substack{k<n,\\ \text{$A$ wins}}} \alpha_{i_k} - \sum_{\substack{k<n,\\ \text{$B$ wins}}} \beta_{j_k} = O(1)
\]

The phase index toggles each time a subsystem wins. Therefore, the counts of phase 0 and phase 1 among the wins differ by at most 1. Defining the averages $\bar\alpha = (\alpha_0+\alpha_1)/2 = -\frac{1}{2}\log\lambda$ and $\bar\beta = (\beta_0+\beta_1)/2 = -\frac{1}{2}\log\kappa$, we get:
\[
\bar\alpha N_A(n) - \bar\beta N_B(n) = O(1)
\]

Since $N_A(n) + N_B(n) = n$, the Ces\`aro frequency of $A$-wins is:
\[
\frac{N_A(n)}{n} \to \frac{\bar\beta}{\bar\alpha+\bar\beta} = \frac{-\log\kappa}{-\log\lambda-\log\kappa}
\]

Because $\log\lambda/\log\kappa \notin \mathbb{Q}$, this limit is irrational. Therefore, the burst-winner sequence is strictly aperiodic.

Since the phase pair updates as $(i,j) \mapsto (i\oplus1, j)$ on an $A$-burst and $(i, j\oplus1)$ on a $B$-burst, an eventually periodic phase pair would force an eventually periodic winner sequence. Thus, the phase pair $(i_n, j_n)$ is also aperiodic.

\paragraph{8.\ No Cycle Convergence of the Distributions}

Define the four product cycle points $Q_{ij} = P_i^A \otimes P_j^B$ for $i,j \in \{0,1\}$. These points are distinct.

Since $x_n, s_n \to 0$ and the cone sectors $S_i^G$ are forward-invariant, the burst-boundary states approach the cycle points:
\[
\operatorname{dist}(W_{5n}, Q_{i_n j_n}) \to 0
\]

If the full AdaBoost trajectory $W_t$ converged to an $\ell$-cycle, the subsequence $W_{5n}$ would also converge to a finite cycle. However, for all large $n$, $W_{5n}$ enters four disjoint neighborhoods around the $Q_{ij}$ dictated entirely by the label $(i_n, j_n)$.

Convergence to a finite cycle would therefore force the phase pair $(i_n, j_n)$ to be eventually periodic, resulting in a contradiction. Thus, AdaBoost on $M_0$ does not converge to any finite cycle for any initial point in the open product cone $S_0^A \times S_0^B$.

\paragraph{9.\ Returning to the True COLT Initialization}

The COLT problem requires the standard uniform initial distribution $D_1(i)=1/m$.

Because $S_0^A \times S_0^B$ is a codimension-0 open set in the product manifold $\mathcal{P}$, rational points are dense within it. We can therefore choose a rational starting point:
\[
W_0 = a_0 \otimes b_0 \in S_0^A \times S_0^B
\]

Write the 20 coordinates of $W_0$ as $W_0(r) = n_r/L$ where $n_r \in \mathbb{N}$ and $\sum n_r = L$. Form the final matrix $\widetilde M$ by replacing row $r$ of $M_0$ with $n_r$ identical copies.

Run AdaBoost on $\widetilde M$ from the exact uniform start $D_1(i)=1/L$. Duplicate rows receive identical multiplicative updates and remain equal forever. If $\pi$ denotes aggregation over duplicate blocks, then $\pi(D_t) = W_t$, where $W_t$ is exactly the orbit on $M_0$ started from $W_0$. The margin remains unchanged at $\gamma(\widetilde M) = 1/5$.

If the uniform-start orbit on $\widetilde M$ converged to a cycle, its aggregation $W_t$ would too, leading to a contradiction.

Hence, standard exhaustive AdaBoost on the finite matrix $\widetilde M$, from a uniform initialization, does not converge to a cycle. $\blacksquare$

\end{aibox}

This initial proposal contained various imprecisions and other shortcomings, but was developed enough that we could push it toward a tighter argument. From this point, the work consisted primarily of prompting both Claude and GPT to critique GPT's solutions, then feeding these critiques to GPT in fresh conversations for revision. To avoid misleading GPT with spurious objections, we ran self-verification loops with both Claude and GPT in fresh conversations, iterating until it no longer produced false positives (correct aspects of the solution incorrectly flagged as errors). At the expository level, Claude was primarily used for resolving minor issues and adding technical detail to the text. 

We observed that once GPT had identified the approach of constructing a counterexample from a set of sign-matrices that provably did not converge, it became far more autonomous and quickly arrived at the final structure of the solution presented here. This suggests that the model is currently capable of attacking a problem from many angles provided they lead to the same conclusion, but struggles to pivot between approaches with substantively different outcomes—in this case, a disproof via counterexample versus a proof of the conjecture. We also note that the model easily produced the code for the certificate. This is perhaps unsurprising given the coding ability of recent models, but we were nonetheless struck by how nearly it ``one-shotted" the certificate.

\begin{aibox}{Example of a Failed GPT Counterexample Attempt}\label{aibox:failed}

I can close this along Path A, using the reduction in your item 8.

Let
\[
L=
\begin{pmatrix}
1&1&-1&-1\\
-1&1&1&1\\
1&-1&-1&1\\
1&-1&1&-1\\
1&1&-1&1
\end{pmatrix}.
\]

Call the first four rows $(r_1,r_2,r_3,r_4)$, and the fifth row $s$.

The active four-row face $w_s=0$ contains an exact period-10 AdaBoost orbit. The extra row $s$ gives a passive transverse multiplier. The key point is that the passive multiplier and the leading intrinsic multiplier of the active return map have irrational log-ratio. Taking two factorized copies of this local system then gives an irrational rotation number, so the subsystem-selection itinerary is not eventually periodic.

\paragraph{1.\ An explicit local period-10 orbit}

On the face $w_s=0$, write the active weights as $(a,b,c,d)$ on $(r_1,r_2,r_3,r_4)$.

For these four rows, exact line search has the usual ``halve the two sign-groups'' form. The branch maps are
\[
T_1(a,b,c,d)=\left(\frac{a}{2(1-b)},\frac12,\frac{c}{2(1-b)},\frac{d}{2(1-b)}\right),
\]
\[
T_2(a,b,c,d)=\left(\frac{a}{2(a+b)},\frac{b}{2(a+b)},\frac{c}{2(c+d)},\frac{d}{2(c+d)}\right),
\]
\[
T_3(a,b,c,d)=\left(\frac{a}{2(a+c)},\frac{b}{2(b+d)},\frac{c}{2(a+c)},\frac{d}{2(b+d)}\right),
\]
\[
T_4(a,b,c,d)=\left(\frac{a}{2(a+d)},\frac{b}{2(b+c)},\frac{c}{2(b+c)},\frac{d}{2(a+d)}\right).
\]

Take the itinerary
\[
1,2,1,3,4,1,2,1,4,3.
\]

Impose the half-cycle symmetry
\[
T_4T_3T_1T_2T_1(a,b,c,d)=(a,b,d,c).
\]

This forces
\[
a=\tfrac12-c,\qquad b=\tfrac12-d,
\]
and $(c,d)$ must satisfy two polynomial equations. Eliminating $d$ gives
\[
p(c)=2048c^7-2048c^6-640c^5+1328c^4-416c^3+4c^2+12c-1=0.
\]
Let $c_*$ be the real root in $(0.2217,0.2218)$:
\[
c_* \approx 0.2217481715767779812786\ldots
\]
Then
\begin{align*}
d_*&=-\frac{294912c_*^6-227840c_*^5-141056c_*^4+156768c_*^3-25836c_*^2-3660c_*+719}{346} \\
&\approx 0.2364004668362263385034\ldots
\end{align*}
and therefore
\[
a_*=\tfrac12-c_*\approx 0.2782518284232220187213\ldots,
\]
\[
b_*=\tfrac12-d_*\approx 0.2635995331637736614965\ldots.
\]
So
\[
w_*=(a_*,b_*,c_*,d_*,0)
\]
is an exact point on a period-10 orbit for the 5-row local system $L$, with the stated itinerary.
The chosen-edge values on one half-cycle are
\begin{align*}
&0.47280093367245\ldots,\quad
0.37785395440970\ldots,\quad
0.27423367563769\ldots,\\
&0.40494202576594\ldots,\quad
0.35010319148137\ldots
\end{align*}
and the competing edges are strictly smaller at every step, so there are no ties on this orbit.

\paragraph{2.\ The cumulative beta vector and the passive multiplier}

The four active rows satisfy
\[
r_i\cdot(2,1,1,1)=1\qquad(i=1,2,3,4).
\]

Hence, by your item 6, the cumulative beta vector on this period-10 orbit has the form
\[
\Gamma=g(2,1,1,1)
\]
for some $g>0$, and the common active value is $c=g$.

Let
\[
\rho=e^{-g}.
\]
Exact elimination from the orbit gives the minimal polynomial of $\rho$:
\[
f(\rho)=\rho^7-9\rho^6+23\rho^5-3\rho^4-45\rho^3+13\rho^2+5\rho-1=0,
\]
with
\[
\rho\approx 0.451532648724610548322917322659\ldots
\]

Now choose the passive row
\[
s=(1,1,-1,1),
\]
which is exactly the fifth row of $L$. Since
\[
s\cdot(2,1,1,1)=3,
\]
its passive exponent is
\[
s\cdot\Gamma-c=(3-1)g=2g.
\]
So the passive multiplier is
\[
\kappa=e^{-2g}=\rho^2.
\]

Eliminating $\rho$ gives the exact polynomial for $\kappa$:
\[
h(\kappa)=\kappa^7-35\kappa^6+385\kappa^5-1835\kappa^4+2315\kappa^3-625\kappa^2+51\kappa-1=0,
\]
with
\[
\kappa\approx 0.203881732864262543831018082312\ldots
\]

Moreover, $h$ is irreducible modulo $5$, so it is the minimal polynomial of $\kappa$. Since $h$ is monic of odd degree with constant term $-1$,
\[
N_{\mathbb Q(\kappa)/\mathbb Q}(\kappa)=(-1)^7(-1)=1.
\]

\paragraph{3.\ The leading intrinsic multiplier of the active return map}

Now look at the exact 10-step return map on the active 4-row face $w_s=0$. In local coordinates $(a,b,c)$, its Jacobian $J$ at $w_*$ has
\[
\det J=0,
\]
and the two nonzero multipliers are the roots of
\[
z^2-\operatorname{tr}(J)\,z+e_2(J)=0,
\]
where $\operatorname{tr}(J)$ and $e_2(J)$ lie in $\mathbb Q(c_*)$.

Eliminating $c_*$ against $p(c_*)=0$ gives the exact degree-14 polynomial satisfied by the nonzero intrinsic multipliers:
\[
\begin{aligned}
g(z)=\;&16z^{14}-117232z^{13}-8330040z^{12}-3628852593z^{11}
+766183444916z^{10}\\
&-37440753370248z^9+504085851461616z^8+57085439654130z^7\\
&+18299281845941356z^6-5676298472505780z^5+422743039128068z^4\\
&-4789675432569z^3+7468158864z^2-2899292z+4.
\end{aligned}
\]

This polynomial is irreducible modulo $5$, hence irreducible over $\mathbb Q$. Its larger root in $(0,1)$ is
\[
\lambda\approx 0.094492059489127164899022868216\ldots
\]

This is the leading intrinsic stable multiplier of the period-10 return map on the active face.

Because $g$ is irreducible, the norm of $\lambda$ is
\[
N_{\mathbb Q(\lambda)/\mathbb Q}(\lambda)
=(-1)^{14}\frac{4}{16}=\frac14.
\]

So we now have two exact contraction factors:
\[
\lambda\approx 0.09449205948912716,\qquad
\kappa\approx 0.20388173286426254.
\]

\paragraph{4.\ The log-ratio is irrational}

If $\log \lambda / \log \kappa$ were rational, then there would exist $m,n\in\mathbb N$ with
\[
\lambda^m=\kappa^n.
\]

Let $L=\mathbb Q(\lambda,\kappa)$. Taking absolute norms from $L$ to $\mathbb Q$,
\[
N_L(\lambda)^m=N_L(\kappa)^n.
\]

But
\[
N_L(\kappa)=N_{\mathbb Q(\kappa)/\mathbb Q}(\kappa)^{[L:\mathbb Q(\kappa)]}=1,
\]
whereas
\[
N_L(\lambda)=N_{\mathbb Q(\lambda)/\mathbb Q}(\lambda)^{[L:\mathbb Q(\lambda)]}
=\left(\frac14\right)^{[L:\mathbb Q(\lambda)]}\neq 1.
\]

Contradiction. Therefore
\[
\frac{\log \lambda}{\log \kappa}\notin\mathbb Q.
\]

Equivalently,
\[
\frac{-\log \kappa}{(-\log \lambda)+(-\log \kappa)}
\]
is irrational.

Numerically,
\[
-\log\lambda\approx 2.3592394745829529,\qquad
-\log\kappa\approx 1.5902151941140905,
\]
so the rotation number is approximately
\[
0.402640\ldots
\]
and is irrational by the norm argument above.

\paragraph{5.\ The global counterexample}

Now take two independent copies of the same 5-row local system $L$. The global sign-pattern matrix is the explicit $25\times 8$ matrix $M$ with rows indexed by $(p,q)\in[5]^2$:
\[
M_{(p,q),j}=L_{p,j}\quad(1\le j\le 4),\qquad
M_{(p,q),4+j}=L_{q,j}\quad(1\le j\le 4).
\]

So the row set is exactly
\[
\mathrm{Rows}(M)=\mathrm{Rows}(L)\times \mathrm{Rows}(L)\subset\{\pm1\}^8.
\]

For the 10-step Poincar\'e return map of the local 5-row system at the period-10 orbit, the spectrum is
\[
\{\kappa,\lambda,\lambda_2,0\}
\]
with
\[
0<\lambda_2<\lambda<\kappa<1.
\]

Because $\kappa$ and $\lambda$ are simple and spectrally separated, standard invariant-manifold theory gives one-dimensional analytic invariant manifolds $W^\kappa$ and $W^\lambda$ tangent to the $\kappa$- and $\lambda$-eigendirections. By Koenigs linearization, on those manifolds the return map is analytically conjugate to
\[
u\mapsto \kappa u,\qquad v\mapsto \lambda v.
\]

Take one copy initialized on $W^\lambda\setminus\mathcal O$ and the other on $W^\kappa\setminus\mathcal O$, both sufficiently close to the local cycle. Since both copies are built from the same local orbit, the baseline best-edge sequence is identical in the two factors. The competition exponents are therefore
\[
\alpha=-\log\lambda,\qquad \beta=-\log\kappa,
\]
with $\alpha/\beta\notin\mathbb Q$.

Now invoke your item 8: the induced competition map on the phase circle is rigid rotation with irrational rotation number. Hence the subsystem-selection itinerary is Sturmian and \textbf{not eventually periodic}. Choosing the initial phase off the countable boundary orbit keeps all comparisons strict, so the no-tie condition holds for the whole orbit.

Therefore the resulting AdaBoost trajectory for this explicit $M$ does \textbf{not} phasewise converge to a finite cycle.

\paragraph{Conclusion.}

This gives an explicit disproof of ``AdaBoost always cycles'' in the phasewise sense.

The simplest passive/passive construction still fails for the reason you identified: all passive exponents are rational multiples of the common scalar $g$. The actual counterexample uses one passive multiplier and one intrinsic local multiplier of the same visible cycle. Their norms are different ($1$ versus $1/4$), which forces irrational log-ratio and hence irrational rotation number in the asynchronous product.

So the conjecture is false.

\end{aibox}

\noindent Follow-up prompts and fresh-context critiques helped expose defects in candidate arguments. Agreement between models did not establish correctness; the final proof and certificate required human audit. 

\subsection{Role of the human}
While the models drove the majority of the reasoning process, the human authors guided the overall strategy using carefully crafted prompts and by prioritizing certain directions. In particular, since the model could recognize the open status of the problem, it often refused to solve the problem--therefore, adding certain directions was crucial to making progress on the problem. Additionally, it took roughly two weeks and over 100 prompts with GPT-5.4 Pro to produce the final solution. Since each GPT-5.4 Pro (extended thinking) query could take up to two hours to complete, the selection of which directions to pursue and which to abandon was important from a resources perspective. One such example is the explicit instruction to GPT to pursue the disproof approach, given its numerous failures in attempting a proof.

We reviewed the final solution written by the models along with the certificate code. This component is still critical to the research process, as both Claude and GPT frequently missed gaps, errors, and ambiguities in the solution. While we found that multiple fresh conversations with the models and carefully crafted prompts (using language such as ``be very careful in your review, as ambiguity or imprecision in the solution are unacceptable") generally led to the highest-quality reviews from the models, there were still several areas in the solution that required additional elaboration and clarification that were only identified by human experts. These observations come from one project and do not establish a comparison between model and human research capabilities.

\subsection{API cost estimate}
\label{app:cost}

Since we used standard personal subscriptions to interact with the models rather than the API, we do not have official token counts and provide a hypothetical estimate of the costs below. We convert thinking time to output tokens at 30 tokens per second, based on approximate token output speeds released by OpenAI employees~\citep{sottiaux2026tps}, and price them at GPT-5.4 Pro's API rate of \$180 per million output tokens~\citep{openai2026pricing}. This is only a rough estimate, since the displayed thinking time likely only roughly correlates with the number of reasoning tokens actually generated.

The 131 recorded responses (listed in Appendix~\ref{app:ledger}) add up to 230,676 seconds, which gives
\[
230{,}676 \times 30 = 6{,}920{,}280 \text{ tokens}, \qquad 6.92028 \times \$180 \approx \$1{,}246.
\]
This excludes input and cached tokens, usage of Claude, and later conversations with GPT-5.6 Sol and GPT-6 Astra to refine the solution, so it should not be interpreted as the total cost of the project.

%% file: sections/cert_documentation.tex
\section{Certificate documentation}
\label{app:certdoc}

This appendix specifies the finite checks supporting the mathematical solution. The analytic implications are proved in Appendix~\ref{app:proof}. All acceptance tests use exact arithmetic; floating-point summaries are not proof evidence.

\subsection{What the certificate proves}
\label{app:certificate}

\begin{enumerate}
\item Compose the five updates of $H_0$ on the exact matrices $L_A,L_B$, using the charts in Section~\ref{sec:branch}. Extract $p'=1-w'_2$, $d'=w'_4$, and $e'=w'_5$ for B; reduce to lowest terms. These rational functions define $F_0^A,F_0^B$. The reflection identity gives the phase-one maps.
\item Verify the zero-mass face identities $F_i^B(p,d,0)=(F_i^A(p,d),0)$ and divisibility of the B output numerator by $e$. Verify that swapping rows 3 and 4 swaps columns 3 and 4, fixing the other original columns and inducing the same permutation on their negatives.
\item Verify the displayed common-cycle equations $N_1=N_2=0$ and their elimination polynomial $f$, with exactly one root in the stated interval for $p_*$. Denominators at that root exclude zero.
\item Substitute the polynomial formula for $d_*$ into $F_0^A(p,d)-(p,1/2-d)$ and verify the numerator identities modulo $f(p)$. This proves the specified cycle, not just a necessary resultant equation.
\item Verify the exact elimination identity for the irreducible degree-14 polynomial $P_\lambda$ and the associated characteristic-polynomial identities of the actual A return derivative.
\item Verify the exact elimination identity for the irreducible degree-7 polynomial $P_\kappa$ and its identification with the transverse B return multiplier.
\item Isolate $\lambda,\nu,\kappa$ and verify $0<\nu<\lambda<\kappa<1$. Before intersecting an evaluated eigenvalue enclosure and a polynomial isolator, require exactly one root on their convex hull and neither endpoint a root. The exact annihilation identities and this root count justify the identification.
\item Solve the exact eigenvector equations with first coordinate 1 and enclose the directions in Section~\ref{sec:halfstep}. Verify the zero transverse entries in the first column of each normalized residual derivative as identities.
\item Verify the positive half-step multipliers $c_A=\sqrt\lambda$ and $c_B=\sqrt\kappa$ in their displayed intervals. Check the direct half-step multiplier polynomials and their norms stated below. Their identification follows from the return-spectrum identities, the positive square-root choice, and exact root isolation. The norm argument establishing irrationality is analytic.
\item On the two physical boxes in Section~\ref{sec:sectors}, bound the derivatives and Hessians of $F_0^A,F_0^B$ with rational intervals. Use the stated coordinate changes to compute $K_x,K_\perp,a,b$ and verify $m>3/10$, $M<1/2$, $\alpha<7/10$, $\beta<10$, $\Gamma<1/3$, and $\Delta<6$. For the optional run-length conclusion, the same enclosures give radial factors in $(0.307,0.308)$ and $(0.451,0.452)$.
\item Apply interval updates along $H_0$ on each physical box, retaining all original edge intervals. At each of the five steps the prescribed lower endpoint exceeds every competing local upper endpoint by more than $1/40$. Reflection transports the result to $H_1$.
\item With the boundary and later-edge enclosures defined in Section~\ref{sec:locking}, verify $\Delta^{\mathrm{box}}-\delta_{\mathrm{bdry}}^{\mathrm{box}}>3/40$. This is the whole-burst locking input.
\item Initialize all twenty weights to $A_r^{\mathrm{num}}B_s^{\mathrm{num}}/2500$. Expand the 28-letter word in Proposition~\ref{prop:entry} into 140 columns, alternating $H_0,H_1$ separately within each gadget. At every step require the prescribed edge's lower endpoint to exceed the other fifteen upper endpoints by $10^{-8}$. Apply the update with positive denominator and round endpoints outward to $2^{-1024}\mathbb Z$. After 140 steps, take marginals, isolate $p_*$ to width at most $10^{-50}$, evaluate $d_*$, and verify the endpoint intervals in Proposition~\ref{prop:entry}. The phase bits are $(0,0)$. The chart constraints follow exactly from the last column-2 updates, not from rounded coordinate sums.
\item Require rational arithmetic in all acceptance calculations and assertions enabled at execution. Every interval divisor must exclude zero. Floating-point conversion is used only for reporting. This is an implementation discipline, not an extra mathematical lemma.
\item For each choice of active gadget, compare its selected-edge lower endpoint at every step of $H_0$ with every negated-edge upper endpoint from both the active and frozen boxes. Require the gap to exceed $1/4$. Reflection covers all phase pairs; Item 13 handles the finite prefix separately.
\item Verify the row multiplicities sum to 2,500 and are positive. Check the column witness $(2,1,1,1,2,1,1,1)/10$ and row witness $(1,2,1,1)/5\otimes(1,2,1,1,0)/5$. They prove margin $1/5$ for the original and negation-closed compact matrices; row aggregation transfers it to the uniform encoding.
\end{enumerate}

\paragraph{Proof--certificate alignment.}
Table~\ref{tab:alignment} connects each numbered obligation above to its consuming argument and stable code anchors in the active SymPy implementation. The item text specifies its domain and acceptance condition; \path{certificate/CLAIM_MAP.md} gives the same audit map outside the PDF. Source hashes in the regenerated report identify executable versions. A code anchor names the relevant calculation, not a formal proof of its mathematical interpretation.

{\small
\begin{longtable}{@{}p{0.06\linewidth}p{0.43\linewidth}p{0.45\linewidth}@{}}
\caption{Finite obligations and their analytic consumers.}\label{tab:alignment}\\
\toprule
Item & Argument supported & Code anchors (file in certificate/)\\
\midrule
\endfirsthead
\toprule Item & Argument supported & Code anchors\\\midrule
\endhead
1 & Branch-map definitions, Section~\ref{sec:branch} & \path{algebra.py}: \texttt{update\_state}, \texttt{F0A\_p}, \texttt{F0B\_p}\\
2 & Face restriction, Lemma~\ref{lem:face-restriction}; phase transport & \path{algebra.py}: \texttt{rem0}, \texttt{rem1}; \path{verify.py}: \texttt{perm}\\
3 & Common-cycle isolation, Section~\ref{sec:2cycle} & \path{algebra.py}: \texttt{N1}, \texttt{N2}, \texttt{f\_poly}\\
4 & Sufficiency of cycle identities, Section~\ref{sec:2cycle} & \path{algebra.py}: \texttt{Dexpr}, polynomial remainders modulo $f$\\
5 & A return spectrum, Proposition~\ref{prop:spectral-data} & \path{algebra.py}: \texttt{P\_lambda\_poly}, \texttt{tr\_RA}, \texttt{det\_RA}\\
6 & B transverse spectrum, Proposition~\ref{prop:spectral-data} & \path{algebra.py}: \texttt{P\_kappa\_poly}, \texttt{kappa\_expr}\\
7 & Actual eigenvalue identification and ordering & \path{algebra.py}: \texttt{lam\_from\_char\_iv}, convex-hull root counts, \texttt{nu\_iv}, \texttt{kap\_iv}\\
8 & Half-step directions, Lemma~\ref{lem:halfstep}; Lemma~\ref{lem:firstcol} & \path{algebra.py}: \texttt{thetaA\_iv}, \texttt{chiB\_iv}, ideals \texttt{GA}, \texttt{GB}\\
9 & Multipliers and norm input to Lemma~\ref{lem:norm-irrational} & \path{verify.py}: \texttt{qa}, \texttt{qb}, \texttt{norms}; \path{algebra.py}: \texttt{cA\_iv}, \texttt{cB\_iv}\\
10 & Taylor inputs for Lemma~\ref{lem:sector-invariance} & \path{verify.py}: \texttt{abox}, \texttt{bbox}, \texttt{asector}, \texttt{bsector}\\
11 & Unique local words, Section~\ref{sec:locking} & \path{verify.py}: \texttt{wordA}, \texttt{wordB}\\
12 & Whole-burst locking, Lemma~\ref{lem:burst-lock} & \path{verify.py}: \texttt{branch}, \texttt{lock}\\
13 & Finite entry, Proposition~\ref{prop:entry} & \path{verify.py}: \texttt{prefix\_gaps}, \texttt{seq}, \texttt{endpoint}\\
14 & Arithmetic discipline for every finite check & \path{algebra.py}: \texttt{QInterval}; assertion guards in both active scripts\\
15 & Negation-closed conclusion, Corollary~\ref{cor:negation} & \path{verify.py}: \texttt{neggaps}; sixteen-column prefix\\
16 & Margin and uniform encoding, Lemma~\ref{lem:duplication} & \path{verify.py}: \texttt{counts}, \texttt{colweights}, \texttt{rowweights}\\
\bottomrule
\end{longtable}
}

The analytic arrows remain outside the software: Taylor bounds imply sector invariance (Lemma~\ref{lem:sector-invariance}); summability gives the local logarithmic law (Lemma~\ref{lem:log}); competition and the norm obstruction give irrational frequency (Theorem~\ref{thm:competition}, Proposition~\ref{prop:irrational-frequency}); separated A-marginal limits imply noncycling (Proposition~\ref{prop:no-cycle}); row aggregation and continuity yield the uniform-start and open-product-neighborhood conclusions (Section~\ref{sec:start}). Checking these arrows against the finite specification is the remaining human audit.

\paragraph{Exact normalized spectral data.}
The dominant multipliers and normalized eigenvector slopes have certified
enclosures
\begin{align*}
 c_A&\in[0.30739560746557,\,0.30739560746558],\\
 c_B&\in[0.45153264872461,\,0.45153264872462],\\
 \theta_A&\in[0.580405272518281104,\,0.580405291461442916],\\
 \eta_B&\in[0.705067523141833,\,0.705067554266809],\\
 \chi_B&\in[1.01349886578531906,\,1.01349888527878103].
\end{align*}

The multiplier polynomials are
\begin{align}
 \mathcal Q_A(u)={}&4u^{14}-268u^{13}-5676u^{12}-76559u^{11}
       +61076u^{10}+5403070u^9\notag\\
 &+3714640u^8-44707134u^7+159688998u^6-80404208u^5\notag\\
 &+1806746u^4+2355093u^3+127106u^2+1846u+2,\label{eq:halfstep-poly-a}\\
 \mathcal Q_B(u)={}&u^7-9u^6+23u^5-3u^4-45u^3+13u^2+5u-1.
 \label{eq:halfstep-poly-b}
\end{align}
Exact elimination from the characteristic polynomial identifies $c_G$ as
the unique root of $\mathcal Q_G$ in its displayed interval, and both $\mathcal Q_G$ are
irreducible over $\mathbb Q$. This identification needs two ingredients: an
algebraic identity showing that $\mathcal Q_G$ annihilates the actual derivative
eigenvalue, and a root isolation argument. If an enclosure of the
evaluated eigenvalue is combined with an isolating interval for a root of
the polynomial, it is not enough for the two intervals to overlap. One
must count roots in their \emph{convex hull}, find exactly one, and check
that neither endpoint is a root. Refining the isolating intervals and
solving the exact eigenvector equations with first coordinate $1$ gives
the direction enclosures above. The transverse entries of
$(DT_G-c_GI)v_G$ are exactly zero, not merely small.

The shared algebra module also verifies the return-spectrum data in Section~\ref{sec:spectra}, including $0<\nu<\lambda<\kappa<1$. The identities $\lambda=c_A^2$ and $\kappa=c_B^2$ connect the normalized proof to the return-map notation used in the figures.

The norms follow directly from the leading and constant coefficients:
\[
 N_{\mathbb Q(c_A)/\mathbb Q}(c_A)=\tfrac12,\qquad
 N_{\mathbb Q(c_B)/\mathbb Q}(c_B)=1.
\]
If $\log c_A/\log c_B=m/n\in\mathbb Q$, with positive integers $m,n$, then
$c_A^n=c_B^m$. Taking norms in
$K=\mathbb Q(c_A,c_B)$ gives
$2^{-n[K:\mathbb Q(c_A)]}=1$, a contradiction. This proves the irrationality
used in the global argument.

The core symbolic work is in \texttt{algebra.py}; the normalized cones, symmetry, half-step checks, entry, and margin tests are in \texttt{verify.py}. Stable code anchors include \texttt{F0A}, \texttt{F0B}, \texttt{lam\_from\_char\_iv}, \texttt{abox}, \texttt{bbox}, \texttt{asector}, \texttt{bsector}, \texttt{wordA}, \texttt{wordB}, \texttt{neggaps}, \texttt{prefix\_gaps}, and \texttt{endpoint}. The result file fixes the executable versions by SHA-256 hashes. The local bound report also supplies the rational enclosures for the optional radial intervals in Item 10.

\subsection{Implementations and reproducibility}
\label{app:certdoc-impl}

\paragraph{Directory structure.}
The \texttt{certificate/} directory in the linked repository contains the complete SymPy checker: \texttt{algebra.py}, \texttt{verify.py}, and \texttt{verification.json}. Its algebra module contains the shared exact algebra and spectral identification used by the active proof. The optional \texttt{legacy\_verify.py} keeps the earlier four-phase checks and original initialization separate; the active verifier does not execute those checks. The portable \path{certificate/replication/replay_sage.py} replays the current normalized local bounds, multiplier polynomials, strict entry and margin with native SageMath algebra and shared geometric/entry acceptance logic. It supplies convex-hull root identification before intersecting the retained Sage algebra module's eigenvalue enclosures. Its report is \path{certificate/replication/verification_sage.json}; older reports in the same folder are historical. The two backends share the mathematical specification and some verification logic, so their agreement is cross-implementation replication rather than independent formal verification of the entire proof.

\paragraph{Requirements.}
The repository checker uses Python 3.11 or newer and SymPy 1.14; rational intervals use Python's standard-library fraction and integer arithmetic. The current Sage replay uses SageMath 10.8.

\paragraph{Safety guard.}
Run with assertions enabled. The entry script refuses optimized execution, which would strip its assertions.

\paragraph{Usage.}
From the root of the linked repository, run
\begin{verbatim}
python3 certificate/verify.py
python3 certificate/check_report.py
# Optional native Sage replay, then check both reports:
sage -python certificate/replication/replay_sage.py
python3 certificate/check_report.py
\end{verbatim}
The first command regenerates the active two-cone report and executable hashes, including checks of the outward algebraic enclosures and endpoint table displayed in this manuscript. The report checker validates source hashes, the retained radial bounds, and, when the Sage report is present, agreement on the exact selected-column itinerary, encoding, margin and norms. A successful run exits with status zero. The report's finite assertions support the specified propositions; the analytic deductions remain the mathematical text's responsibility.

%% file: sections/19_final_evidence.tex
\section{Post-submission refinements and exploratory evidence}
\label{app:final-evidence}

\subsection{What was simplified}
The post-submission work described in Section~\ref{sec:workflow-refinement} used GPT-5.6 Sol and GPT-6 Astra, including follow-up criticism and verification. The saved account of that work credits the later-model conversations collectively: it does not consistently identify a separate model for each individual improvement. The original construction and transcript excerpts concern GPT-5.4 Pro and Claude Opus 4.6. The original effort ledger does not include the later refinements.

The final revisions preserve the block-product construction and its irrational-rate obstruction. Reflection identifies the two local phases, leaving one normalized map and one cone per gadget. The competition theorem permits any maximizing tie choices and makes the optional run-length estimates unnecessary for noncycling. The state argument uses just the fourth coordinate of the A-marginal, whose two limiting phase values are separated. These are simplifications of the argument, not a different counterexample mechanism.

The smaller initialization follows a separate observation: the rational start need not already lie in a narrow cone. It suffices to certify a finite transient entering it. Numerical screening of small integer marginals with denominators 50 or 100 led to the start in~\eqref{eq:simple-start}. An earlier candidate with $b=(7,18,11,13,1)/50$ had early ties; moving one numerator unit from its second coordinate to its first produced the retained strict-prefix candidate. The exact 140-step calculation, rather than screening, establishes its validity. The search was not exhaustive, and 2,500 is an explicit sufficient row count, not a minimum.

\subsection{Binary transfer to AdaBoost.R}
For a sign-valued weak hypothesis, the normalized prediction in \citet{NockNielsen2007} is $z_i=y_i h(x_i)/\max_k|h(x_k)|\in\{-1,+1\}$. Their update becomes
\[
D'_i=D_i\frac{1-rz_i}{1-r^2}=\frac{D_i}{1+rz_i}.
\]
This is the discrete update used here. Under the same finite sign class, exhaustive maximizing selection and tie choices, the complete weight trajectory agrees. This proves the stated binary-subcase transfer, without a claim about nonbinary hypotheses or other update rules.

\subsection{Records and reproduction}
Run the current certificate and its report check from the repository root:
\begin{verbatim}
python3 certificate/verify.py
python3 certificate/check_report.py
\end{verbatim}
The usage ledger is reconstructed from all of the saved conversations. Its labels are not independently recoded from the full conversations, and completeness relative to all project interactions has not been established. Appendix~\ref{app:ledger} lists each conversation and marks the two missing durations.

%% file: sections/18_usage_ledger.tex
\clearpage
\section{Usage by conversation}\label{app:ledger}

Table~\ref{tab:ledger} breaks down the totals in Table~\ref{tab:workflow-stats} by conversation. It covers the 46 GPT-5.4 Pro conversations from the original construction and verification, grouped by their main activity. For each conversation we give the number of prompts and the total response time, which is the sum of the ``Thought for \ldots'' durations shown in the web interface. Response time is not wall-clock time, and it does not include any human work. In the two conversations marked $^\dagger$, one response has no saved duration, so their times cover one fewer response than their prompts. The table does not include Claude conversations or the later work with GPT-5.6 Sol and GPT-6 Astra.

\begin{table}[ht]
\centering
\caption{GPT-5.4 Pro prompts and response time per conversation. Over all 46 conversations there are 133 prompts and 64h 04m 36s of response time.}
\label{tab:ledger}
{\small
\begin{tabular}{@{}lrr@{\qquad}lrr@{\qquad}lrr@{}}
\toprule
\multicolumn{3}{@{}c}{Exploration} & \multicolumn{3}{c}{Construction} & \multicolumn{3}{c@{}}{Verification}\\
\cmidrule(r{1.5em}){1-3}\cmidrule(r{1.5em}){4-6}\cmidrule{7-9}
ID & Prompts & Time & ID & Prompts & Time & ID & Prompts & Time\\
\midrule
C17$^\dagger$ & 2 & 11m 48s & C01 & 4 & 19m 04s & C04 & 2 & 1h 26m 13s\\
C18 & 1 & 1h 32m 50s & C02 & 3 & 2h 23m 27s & C05 & 1 & 56m 43s\\
C19 & 2 & 2h 00m 37s & C03 & 1 & 26m 32s & C06 & 1 & 1h 04m 06s\\
C20 & 2 & 1h 18m 09s & C13 & 3 & 1h 38m 00s & C07 & 1 & 41m 22s\\
C31 & 4 & 3h 31m 46s & C15 & 4 & 2h 02m 43s & C08 & 2 & 1h 08m 37s\\
C32 & 7 & 3h 12m 07s & C16 & 3 & 1h 02m 46s & C09 & 1 & 41m 12s\\
C33 & 2 & 50m 45s & C21 & 3 & 3h 20m 39s & C10 & 4 & 1h 44m 27s\\
C34 & 2 & 32m 51s & C22 & 6 & 2h 56m 14s & C11 & 2 & 51m 05s\\
C40 & 1 & 5m 01s & C23 & 2 & 1h 55m 02s & C12 & 4 & 2h 29m 46s\\
C41 & 6 & 1h 41m 04s & C24 & 7 & 5h 18m 57s & C14 & 2 & 1h 17m 07s\\
C42 & 1 & 5m 09s & C25 & 5 & 2h 48m 48s & C35 & 1 & 9m 39s\\
C43 & 8 & 1h 48m 27s & C26 & 5 & 2h 21m 37s & & &\\
C44 & 3 & 1h 50m 52s & C27 & 1 & 55m 34s & & &\\
C45 & 3 & 1m 18s & C28 & 1 & 43m 06s & & &\\
C46 & 1 & 10m 18s & C29 & 1 & 43m 51s & & &\\
& & & C30 & 2 & 1h 20m 10s & & &\\
& & & C36 & 5 & 35m 03s & & &\\
& & & C37 & 1 & 9m 07s & & &\\
& & & C38$^\dagger$ & 5 & 47m 57s & & &\\
& & & C39 & 5 & 52m 40s & & &\\
\midrule
Total & 45 & 18h 53m 02s & Total & 67 & 32h 41m 17s & Total & 21 & 12h 30m 17s\\
\bottomrule
\end{tabular}
}
\end{table}

%% file: references.bib
@article{FS1997,
  author    = {Yoav Freund and Robert E. Schapire},
  title     = {A Decision-Theoretic Generalization of On-Line Learning and an Application to Boosting},
  journal   = {Journal of Computer and System Sciences},
  volume    = {55},
  number    = {1},
  pages     = {119--139},
  year      = {1997}
}

@book{Moore2009,
  author    = {Ramon E. Moore and R. Baker Kearfott and Michael J. Cloud},
  title     = {Introduction to Interval Analysis},
  publisher = {SIAM},
  address   = {Philadelphia},
  year      = {2009}
}

@article{stein2005sage,
  title={{Sage}: System for algebra and geometry experimentation},
  author={Stein, William and Joyner, David},
  journal={{ACM SIGSAM} Bulletin},
  volume={39},
  number={2},
  pages={61--64},
  year={2005},
  publisher={ACM New York, NY, USA}
}

@book{Lang2002,
  title={Algebra},
series = {Graduate Texts in Mathematics},
  author={Lang, Serge},
  year={2012},
  publisher={Springer}
}

@misc{solveall2026,
  author = {Edgar Dobriban},
  title = {{SolveAll.org} - real research problems. No known solutions.},
  year = {2026},
  howpublished = {\url{https://solveall.org/}},
  note = {Accessed: 2026-04-06}
}

@article{meurer2017sympy,
  title={{SymPy}: symbolic computing in {Python}},
  author={Meurer, Aaron and Smith, Christopher P and Paprocki, Mateusz and {\v{C}}ert{\'\i}k, Ond{\v{r}}ej and Kirpichev, Sergey B and Rocklin, Matthew and Kumar, AMiT and Ivanov, Sergiu and Moore, Jason K and Singh, Sartaj and others},
  journal={PeerJ Computer Science},
  volume={3},
  pages={e103},
  year={2017},
  publisher={PeerJ Inc.}
}

@article{RDS2004,
  author    = {Cynthia Rudin and Ingrid Daubechies and Robert E. Schapire},
  title     = {The Dynamics of {AdaBoost}: Cyclic Behavior and Convergence of Margins},
  journal   = {Journal of Machine Learning Research},
  volume    = {5},
  pages     = {1557--1595},
  year      = {2004}
}

@book{cox1997ideals,
  title={Ideals, varieties, and algorithms},
  author={Cox, David and Little, John and O'Shea, Donal},
  year={1997},
  publisher={Springer}
}

@inproceedings{RSD2012,
  author    = {Cynthia Rudin and Robert E. Schapire and Ingrid Daubechies},
  title     = {Open Problem: Does {AdaBoost} Always Cycle?},
  booktitle = {Proceedings of the 25th Annual Conference on Learning Theory (COLT)},
  series    = {Proceedings of Machine Learning Research},
  volume    = {23},
  pages     = {46.1--46.4},
  publisher = {PMLR},
  year      = {2012}
}

@book{SF2012,
  author    = {Robert E. Schapire and Yoav Freund},
  title     = {Boosting: Foundations and Algorithms},
  publisher = {MIT Press},
  year      = {2012}
}

@article{feng2026towards,
  title={Towards autonomous mathematics research},
  author={Feng, Tony and Trinh, Trieu H and Bingham, Garrett and Hwang, Dawsen and Chervonyi, Yuri and Jung, Junehyuk and Lee, Joonkyung and Pagano, Carlo and Kim, Sang-hyun and Pasqualotto, Federico and others},
  journal={arXiv preprint arXiv:2602.10177},
  year={2026}
}

@article{woodruff2026accelerating,
  title={Accelerating scientific research with gemini: Case studies and common techniques},
  author={Woodruff, David P and Cohen-Addad, Vincent and Jain, Lalit and Mao, Jieming and Zuo, Song and Bateni, MohammadHossein and Branzei, Simina and Brenner, Michael P and Chen, Lin and Feng, Ying and others},
  journal={arXiv preprint arXiv:2602.03837},
  year={2026}
}

@article{bubeck2025early,
  title={Early science acceleration experiments with {GPT}-5},
  author={Bubeck, S{\'e}bastien and Coester, Christian and Eldan, Ronen and Gowers, Timothy and Lee, Yin Tat and Lupsasca, Alexandru and Sawhney, Mehtaab and Scherrer, Robert and Sellke, Mark and Spears, Brian K and others},
  journal={arXiv preprint arXiv:2511.16072},
  year={2025}
}

@article{dobriban2025solving,
  title={Solving a Research Problem in Mathematical Statistics with {AI} Assistance},
  author={Dobriban, Edgar},
  journal={arXiv preprint arXiv:2511.18828},
  year={2025}
}

@misc{sottiaux2026tps,
  author = {Sottiaux, Tibo},
  title = {Post on {X} (@thsottiaux) on model token output speed (``Reaching 50 {TPS} instead of 30{TPS}\ldots'')},
  year = {2026},
  note = {Posted October 5, 2026},
  howpublished = {\url{https://x.com/thsottiaux/status/2107159119107146237}}
}

@misc{openai2026pricing,
  author = {{OpenAI}},
  title = {{OpenAI} {API} Pricing},
  year = {2026},
  note = {Accessed April 2026},
  howpublished = {\url{https://developers.openai.com/api/docs/pricing}}
}

@article{swaminathan2026quartic,
  title={On the Quartic Invariant of Odd Degree Binary Forms},
  author={Swaminathan, Ashvin},
  journal={arXiv preprint arXiv:2603.24330},
  year={2026}
}

@article{alexeev2026short,
  title={Short proofs in combinatorics and number theory},
  author={Alexeev, Boris and Putterman, Moe and Sawhney, Mehtaab and Sellke, Mark and Valiant, Gregory},
  journal={arXiv preprint arXiv:2603.29961},
  year={2026}
}

@misc{BelanichOrtiz2023,
  author = {Joshua Belanich and Luis E. Ortiz},
  title = {On the Convergence Properties of Optimal {AdaBoost}},
  year = {2023},
  note = {arXiv:1212.1108v3; first posted 2012},
  url = {https://arxiv.org/abs/1212.1108}
}

@misc{Snedeker2022,
  author = {Conor Snedeker},
  title = {Limit Cycles of {AdaBoost}},
  year = {2022},
  note = {arXiv:2209.06928},
  url = {https://arxiv.org/abs/2209.06928}
}

@misc{BejaBattais2023,
  author = {Perceval Beja-Battais},
  title = {Overview of {AdaBoost}: Reconciling its views to better understand its dynamics},
  year = {2023},
  note = {arXiv:2310.18323},
  url = {https://arxiv.org/abs/2310.18323}
}

@article{NockNielsen2007,
  author = {Richard Nock and Frank Nielsen},
  title = {A Real generalization of discrete {AdaBoost}},
  journal = {Artificial Intelligence},
  volume = {171},
  number = {1},
  pages = {25--41},
  year = {2007},
  doi = {10.1016/j.artint.2006.10.014}
}
